\documentclass[twoside,11pt, final]{article}

\usepackage{jmlr2e}
\usepackage{amsmath}
\usepackage{amssymb}
\usepackage{graphicx}
\usepackage{color}
\usepackage{ifpdf}
\usepackage{url}
\usepackage{hyperref}
\usepackage{algorithm}
\usepackage{algorithmic}
\usepackage{longtable}
\usepackage{subfigure}
\usepackage{caption}
\usepackage{multirow}

\newcommand{\abs}[1]{\left\vert#1\right\vert}
\newcommand{\set}[1]{\left\{#1\right\}}

\newcommand{\norm}[1]{\left\Vert#1\right\Vert}
\newcommand{\eproof}{\hfill $\square$}

\newcommand{\dom}[1]{\mathrm{dom}(#1)}
\renewcommand{\vec}{\mathbf{vec}}

\newcommand{\trace}{\mathrm{tr}}
\newcommand{\cov}{\boldsymbol{\Sigma}}
\newcommand{\invcov}{\boldsymbol{\Sigma}^{-1}}
\newtheorem{assumption}{Assumption A.\!\!}

\newcommand{\vectornorm}[1]{\|#1\|}

\newcommand{\T}{\mathbf{\Theta}}

\newcommand{\DT}{\boldsymbol{\Delta}}
\newcommand{\xb}{\mathbf{x}}
\newcommand{\phib}{\mbox{${\mathbf \phi}$}}
\newcommand{\yb}{\mathbf{y}}
\newcommand{\zb}{\mathbf{z}}
\newcommand{\ub}{\mathbf{u}}
\newcommand{\vb}{\mathbf{v}}
\renewcommand{\sb}{\mathbf{s}}
\newcommand{\db}{\mathbf{d}}
\newcommand{\pb}{\mathbf{p}}

\newcommand{\eb}{\mathbf{e}}
\newcommand{\rb}{\mathbf{r}}
\newcommand{\Hb}{\mathbf{H}}

\newcommand{\tr}{\text{tr}}

\DeclareMathOperator*{\argmin}{arg\,min}

\usepackage{booktabs}
\begin{document}

\title{Composite Self-Concordant Minimization}

 \author{\name Quoc Tran-Dinh$^{\dagger}$ \email quoc.trandinh@epfl.ch
        \AND
        \name Anastasios Kyrillidis$^{\dagger}$ \email anastasios.kyrillidis@epfl.ch
	\AND
	\name Volkan Cevher$^{\dagger}$ \email volkan.cevher@epfl.ch\\
	\addr $^{\dagger}$Laboratory for Information and Inference Systems (LIONS)\\
        \'{E}cole Polytechnique F\'{e}d\'{e}rale de Lausanne (EPFL)\\
        CH1015-Lausanne, Switzerland
	}
	
\editor{Unknown}

\maketitle	

\begin{abstract}
We propose a variable metric framework for minimizing the sum of a self-concordant function and a possibly non-smooth convex function, endowed with an easily computable proximal operator. We theoretically establish the convergence of our framework without relying on the usual Lipschitz gradient assumption on the smooth part. An important highlight of our work is a new set of analytic step-size selection and correction procedures based on the structure of the problem. 
We describe concrete algorithmic instances of our framework for several interesting applications and demonstrate them numerically on both synthetic and real data.
\end{abstract}

\begin{keywords}
Proximal-gradient/Newton method, composite minimization, self-concordance, sparse convex optimization, graph learning.
\end{keywords}

\section{Introduction}\label{sec:intro}
The literature on the formulation, analysis, and applications of \textit{composite convex minimization}  is ever expanding due to its broad applications in machine learning, signal processing, and statistics. By composite minimization, we refer to the following optimization problem:
\begin{align}\label{eq:COP} 
F^{*} := \min_{\xb \in \mathbb{R}^n} \left \{ F(\xb)~\left|~~ F(\xb) := f(\xb) + g(\xb) \right.\right\},
\end{align}
where $f$ and $g$ are both closed and convex, and $n$ is the problem dimension. In the canonical setting of the composite minimization problem \eqref{eq:COP}, the functions $f$ and $g$ are assumed to be smooth and non-smooth, respectively \citep{Nesterov2007}. Such composite objectives naturally arise, for instance, in maximum a posteriori model estimation, where we regularize a model likelihood function as measured by a data-driven smooth term $f$ with a non-smooth model prior $g$, which carries some notion of  model complexity (e.g., sparsity, low-rankness, etc.).

\emph{In theory}, many convex problem instances of the form \eqref{eq:COP} have a well-understood structure, and hence high accuracy solutions can be efficiently obtained with polynomial time methods, such as interior point methods (IPM) after transforming them into conic quadratic programming or semidefinite programming formulations \citep{BenTal2004,Grant2006,Nesterov1994}. 
\emph{In practice}, however, the curse-of-dimensionality renders these methods impractical for large-scale problems. Moreover, the presence of a non-smooth term $g$ prevents direct applications of scalable smooth optimization techniques, such as sequential linear or quadratic programming. 
\begin{figure}
\begin{minipage}[l]{0.39\textwidth}\centering
\includegraphics[width=5.6cm, height=3.0cm]{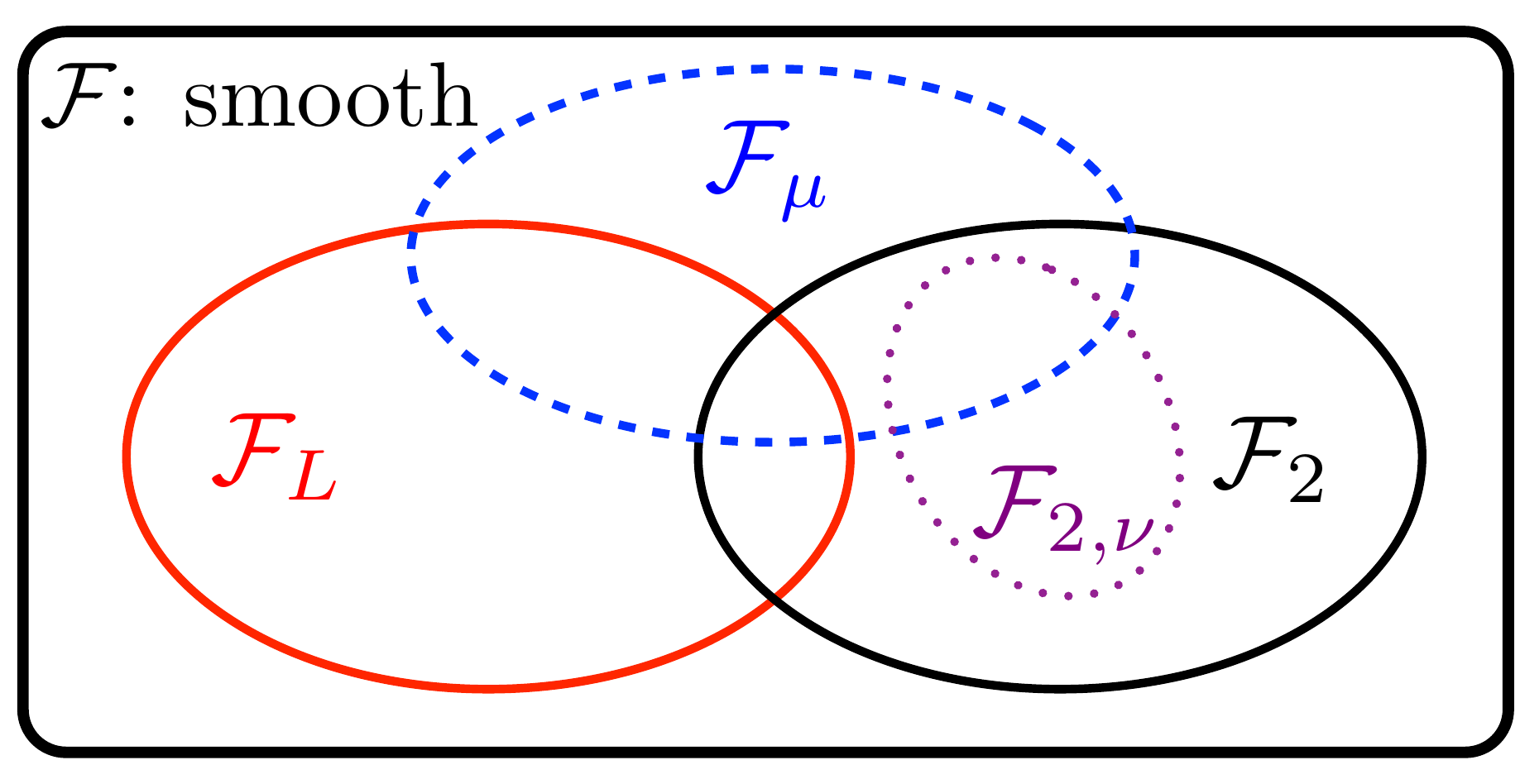}
\end{minipage}
\begin{minipage}[r]{0.50\textwidth}\centering
\begin{footnotesize}
\begin{tabular}{c|c} \toprule
\textbf{Class}              &  \textbf{Property} \\
                                   &  $\xb, \yb \in \dom{f}, ~\vb \in\mathbb{R}^n, 0\leq \mu \leq L <+\infty$ \\ \cmidrule{1-2}
$\mathcal{F}_L$         &  $\norm{\nabla f(\xb) - \nabla f(\yb)}_{*} \leq L \norm{\xb-\yb} $\\ \cmidrule{1-2}
$\mathcal{F}_\mu$     & $\frac{\mu}{2}\norm{\xb-\yb}^2 + f(\xb) + \nabla f(\xb)^T(\yb-\xb)  \leq f(\yb)$\\ \cmidrule{1-2}
$\mathcal{F}_2$         &  $\abs{\varphi^{\prime\prime\prime}(t)} \leq 2 \varphi^{\prime\prime}(t)^{3/2}$: $\varphi(t) = f(\xb+t\mathbf{v}),~ t\in\mathbb{R}$\\ \cmidrule{1-2}
$\mathcal{F}_{2,\nu}$ &  $\mathcal{F}_2$ and  $\sup_{\vb \in\mathbb{R}^n}\set{2\nabla{f}(\xb)^T\vb - \norm{\vb}_{\xb}^2} \leq \nu$ \\ \bottomrule
\end{tabular}
\end{footnotesize}
\end{minipage}
\caption{Common structural assumptions on the smooth function $f$.}\label{fig: smooth function classes}
\end{figure} 

Fortunately, we can provably trade-off accuracy with computation by further exploiting the individual structures of $f$ and $g$. Existing methods invariably rely on two structural assumptions that particularly stand out among many others. First, we often assume that $f$ has Lipschitz continuous gradient (i.e., $f \in \mathcal{F}_L$: cf., Fig.~\ref{fig: smooth function classes}). Second, we assume that the proximal operator of $g$ ($\text{prox}^{\Hb}_g(\yb):= \arg\displaystyle\min_{\xb\in\mathbb{R}^n} \left\{g(\xb)+ 1/2\|\xb -\yb \|_{\Hb}^2 \right\}$) is, in a user-defined sense, easy to compute for some $\Hb \succ 0$ (e.g., $\Hb$ is diagonal); i.e., we can computationally afford to apply the proximal operator in an iterative fashion. In this case, $g$ is said to be ``\textit{tractably proximal}''.
On the basis of these structures, we can design algorithms featuring a full spectrum of (nearly) dimension-independent, global convergence rates with well-understood analytical complexity (see Table \ref{table: taxonomy}). 
\begin{table}[!h]
\begin{center}
\caption{Taxonomy of [accelerated] [proximal]-gradient methods when $f \in \mathcal{F}_L$ or proximal-[quasi]-Newton methods when $f\in\mathcal{F}_L \cap\mathcal{F}_{\mu}$  to reach an $\varepsilon$-solution (e.g., $F(\xb^k)-F^*\le \epsilon$).}\label{table: taxonomy}
\vskip0.1cm
\begin{tabular}{c|c|c|c}\toprule
  Order     &   Method example                &                             Main oracle                               &  Analytical complexity \\ \cmidrule{1-4}\cmidrule{1-4}
  $1$-st      &   [Accelerated]$^{a}$-[proximal]-gradient$^{b}$      &  $\nabla f, \text{prox}^{L\mathbb{I}_n}_g$            &  [$\mathcal{O}(\epsilon^{-1/2})$] $\mathcal{O}(\epsilon^{-1})$ \\ \cmidrule{1-4}
$1^+$-th  &   Proximal-quasi-Newton$^{c}$      &  $\Hb_k, \nabla f, \text{prox}^{\Hb_k}_g$              &  $\mathcal{O}(\log \epsilon^{-1})$ or faster \\ \cmidrule{1-4}
$2$-nd     &   Proximal-Newton$^d$               &  $\nabla^2 f, \nabla f, \text{prox}^{\nabla^2 f}_g$   &  $\mathcal{O}(\log\log \epsilon^{-1}) \textrm{[local]}$  \\ \bottomrule
\end{tabular}
\vskip0.1cm
\begin{footnotesize}See \citep{Beck2009}$^{a,b}$,\citep{Becker2012b}$^{c}$,\citep{Lee2012}$^{d}$,\citep{Nesterov2004,Nesterov2007}$^{a,b}$.\end{footnotesize} 
\end{center}
\vskip-0.5cm
\end{table}

Unfortunately,  existing algorithms have become inseparable with the Lipschitz gradient assumption on $f$ and are still being applied to solve \eqref{eq:COP} in applications where this assumption does not hold. For instance, when $\text{prox}^{\mathbf{H}}_g(\yb)$ is not easy to compute, it is still possible to establish convergence---albeit slower---with smoothing, splitting or primal-dual decomposition techniques \citep{Chambolle2011,Eckstein1992,Nesterov2005c,Nesterov2005d,TranDinh2012a}. However, when $f \notin \mathcal{F}_L$, the composite problems of the form \eqref{eq:COP} are not within the full theoretical grasp. In particular, there is no known global convergence rate. One kludge to handle $f \notin \mathcal{F}_L$ is to use sequential quadratic approximation of $f$ to reduce the subproblems to the Lipschitz gradient case. For local convergence of these methods, we need \textit{strong regularity} assumptions on $f$ (i.e., $\mu\mathbb{I} \preceq \nabla^2f(\xb) \preceq L\mathbb{I}$) near the optimal solution. Attempts at global convergence require a \textit{globalization strategy} such as line search procedures (cf., Section \ref{sec: related}). However, neither the strong regularity nor the line search assumptions can be certified \emph{a priori}. 

To this end, we address the following question in this paper: ``Is it possible to efficiently solve non-trivial instances of \eqref{eq:COP} for non-global Lipschitz continuous gradient $f$ with rigorous global convergence guarantees?'' The answer is positive (at least for a broad class of functions): We can still cover a full spectrum of global convergence rates with well-characterizable computation and accuracy trade-offs (akin to Table \ref{table: taxonomy} for $f\in \mathcal{F}_L$) for self-concordant $f$ (in particular, self-concordant barriers) \citep{Nemirovski2009,Nesterov1994}:

\begin{definition}[Self-concordant (barrier) functions]\label{de:concordant}
A convex function $f: \mathbb{R}^n \rightarrow \mathbb{R} $ is  said to be self-concordant $($i.e., $f\in \mathcal{F}_{M}$ $)$ with  parameter $M \geq 0$, if $\abs{\varphi'''(t)} \leq M\varphi''(t)^{3/2}$, where $\varphi(t) := f(\xb + t\vb)$ for all $ t \in\mathbb{R}$, $\xb \in \dom{f}$ and $\vb \in \mathbb{R}^{n}$ such that $\xb + t\vb \in \dom{f}$. 
When $M = 2$, the function $f$ is said to be a standard self-concordant, i.e., $f \in \mathcal{F}_2$.\footnote{We use this constant for convenience in the derivations since if $f \in \mathcal{F}_M$, then $(M^2/4)f \in \mathcal{F}_2$.}
A standard self-concordant function $f \in \mathcal{F}_2$ is a $\nu$-self-concordant barrier of  a given convex set $\Omega$ with parameter  $\nu> 0$, i.e., $f\in \mathcal{F}_{\nu}$, when $\varphi$ also satisfies $\abs{\varphi'(t)} \leq \sqrt{\nu}\varphi''(t)^{1/2}$ and $f(\xb)\to+\infty$ as $\xb\to\partial{\Omega}$, the boundary of $\Omega$. 
\end{definition}
While there are other definitions of self-concordant functions and self-concordant barriers \citep{Boyd2004,Nemirovski2009,Nesterov1994,Nesterov2004}, we use Definition \ref{de:concordant} in the sequel, unless otherwise stated. 

\subsection{Why is the assumption $f \in \mathcal{F}_2$ interesting for composite minimization?}
The assumption $f \in \mathcal{F}_2$ in \eqref{eq:COP} is quite natural for two reasons. First, several important applications directly feature a self-concordant $f$, which does not have global Lipschitz continuous gradient. Second, self-concordant composite problems can enable approximate solutions of general constrained convex problems where the constraint set is endowed with a $\nu$-self-concordant barrier function.\footnote{Let us consider a constrained convex minimization $\xb^{*}_C := \arg\min_{\xb \in C} g(\xb)$, where the feasible convex set $C$ is endowed with a $\nu$-self-concordant barrier $\Psi_C(\xb)$. If we let $f(\xb) := \frac{\epsilon}{\nu} \Psi_C(\xb)$, then the solution $\xb^{*}$ of the composite minimization problem \eqref{eq:COP} well-approximates $\xb^{*}_C$ as  $g(\xb^{*}) \leq g(\xb^{*}_C) + (\nabla f(\xb^{*}) +\partial g(\xb^{*}))^T(\xb^{*} - \xb^{*}_C) + \epsilon$. The middle term can be controlled by accuracy at which we solve the composite minimization problem \citep{Nesterov2007,Nesterov2011c}.} 
Both settings clearly benefit from scalable algorithms. Hence, we now highlight three examples below, based on compositions with the $\log$-functions. Keep in mind that this list of examples is not meant to be exhaustive. 

\paragraph{Log-determinant:} 
The matrix variable function $f(\mathbf{\Theta}) := - \log \det \mathbf{\Theta}$ is self-concordant with $\dom{f} := \set{ \mathbf{\Theta}\in \mathbb{S}^p  ~|~ \mathbf{\Theta}\succ 0}$, where $\mathbb{S}^p$ is the set of $p\times p$ symmetric matrices. As a stylized application,  consider learning a Gaussian Markov
random field (GMRF) of $p$ nodes/variables from a dataset $\mathcal{D} := \set{\boldsymbol{\phib}_{1}, \boldsymbol{\phib}_{2}, \dots, \boldsymbol{\phib}_{m}}$, where $\boldsymbol{\phib}_j \in
\mathcal{D}$ is a $p$-dimensional random vector with Gaussian distribution $\mathcal{N}(\boldsymbol{\mu}, \cov)$. Let $\T := \invcov$ be the inverse covariance
(or the precision) matrix for the model. To satisfy the conditional dependencies with respect to the GMRF, $\T$ must have zero in $(\T)_{ij}$
corresponding to the absence of an edge between node $i$ and node $j$; cf., \citep{Dempster1972}. 

We can learn GMRF's with theoretical  guarantees from as few as $ \mathcal{O}(d^2\log p)$ data samples, where $d$ is the graph node degree, via $\ell_1$-norm regularization formulation (see \citep{Ravikumar2011}):
\begin{equation}\label{eq:glearn_prob}
\mathbf{\Theta}^{\ast} :=\argmin_{\T \succ 0} \Big\{ \underbrace{-\log \det (\T) + \trace(\widehat{\cov} \T)}_{=:f(\T)}  + \underbrace{\rho
\vectornorm{\vec(\T)}_1}_{=:g(\T)} \Big \}, 
\end{equation} 
where $\rho > 0$ parameter balances a Gaussian model likelihood and the sparsity of the solution, $\widehat{\cov} $ is the empirical covariance estimate, and $\vec$ is the vectorization operator. The  formulation also applies for learning models beyond GMRF's, such as the Ising model, since $f(\T)$ acts also as a Bregman distance \citep{Banerjee2008}. 

Numerical solution methods for  solving problem \eqref{eq:glearn_prob} have been extensively studied, e.g. in \citep{Banerjee2008, Hsieh2011, Lee2012, Lu2010, Olsen2012, Rolfs2012, Scheinberg2009, Scheinberg2010, Yuan2012}. However, none so far exploits $f\in \mathcal{F}_{2,\nu}$ and feature global convergence guarantees: cf., Sect.~\ref{sec: related}. 

\paragraph{Log-barrier for linear inequalities:}  
The function $f(\xb) := -\log(\mathbf{a}^T\xb-b)$ is a self-concordant barrier with $\dom{f} := \set{ \xb \in \mathbb{R}^n ~|~ \mathbf{a}^T\xb > b}$. 
As a stylized application,  consider the low-light imaging problem in signal processing \citep{Harmany2012}, where the imaging data is collected by counting photons hitting a detector over the time. In this setting, we wish to accurately reconstruct an image in low-light, which leads to noisy measurements due to low photon count levels. We can express our observation model using the Poisson distribution as: 
\begin{equation*}
\mathbb{P}(\yb | \mathcal{A}(\xb)) = \prod_{i=1}^m\frac{(\mathbf{a}_i^T\xb)^{y_i}}{y_i!}e^{-{\mathbf{a}_i^T\xb}},
\end{equation*}
where  ${\xb}$ is the true image, $\mathcal{A}$ is a linear operator that projects the scene onto the set of observations, $\mathbf{a}_i$ is the $i$-th row of $\mathcal{A}$, and $\yb \in \mathbb{Z}^m_{+}$ is a vector of observed photon counts. 

Via the log-likelihood formulation, we stumble upon a composite minimization problem:
\begin{equation}\label{eq:Poisson_prob}
\xb^{*} := \argmin_{\xb\in\mathbb{R}^n}\Big\{\underbrace{\sum_{i=1}^m\mathbf{a}_i^T\xb - \sum_{i=1}^m y_i\log(\mathbf{a}_i^T\xb)}_{ =: {f}(\xb)} + g(\xb)\Big\},
\end{equation}
where ${f}(\xb) $ is self-concordant (but not standard). In the above formulation, the typical image priors $g(\xb)$ include the $\ell_1$-norm for sparsity in a known basis, total variation semi-norm of the image, and the positivity of the image pixels.  While the formulation \eqref{eq:Poisson_prob} seems specific to imaging, it is also common in sparse regression with unknown noise variance \citep{Stadler2012}, heteroschedastic LASSO \citep{Dalalyan2013}, barrier approximations of, e.g., the Dantzig selector \citep{Candes2007} and quantum tomography \citep{Banaszek1999} as well.

The current state of the art solver is called SPIRAL-TAP \citep{Harmany2012}, which biases the logarithmic term (i.e., $\log(\mathbf{a}_i^T\xb + \varepsilon)\rightarrow\log(\mathbf{a}_i^T\xb)$, where $\varepsilon \ll 1$) and then applies non-monotone composite gradient descent algorithms for $\mathcal{F}_L$ with a Barzilai-Borwein step-size as well as other line-search strategies. 

\paragraph{Logarithm of concave quadratic functions:}
The function $f(\xb) := - \log\left(\sigma^2 - \|\mathbf{A}\xb-\yb\|^2_2\right)$ is self-concordant with $\dom{f} := \set{ \xb \in \mathbb{R}^n ~|~ \|\mathbf{A}\xb - \yb\|_2^2 < \sigma^2}$. As a stylized application, we consider the basis pursuit denoising (BPDN) formulation \citep{Berg2008} as:
\begin{equation}\label{eq: MRI}
  \xb^{*} := \arg\!\min_{\xb \in \mathbb{R}^n} \set{g(\xb) ~|~  \norm{\mathbf{A}\xb - \yb}_2^2 \le \sigma^2}.
\end{equation}
The BPDN criteria is commonly used in magnetic resonance imaging (MRI)
where $\mathbf{A}$ is a subsampled Fourier operator, $\yb$ is the MRI scan data, and $\sigma^2$ is a known machine noise level (i.e., obtained during a pre-scan). In \eqref{eq: MRI}, $g$ is an image prior, e.g., similar to the Poisson imaging problem. Approximate solutions to \eqref{eq: MRI} can be obtained via a barrier formulation:
\begin{equation}
  \xb_t^{*} := \argmin_{\xb \in \mathbb{R}^n}\Big\{ \underbrace{- t \log\left(\sigma^2 - \norm{\mathbf{A}\xb-\yb}_2^2\right)}_{ =: {f}(\xb)} ~+ ~g(\xb) \Big\},
\end{equation}
where $t > 0$ is a penalty parameter which controls the quality of the approximation. 
The BPDN formulation is quite generic and has several other applications in statistical regression, geophysics, and signal processing.

Several different approaches solve the BPDN problem \eqref{eq: MRI}, some of which require projections onto the constraint set, including Douglas-Rachford splitting, proximal methods, and the SPGL$_1$ method \citep{Berg2008,Combettes2005}.

\subsection{Related work}\label{sec: related}
Our attempt is to briefly describe the work that revolves around \eqref{eq:COP} with the main assumptions of $f\in \mathcal{F}_L$ and the proximal operator of $g$ being computationally tractable. In fact, Douglas-Rachford splitting methods can obtain numerical solutions to \eqref{eq:COP}  when the self-concordant functions are endowed with tractable proximal maps. However, it is computationally  easier to calculate the gradient of $f \in \mathcal{F}_2$ than their proximal maps.

One of the main approaches in this setting is based on operator splitting.  By presenting the optimality condition of problem \eqref{eq:COP} as an inclusion of two monotone operators, one can apply splitting techniques, such as forward-backward or Douglas-Rachford methods,  to solve the resulting monotone inclusion \citep{BricenoArias2011,Facchinei2003,Goldstein2009}. In our context, several variants of this approach have been studied.  For example, projected gradient  or proximal-gradient methods and fast proximal-gradient methods have been considered, see, e.g., \citep{Beck2009,Mine1981,Nesterov2007}.  In all these methods, the main assumption required to prove the convergence is the global Lipschitz continuity of the gradient of the smooth function $f$. Unfortunately, when $ f\notin \mathcal{F}_L$ but $f \in \mathcal{F}_2$, these theoretical results on the global convergence and the global convergence rates are no longer applicable.

Other mainstream approaches for \eqref{eq:COP} include augmented Lagrangian and alternating techniques: cf., \citep{Boyd2011,Goldfarb2012}. These methods have empirically proven to be quite powerful in specific applications. The main disadvantage of these methods is the manual tuning of the penalty parameter in the augmented Lagrangian function, which is not yet well-understood for general problems. Consequently, the analysis of global convergence as well as the convergence rate is an issue since the performance of the algorithms strongly depends on the choice of this penalty parameter in practice. Moreover, as indicated in a recent work  \citep{Goldstein2012}, alternating direction methods of multipliers as well as alternating linearization methods can be viewed as splitting methods in the convex optimization context. Hence, it is unclear if this line of work is likely to lead to any rigorous guarantees when $f\in \mathcal{F}_2$.

An emerging direction for solving composite minimization problems \eqref{eq:COP} is based on the proximal-Newton method. The origins of this method can be traced back to the work of \citep{Bonnans1994}, which relies on the concept of \textit{strong regularity} introduced by \citep{Robinson1980} for generalized equations. In the convex case, this method has been studied by several authors such as \citep{Becker2012b,Lee2012,Schmidt2011}.  So far, methods along this line are applied to solve a generic problem of the form \eqref{eq:COP} even when $f\in \mathcal{F}_2$. The convergence analysis of these methods is encouraged by standard Newton methods and requires the strong  regularity of the Hessian of $f$ near the optimal solution (i.e., $\mu\mathbb{I} \preceq \nabla^2f(\xb) \preceq L\mathbb{I}$). This assumption used in \citep{Lee2012} is stronger than assuming $\nabla^2f(\xb^{*})$ to be positive definite at the solution $\xb^{*}$ as in our approach below. 
Moreover, the global convergence can only be proved by applying a certain globalization strategy such as line-search \citep{Lee2012} or trust-region. Unfortunately, none of these assumptions can be verified before the algorithm execution for the intended applications. 
By exploiting the self-concordance concept, we can show the global convergence of proximal-Newton methods without any globalization strategy (e.g., linesearch or trust-region approach).

\subsection{Our contributions} 
Interior point methods are always an option while solving the self-concordant composite problems \eqref{eq:COP} numerically by means of disciplined convex programming \citep{Grant2006,Lofberg2004}. More concretely, in the IPM setting, we set up an equivalent problem to \eqref{eq:COP} that typically avoids the non-smooth term $g(x)$ in the objective by lifting the problem dimensions with slack variables and introducing additional constraints. The new constraints may then be embedded into the objective through a barrier function. We then solve a sequence of smooth problems (e.g., with Newton methods) and ``path-follow''\footnote{It is also referred to as a homotopy method.} to obtain an accurate solution \citep{Nemirovski2009,Nesterov2004}.  In this loop, many of the underlying structures within the original problem, such as sparsity,  can be lost due to pre-conditioning or Newton direction scaling (e.g., Nesterov-Todd scaling, \cite{Nesterov1997}). The efficiency and the memory bottlenecks of the overall scheme then heavily depends on the workhorse algorithm that solves the smooth problems.

In stark contrast, we introduce an algorithmic framework that directly handles the composite minimization problem \eqref{eq:COP} without increasing the original problem dimensions. For problems of larger dimensions, this is the main argument in favor of our approach. Instead of solving a sequence of smooth problems, we solve a sequence of non-smooth proximal problems with a variable metric (i.e., our workhorse). Fortunately, these proximal problems feature the composite form \eqref{eq:COP} with a Lipschitz  gradient (and oft-times strongly convex) smooth term. Hence, we leverage the tremendous amount of research (cf., Table \ref{table: taxonomy}) done over the last decades. Surprisingly, we can even retain the original problem structures that lead to computational ease in many cases (e.g., see Section \ref{sec:app_graph_select}). 

Our specific contributions can be summarized as follows:
\begin{enumerate}
  \item We propose a new \textit{variable metric} framework for minimizing the sum $f+g$ of a self-concordant function $f$ and a convex, possibly nonsmooth function $g$. Our approach relies on the solution of a convex subproblem obtained by linearizing and regularizing the first term $f$. To achieve monotonic descent, we develop a new set of \textit{analytic} step-size selection and correction procedures based on the structure of the problem. 
  
  \item We establish both the global and the local convergence of different variable metric strategies. We first derive an expected result: when the variable metric is the Hessian $\nabla^2f(\xb^k)$ of $f$ at  iteration $k$, the resulting algorithm locally exhibits quadratic convergence rate within an explicit region. We then show   that variable metrics satisfying the Dennis-Mor\'{e}-type condition \citep{Dennis1974} exhibit superlinear convergence. 
  
  \item We pay particular attention to diagonal variable metrics as many of the proximal subproblems can be solved exactly (i.e., in closed form). We derive conditions on when these variants achieve locally linear convergence. 

 \item We apply our algorithms to the aforementioned real-world and synthetic problems to highlight the strengths and the weaknesses of our scheme. For instance, in the graph learning problem \eqref{eq:glearn_prob}, our framework can avoid matrix inversions as well as Cholesky decompositions in learning graphs. In Poisson intensity reconstruction \eqref{eq:Poisson_prob}, up to around $80\times$ acceleration is possible over the state-of-the-art solver. 
\end{enumerate}

We highlight three key practical contributions to numerical optimization. First, in the proximal-Newton method, our analytical step-size procedures allow us to do away with any globalization strategy (e.g., line-search). This has a significant practical impact when the evaluation of the functions is expensive. We show how to combine the analytical step-size selection with the standard backtracking or forward line-search procedures to enhance the global convergence of our method. Our analytical quadratic convergence characterization helps us adaptively switch from \textit{damped} step-size to a \textit{full} step-size. Second, in the proximal-gradient method setting, we establish a step-size selection and correction mechanism. The step-size selection procedure can be considered as a predictor, where existing step-size rules that leverage local information can be used. The step-size corrector then adapts the local information of the function to achieve the best theoretical decrease in the objective function. While our procedure does not require any function evaluations, we can further enhance convergence whenever we are allowed function evaluations.  Finally, our framework, as we demonstrate in \citep{TranDinh2013e}, accommodates a path-following strategy, which enable us to approximately solve constrained non-smooth convex minimization problems with rigorous guarantees. 

\paragraph{Paper outline.} In Section \ref{sec:math_tool}, we first recall some fundamental concepts of convex optimization and self-concordant functions used in this paper. Section \ref{sec:alg_framework} presents our algorithmic framework using three different instances with convergence results, complexity estimates and modifications.  Section \ref{sec:applications} deals with three concrete instances of our algorithmic framework.  Section \ref{sec:num_experiment} provides numerical experiments to illustrate the impact of the proposed methods. Section \ref{sec:conclude} concludes the paper.

\section{Preliminaries}\label{sec:math_tool}

\paragraph{Notation:}
We reserve lower-case and bold lower-case letters for scalar and vector representation, respectively. Upper-case bold letters denote matrices. We denote $\mathbb{S}^p_{+}$ (reps., $\mathbb{S}^p_{++}$) for the set of symmetric positive definite (reps., positive semidefinite) matrices of size $p\times p$. For a proper, lower semicontinuous convex function $f$ from $\mathbb{R}^n$ to $\mathbb{R}\cup\set{+\infty}$, we denote its domain by $\dom{f}$, i.e., $\dom{f} := \set{\xb \in\mathbb{R}^n~|~ f(\xb) < +\infty}$ (see, e.g., \cite{Rockafellar1970}).

\paragraph{Weighted norm and local norm:}
Given a matrix $\Hb \in \mathbb{S}^n_{++}$, we define the weighted norm $\norm{\xb}_{\Hb} := \sqrt{\xb^T\Hb\xb}$, $\forall \xb\in\mathbb{R}^n$; its dual norm is defined as $\norm{\xb}^{*}_{\Hb} :=
\max_{\small{\norm{\bf y}_\Hb\leq 1}}\yb^T\xb = \sqrt{\xb^T\Hb^{-1}\xb}$.  Let $f \in \mathcal{F}_2$ and $\xb\in \dom{f}$ so that $\nabla^2f(\xb)$ is positive definite.
For a given vector $\vb \in \mathbb{R}^{n}$, the local norm around $\xb \in \dom{f}$ with respect to $f$ is defined as $\norm{{\bf
v}}_\xb := \left(\vb^T\nabla^2f(\xb)\vb\right)^{1/2}$, while the corresponding dual norm is given by $\norm\vb_{\xb}^{*} =\left(\vb^T\nabla^{2} f(\xb)^{-1}\vb\right)^{1/2}$. 

\paragraph{Subdifferential and subgradient:}  
Given a proper, lower semicontinuous convex function, we define the subdifferential of $g$ at $\xb \in \dom{g}$ as 
\begin{equation*}
\partial{g}(\xb) := \set{\vb \in \mathbb{R}^n ~|~ g(\yb) - g(\xb) \geq \vb^T(\yb - \xb), ~\forall \yb\in\dom{g}}.
\end{equation*} 
If $\partial{g}(\xb)\neq\emptyset$ then each element in $\partial{g}(\xb)$ is called a subgradient of $g$ at $\xb$.
In particular, if $g$ is differentiable, we use $\nabla g(\xb)$ to denote its derivative at $\xb\in\dom{g}$, and $\partial{g}(\xb) \equiv \set{\nabla{f}(\xb)}$.

\paragraph{Proximity operator:}  
A basic tool to handle the nonsmoothness of a convex function $g$ is its proximity operator (or proximal operator) $\mathrm{prox}^{\Hb}_g$, whose definition is given in Section \ref{sec:intro}. For notational convenience in our derivations, we alter this definition in the sequel as follows: Let $g$ be a proper lower semicontinuous and convex in $\mathbb{R}^n$ and $\Hb \in \mathbb{S}_{+}^n$.
We define 
\begin{equation}\label{eq:PgH}
P^g_{\Hb}(\ub) := \mathrm{arg}\!\min_{\xb\in \mathbb{R}^n}\set{g(\xb) + \frac{1}{2} \xb^T\Hb\xb - \ub^T\xb},~~\forall \ub \in \mathbb{R}^n, 
\end{equation}
as the proximity operator for the nonsmooth $g$, which has the following properties. 

\begin{lemma}\label{le:Pg_properties}
Assume that $\Hb\in\mathbb{S}^n_{++}$. Then, the operator $P^g_{\Hb}$ in \eqref{eq:PgH} is single-valued and satisfies the following property:
\begin{equation}\label{eq:P_g_property}
(P^g_{\Hb}(\ub) - P^g_{\Hb}(\vb))^T(\ub-\vb) \geq \norm{P^g_{\Hb}(\ub) - P^g_{\Hb}(\vb)}_{\Hb}^2,
\end{equation}
for all $\ub, \vb \in \mathbb{R}^n$. Consequently, $P^g_{\Hb}$ is a nonexpansive mapping, i.e., 
\begin{equation}\label{eq:P_g_property2}
\norm{P^g_{\Hb}(\ub) - P^g_{\Hb}(\vb)}_{\Hb} \leq \norm{\ub-\vb}_{\Hb}^{*}. 
\end{equation}
\end{lemma}

\begin{proof}
The single-valuedness of $P^g_{\Hb}$ is obvious due to the strong convexity of the objective function in \eqref{eq:PgH}.
Let $\boldsymbol{\xi}_{\ub} := P^g_{\Hb}(\ub)$ and $\boldsymbol{\xi}_{\vb} := P^g_{\Hb}(\vb)$. By the definition of $P^g_{\Hb}$, we have $\ub - \Hb\boldsymbol{\xi}_{\ub} \in \partial{g}(\boldsymbol{\xi}_{\ub})$ and $\vb - \Hb\boldsymbol{\xi}_{\ub} \in \partial{g}(\boldsymbol{\xi}_{\vb})$. Since $g$ is convex, we have $\left(\ub - \Hb\boldsymbol{\xi}_{\ub} - (\vb - \Hb\boldsymbol{\xi}_{\vb})\right)^T(\boldsymbol{\xi}_{\ub} - \boldsymbol{\xi}_{\vb}) \geq 0$. This inequality leads to $(\ub - \vb)^T(\boldsymbol{\xi}_{\ub} - \boldsymbol{\xi}_{\vb})  \geq (\boldsymbol{\xi}_{\ub} - \boldsymbol{\xi}_{\vb})^T\Hb(\boldsymbol{\xi}_{\ub} - \boldsymbol{\xi}_{\vb}) = \norm{\boldsymbol{\xi}_{\ub} - \boldsymbol{\xi}_{\vb}}_{\Hb}^2$ which is indeed \eqref{eq:P_g_property}. 
Via the generalized Cauchy-Schwarz inequality,  \eqref{eq:P_g_property} leads to \eqref{eq:P_g_property2}.
\end{proof}

\paragraph{Key self-concordant bounds:} 
Based on \citep[Theorems 4.1.7 and 4.1.8]{Nesterov2004}, for a given standard self-concordant function $f$, we recall the following inequalities
\begin{eqnarray}\label{eq:SC_bounds}
&\omega(\norm{\yb - \xb}_\xb) + \nabla{f}(\xb)^T(\yb - \xb) + f(\xb) \leq f(\yb),\label{eq:SC_bound1}\\
&f(\yb) \leq f(\xb) + \nabla{f}(\xb)^T(\yb - \xb) + \omega_{*}(\norm{\yb - \xb}_\xb),\label{eq:SC_bound2}
\end{eqnarray}
where $\omega : \mathbb{R}\to\mathbb{R}_{+}$ is defined as $\omega(t) := t - \ln(1+t)$ and $\omega_{*} : [0,1]\to\mathbb{R}_{+}$ is defined as $\omega_{*}(t) := -t - \ln(1-t)$.  These functions
 are both nonnegative, strictly convex and increasing. Hence, \eqref{eq:SC_bound1} holds for all $\xb, \yb\in\dom{f}$, and \eqref{eq:SC_bound2} holds for all $\xb, \yb\in\dom{f}$ such that
$\norm{\yb - \xb}_\xb < 1$. In contrast to the ``global'' inequalities for the function classes $\mathcal{F}_L$ and $\mathcal{F}_\mu$  (cf., Fig.\ \ref{fig: smooth function classes}), the self-concordant inequalities are based on ``local'' quantities. Moreover, these bounds are no longer quadratic which prevents naive applications of the methods from $\mathcal{F}_{\mu, L}$.

\section{Composite self-concordant optimization}\label{sec:alg_framework}
In this section, we propose a \textit{variable metric} optimization framework that rigorously trades off computation and accuracy of solutions without transforming  \eqref{eq:COP}  into a higher dimension smooth convex optimization problem. We assume theoretically that the proximal subproblems can be solved exactly. However, our theory can be analyze for the inexact case, when we solve these problems up to a sufficiently high accuracy (typically, it is at least higher than (e.g., $0.1\varepsilon$) the desired accuracy $\varepsilon$ of \eqref{eq:COP} at the few last iterations), see, e.g., \citep{TranDinh2012c,TranDinh2013e}. In our theoretical characterizations, we only rely on the following assumption:

\begin{assumption}\label{as:A1}
The function $f$ is \textit{convex} and standard self-concordant (see Definition \ref{de:concordant}). The function $g$ from $\mathbb{R}^n$ to $\mathbb{R}\cup\set{+\infty}$ is proper, lower semicontinuous, convex and possibly nonsmooth with a tractable proximity operator.
\end{assumption}

\paragraph{Unique solvability of \eqref{eq:COP} and its optimality condition:}
First, we  show that problem \eqref{eq:COP} is uniquely solvable.
The proof of this lemma can be done similarly as \citep[Theorem 4.1.11]{Nesterov2004} and is provided in the appendix.

\begin{lemma}\label{le:unique_solution}
Suppose that the functions $f$ and $g$ of problem \eqref{eq:COP} satisfy Assumption $\mathbf{A}.\ref{as:A1}$. Let $\lambda(\xb) := \norm{\nabla{f}(\xb) + \vb}_{\xb}^{*} < 1$, for some $\xb\in\dom{F}$ and ${\bf v}\in\partial{g}(\xb)$. 
Then the solution $\xb^{*}$ of \eqref{eq:COP} exists and is unique. 
\end{lemma}
Since this problem is convex, the following optimality condition is necessary and sufficient:
\begin{equation}\label{eq:Fx_optimality}
\mathbf{0} \in \nabla{f}(\xb^{*}) + \partial{g}(\xb^{*}).
\end{equation}
The solution $\xb^{*}$ is called \textit{strongly regular} if $\nabla^2f(\xb^{*}) \succ 0$. In this case, $\infty> \sigma_{\max}^{*} \ge \sigma_{\min}^{*}  >  0$, where $\sigma_{\min}^{*}$ and $\sigma_{\max}^{*}$ are the smallest and the largest eigenvalue of $\nabla^2f(\xb^{*})$, respectively.  

\paragraph{Fixed-point characterization:}
Let $\Hb \in \mathbb{S}_{+}^n$. We define $S_{\Hb}(\xb) := \Hb\xb - \nabla{f}(\xb)$. Then, from \eqref{eq:Fx_optimality}, we have 
\begin{equation*}
S_{\Hb}(\xb^{*}) \equiv \Hb \xb^{*} - \nabla{f}(\xb^{*}) \in \Hb \xb^{*} + \partial{g}(\xb^{*}).
\end{equation*}
By using the definition of $P^g_{\Hb}(\cdot)$ in \eqref{eq:PgH}, one can easily derive the fixed-point expression
\begin{equation}\label{eq:fixed_point}
\xb^{*} = P^g_{\Hb}\left(S_{\Hb}(\xb^{*})\right),
\end{equation} 
that is, $\xb^{*}$ is the fixed-point of the mapping $R^g_{\Hb}(\cdot)$, where $R^g_{\Hb}(\cdot) := P^g_{\Hb}(S_{\Hb}(\cdot))$.
The formula in \eqref{eq:fixed_point} suggests that we can generate an iterative sequence based on the fixed-point principle, i.e., $\xb^{k+1} := R^g_{\Hb}(\xb^k)$ starting from $\xb^0\in\dom{F}$ for $k\geq 0$. 
Theoretically, under certain assumptions, one can ensure that the mapping $R^g_{\Hb}$ is contractive and the sequence generated by this scheme is convergent. 

We note that if $g \equiv 0$ and $\Hb\in\mathbb{S}^n_{++}$, then $P^g_{\Hb}$ defined by \eqref{eq:PgH} reduces to $P^g_{\Hb}(\cdot) = \Hb^{-1}(\cdot)$. Consequently, the fixed-point formula \eqref{eq:fixed_point} becomes $\xb^{*} = \xb^{*} - \Hb^{-1}\nabla{f}(\xb^{*})$, which is equivalent to $\nabla{f}(\xb^{*}) = 0$.

\paragraph{Our variable metric framework:}
Given a point $\xb^k \in \dom{F}$ and a symmetric positive semidefinite matrix $\Hb_k$, we consider the  function 
\begin{equation}\label{eq:F_surrogate}
Q(\xb; \xb^k, \Hb_k) := f(\xb^k) + \nabla{f}(\xb^k)^T(\xb - \xb^k) + \frac{1}{2}(\xb - \xb^k)^T\Hb_k(\xb - \xb^k),
\end{equation} 
for $\xb \in \dom{F}$. 
The function $Q(\cdot;\xb^k, \Hb_k)$ is---seemingly---a quadratic approximation of $f$ around $\xb^k$. 
Now, we study the following  scheme to generate a sequence $\set{\xb^k}_{k\geq 0}$:
\begin{equation}\label{eq:iterative_scheme}
\xb^{k+1} := \xb^k + \alpha_k\db^k,
\end{equation}
where $\alpha_k \in (0, 1]$ is a  step size and $\db^k$ is a search direction.

Let  $\mathbf{s}^k$ be a solution of the following problem:
\begin{equation}\label{eq:cvx_subprob}
\mathbf{s}^k \in \mathcal{S}(\xb^k, \Hb_k) := \argmin_{\xb\in\mathrm{dom}(F)}\set{Q(\xb; \xb^k, \Hb_k) + g(\xb)} = P^g_{\Hb_k}\left(\Hb_k\xb^k - \nabla{f}(\xb^k)\right).
\end{equation}
Since we do not assume that $\Hb_k$ to be positive definite, the solution $\sb^k$ may not exist. We require the following assumption:

\begin{assumption}\label{as:A2}
The subproblem \eqref{eq:cvx_subprob} has at least one solution $\sb^k$, i.e., $\mathcal{S}(\xb^k, \Hb_k)\neq\emptyset$.
\end{assumption}
In particular, if $\Hb_k\in\mathbb{S}^n_{++}$, then the solution $\sb^k$ of \eqref{eq:cvx_subprob} exists and is unique, i.e., $\mathcal{S}(\xb^k, \Hb_k) = \set{\sb^k}\neq\emptyset$. 
Up to now, we have not required the uniqueness of $\sb^k$. This assumption will be specified later in the next sections.
Throughout this paper, we assume that both Assumptions $\mathbf{A}.\ref{as:A1}$ and $\mathbf{A}.\ref{as:A2}$ are satisfied without referring to them specifically.

\noindent Now, given $\sb^k$, the  direction $\db^k$ is computed as
 \begin{equation}\label{eq:search_dir_dk}
 \db^k := \mathbf{s}^k - \xb^k. 
 \end{equation}
 If we define $\mathbf{G}_k := \Hb_k \db^k$, then $\mathbf{G}_k$ is called the \textit{gradient mapping} of \eqref{eq:COP} \citep{Nesterov2004}, which behaves similarly as gradient vectors in non-composite minimization. 
 Since problem \eqref{eq:cvx_subprob} is  solvable due to Assumption $\mathbf{A}.\ref{as:A2}$, we can write its optimality condition as
\begin{equation}\label{eq:subprob_opt}
\mathbf{0} \in \nabla{f}(\xb^k) + \Hb_k(\mathbf{s}^k  - \xb^k) + \partial{g}(\mathbf{s}^k).
\end{equation} 
It is easy to see that if $\db^k = 0$, i.e., $\sb^k \equiv \xb^k$,  then \eqref{eq:subprob_opt} reduces to $0 \in \nabla{f}(\xb^k) + \partial{g}(\xb^k)$, which is exactly \eqref{eq:Fx_optimality}. Hence, $\xb^k$ is a solution of \eqref{eq:COP}.

In the variable metric framework, depending on the choice of $\Hb_k$, the iteration scheme \eqref{eq:iterative_scheme} leads to different methods for solving \eqref{eq:COP}. For instance,
\begin{enumerate}
\item If $\Hb_k := \nabla^2f(\xb^k)$, then the method \eqref{eq:iterative_scheme} is a \textit{proximal-Newton} method. 
\item If $\Hb_k$ is a symmetric positive definite matrix approximation of $\nabla^2f(\xb^k)$, then the method \eqref{eq:iterative_scheme} is a
\textit{proximal-quasi Newton} method.
\item If $\Hb_k := L_k\mathbb{I}$, where $L_k$ is, say, an approximation for the local Lipschitz constant of $f$ 
and $\mathbb{I}$ is the identity matrix, then the method \eqref{eq:iterative_scheme} is a \textit{proximal-gradient} method. 
\end{enumerate}
Many of these above methods have been studied for \eqref{eq:COP} when $f\in\mathcal{F}_L$: cf., \citep{Beck2009,Becker2012b,Chouzenoux2013a,Lee2012}. Note however that, since the self-concordant part $f$ of $F$ is not (necessarily) globally Lipschitz continuously differentiable, these approaches are generally not applicable in theory.

Given the search direction $\db^k$ defined by \eqref{eq:search_dir_dk}, we define the following proximal-Newton decrement\footnote{This notion is borrowed from standard the Newton decrement defined in \cite[Chapter 4]{Nesterov2004}.}  $\lambda_k$ and the weighted [semi-]norm $\beta_k$:
\begin{equation}\label{eq:prox_Newton_decrement}
\lambda_k := \Vert\db^k\Vert_{\xb^k}  = \left((\db^k)^T\nabla^2f(\xb^k)\db^k\right)^{1/2} ~ \textrm{and} ~\beta_k := \Vert\db^k\Vert_{\Hb_k}.
\end{equation}
In the sequel, we study three different instances of the variable metric strategy in detail.

\begin{remark}\label{re:lambda_quantity}
If $g \equiv 0$ and $\nabla^2f(\xb^k)\in\mathbb{S}^n_{++}$, then $\db^k = -\nabla^2f(\xb^k)^{-1}\nabla{f}(\xb^k)$ is the standard Newton direction. In this case, $\lambda_k$ defined by \eqref{eq:prox_Newton_decrement} reduces to $\lambda_k \equiv \Vert\nabla{f}(\xb^k)\Vert_{\xb^k}^{*}$, the Newton decrement defined in \cite[Chapter 4]{Nesterov2004}. Moreover, we have $\lambda_k \equiv \lambda(\xb^k)$, as defined in Lemma \ref{le:unique_solution}.
\end{remark}

\subsection{A proximal-Newton method}\label{subsec:prox_newton_method}
If we choose $\Hb_k := \nabla^2f(\xb^k)$, then the method described in \eqref{eq:iterative_scheme} is called  the \textit{proximal Newton} algorithm. 
For notational ease, we redefine $\sb^k_n := \sb^k$ and $\db^k_n := \db^k$, where the subscript $n$ is used to distinguish proximal Newton related quantities from the other variable metric strategies. Moreover, we use the shorthand notation $P^g_{\bar{\xb}}
:= P^g_{\nabla^2{f}(\bar{\xb})}$, whenever $\bar{\xb}\in\dom{f}$.
Using \eqref{eq:cvx_subprob} and \eqref{eq:search_dir_dk}, $\sb^k_n$ and $\db^k_n$ are given by
\begin{equation}\label{eq:cvx_subprob_xk}
\sb^k_n := P^g_{\xb^k}\left( \nabla^2f(\xb^k)\xb^k - \nabla{f}(\xb^k)\right), ~~ \db^k_n : = \sb^k_n - \xb^k.
\end{equation}
Then, the proximal-Newton method generates a sequence $\set{\xb^k}_{k\geq 0}$ starting from $\xb^0\in\dom{F}$ according to
\begin{equation}\label{eq:DPNM}
\xb^{k+1} := \xb^k + \alpha_k\db^k_n,
\end{equation}
where $\alpha_k \in (0, 1]$ is a step size. 
If $\alpha_k < 1$, then the iteration \eqref{eq:DPNM} is called the \textit{damped proximal-Newton} iteration. 
If $\alpha_k = 1$, then it is called the \textit{full-step proximal-Newton} iteration. 

\paragraph{Global convergence:}
We first show that with an appropriate choice of the step-size $\alpha_k \in (0,1]$, the iterative sequence $\set{\xb^k}_{k\geq 0}$ generated by the damped-step proximal Newton scheme \eqref{eq:DPNM} is a decreasing sequence; i.e., $F(\xb^{k+1}) \leq F(\xb^k) - \omega(\sigma)$ whenever $\lambda_k \geq \sigma$, where $\sigma > 0$ is fixed. 
The following theorem provides an explicit formula for the step size $\alpha_k$ whose proof can be found in the appendix.

\begin{theorem}\label{th:choose_alpha}
If $\alpha_k := \frac{1}{1 + \lambda_k} \in (0, 1]$, then the scheme in \eqref{eq:DPNM} generates $\xb^{k+1}$ satisfies:
\begin{equation}\label{eq:decrease_eq2}
F(\xb^{k+1}) \leq F(\xb^k) - \omega(\lambda_k).
\end{equation}
Moreover, the step $\alpha_k$ is optimal.
The number of iterations to reach the point $\xb^k$ such that $\lambda_k < \sigma$ for some $\sigma \in (0, 1)$  is $k_{\max} :=   \left\lfloor\frac{F(\xb^0)-F(\xb^{\ast})}{\omega(\sigma)}\right\rfloor + 1$.
\end{theorem}

\paragraph{Local quadratic convergence rate:}
We now establish the local quadratic convergence of the scheme \eqref{eq:DPNM}. A complete  proof of this theorem can be found in the appendix.

\begin{theorem}\label{th:quad_converg_DPNM}
Assume that  $\xb^{*}$ is the unique solution of \eqref{eq:COP} and is strongly regular.
Let $\set{\xb^k}_{k\geq 0}$ be a sequence generated by the proximal Newton scheme \eqref{eq:DPNM} with $\alpha_k \in (0, 1]$. 
Then:
\begin{itemize}
\item[$\mathrm{a)}$] If $\alpha_k\lambda_k < 1-\frac{1}{\sqrt{2}}$, then it holds that 
\begin{equation}\label{eq:DPNM_estimate}
\lambda_{k+1} \leq \left(\frac{1-\alpha_k + (2\alpha_k^2 - \alpha_k)\lambda_k}{1 - 4\alpha_k\lambda_k + 2\alpha_k^2\lambda_k^2}\right)\lambda_k.
\end{equation}

\item[$\mathrm{b)}$] If the sequence $\set{\xb^k}_{k\geq 0}$ is generated by the damped proximal-Newton scheme \eqref{eq:DPNM}, starting from $\xb^0$ such that $\lambda_0 \leq \bar{\sigma} := \sqrt{5} - 2 \approx 0.236068$ and $\alpha_k := (1 + \lambda_k)^{-1}$, then $\set{\lambda_k}_k$ locally converges to $0^{+}$ at a quadratic rate.

\item[$\mathrm{c)}$]
Alternatively, if the sequence $\set{\xb^k}_{k\geq 0}$ is generated by the full-step proximal-Newton scheme \eqref{eq:DPNM} starting from $\xb^0$ such that $\lambda_0 \leq \bar{\sigma} := 0.25(5-\sqrt{17}) \approx 0.219224$ and $\alpha_k = 1$, then $\set{\lambda_k}_k$ locally converges to $0^{+}$ at a quadratic rate.
\end{itemize}
Consequently, the sequence $\set{\xb^k}_{k\geq 0}$ also locally converges to $\xb^{\ast}$ at a quadratic rate in both cases $\mathrm{b)}$ and $\mathrm{c)}$, i.e., $\set{\Vert\xb^k - \xb^{*}\Vert_{\xb^{*}}}_{k\geq 0}$ locally converges to $0^{+}$ at a quadratic rate.
\end{theorem}

\paragraph{A two-phase algorithm for solving \eqref{eq:COP}:}
Now, by the virtue of the above analysis, we can propose a two-phase proximal-Newton algorithm for solving \eqref{eq:COP}. 
Initially, we perform the damped-step proximal-Newton iterations until we reach the quadratic convergence region (Phase 1). Then, we perform full-step proximal-Newton iterations, until we reach the desired accuracy (Phase 2).  The pseudocode of the algorithm is presented in Algorithm \ref{alg:A2}.

\begin{algorithm}[!ht]\caption{(\textit{Proximal-Newton algorithm})}\label{alg:A2}
\begin{algorithmic}
   \STATE {\bfseries Inputs:} $\xb^0\in\dom{F}$, tolerance $\varepsilon > 0$.
   \STATE {\bfseries Initialization:} Select a constant $\sigma \in (0, \frac{(5-\sqrt{17})}{4}]$, e.g., $\sigma := 0.2$.
   \STATE \hrulefill
   \FOR{$k = 0$ {\bfseries to} $K_{\max}$}
   \STATE 1. Compute the proximal-Newton search direction $\db^k_n$ as in \eqref{eq:cvx_subprob_xk}.
   \STATE 2. Compute $\lambda_k := \norm{\db^k_n}_{\xb^k}$.
   \STATE 3. \textbf{if}~$\lambda_k > \sigma$ \textbf{then} $\xb^{k + 1} := \xb^k + \alpha_k\db^k_n$, where $\alpha_k := (1 + \lambda_k)^{-1}$.
   \STATE 4. \textbf{elseif}~$\lambda_k > \varepsilon$ \textbf{then} $\xb^{k+1} := \xb^k + \db^k_n$.   
   \STATE 5. \textbf{else} terminate.
   \ENDFOR
\end{algorithmic}
\end{algorithm}

The radius $\sigma$ of the quadratic convergence region in Algorithm \ref{alg:A2} can be fixed at any value in $(0, \bar{\sigma}]$, e.g., at its upper bound $\bar{\sigma}$. 
An upper bound $K_{\max}$ of the iterations can also be specified, if necessary. 
The computational bottleneck in Algorithm \ref{alg:A2} is typically incurred Step 1 in Phase 1 and Phase 2, where we need to solve the subproblem \eqref{eq:cvx_subprob} to obtain a search direction $\db^k_n$. 
When problem \eqref{eq:cvx_subprob} is strongly convex, i.e., $\nabla^2f(\xb^k)\in\mathbb{S}^n_{++}$, one can apply first order methods to efficiently solve this problem with a linear convergence rate (see, e.g., \cite{Beck2009,Nesterov2004,Nesterov2007}) and make use of a \textit{warm-start} strategy by employing the information of the previous iterations. 

\paragraph{Iteration-complexity analysis.}
The choice of $\sigma$ in Algorithm \ref{alg:A2}  can trade-off the number of iterations between the damped-step and full-step iterations. If we fix $\sigma = 0.2$, then the complexity of the full-step Newton phase becomes $\mathcal{O}\left(\ln\ln\left(\frac{0.28}{\varepsilon}\right)\right)$.
The following theorem summarizes the complexity of the proposed algorithm.

\begin{theorem}\label{th:complexity}
The maximum number of iterations required in Algorithm 1 does not exceed $K_{\max}:= \left\lfloor\frac{F(\xb^0)-F(\xb^{\ast})}{0.017}\right\rfloor + \left\lfloor 1.5\left(\ln\ln\left(\frac{0.28}{\varepsilon}\right)\right) \right\rfloor + 2$ provided that $\sigma = 0.2$ to obtain $\lambda_k\leq \varepsilon$.
Consequently, $\Vert\xb^k - \xb^{*}\Vert_{\xb^{*}} \leq 2\varepsilon$, where $\xb^{\ast}$ is the unique solution of \eqref{eq:COP}.
\end{theorem}

\begin{proof}
Let $\sigma = 0.2$. From the estimate \eqref{eq:DPNM_estimate} of Theorem \ref{th:quad_converg_DPNM} and $\alpha_{k-1} = 1$ we have $\lambda_{k} \leq (1 - 4\lambda_{k-1} + 2\lambda_{k-1}^2)^{-1}\lambda_{k-1}^2$ for $k\geq 1$. Since $\lambda_0 \leq \sigma$, by induction, we can easily show that $\lambda_k \leq (1 - 4\sigma + 2\sigma^2)^{-1}\lambda_{k-1}^2 \leq c\lambda_{k-1}^2$, where $c := 3.57$.
This implies $\lambda_k \leq c^{2^k -1}\lambda_0^{2^k} \leq c^{2^k-1}\sigma^{2^k}$. The stopping criterion $\lambda_k \leq \varepsilon$ in Algorithm \ref{alg:A2} is ensured if $(c\sigma)^{2^k} \leq c\varepsilon$.
Since $c\sigma \approx 0.71 < 1$, the last condition leads to $k \geq (\ln 2)^{-1}\ln\left(\frac{-\ln(c\sigma)}{-\ln(c\varepsilon)}\right)$. By using $c = 3.57$, $\sigma = 0.2$ and the fact that $\ln(2)^{-1} < 1.5$, we can show that the last requirement is fulfilled if $k \geq \left\lfloor 1.5\left(\ln\ln\left(\frac{0.28}{\varepsilon}\right)\right) \right\rfloor + 1$.
Now, combining the last conclusion and Theorem \ref{th:choose_alpha} with noting that $\omega(\sigma) > 0.017$ we obtain $K_{\max}$ as in Theorem \ref{th:complexity}.

Finally, we prove $\Vert\xb^k - \xb^{*}\Vert_{\xb^{*}} \leq 2\varepsilon$. Indeed, we have $\rb_k := \Vert\xb^k - \xb^{*}\Vert_{\xb^{*}} \leq \frac{\Vert\xb^{k+1} - \xb^k\Vert_{\xb^k}}{1-\Vert\xb^k - \xb^{*}\Vert_{\xb^{*}}} + \Vert\xb^{k+1} - \xb^k\Vert_{\xb^{*}} = \frac{\lambda_k}{1 - \rb_k} + \rb_{k+1}$, whenever $\rb_k < 1$. 
Next, using \eqref{eq:thm5_est4} with $\alpha_k = 1$, we have $\rb_{k+1} \leq \frac{(3-\rb_k)\rb_k^2}{1-4\rb_k + 2\rb_k^2}$. Combining these inequalities, we obtain $\frac{(1-\rb_k)(1-7\rb_k + 3\rb_k^2)\rb_k}{1-4\rb_k + 2\rb_k^2} \leq \lambda_k \leq \varepsilon$. Since the function $s(\rb) := \frac{(1-\rb)(1-7\rb + 3\rb^2)\rb}{1-4\rb + 2\rb^2}$ attains a maximum at $\rb^{*} \approx 0.08763$ and it is increasing on $[0, \rb^{*}]$. Moreover, $\frac{(1-\rb_k)(1-7\rb_k + 3\rb_k^2)}{1-4\rb_k + 2\rb_k^2} \geq 0.5$ for $\rb_k \in [0, \rb^{*}]$, which leas to $0.5\rb_k \leq \frac{(1-\rb_k)(1-7\rb_k + 3\rb_k^2)\rb_k}{1-4\rb_k + 2\rb_k^2} \leq \varepsilon$. Hence, $\rb_k \leq 2\varepsilon$ provided that $\rb_k \leq \rb_0 \leq  \rb^{*}\approx 0.08763$.
\end{proof}

\begin{remark}\label{re:g_absent}
When $g \equiv 0$, we can modify the proof of estimate \eqref{eq:DPNM_estimate} to obtain a tighter bound $\lambda_{k+1} \leq \frac{\lambda_k^2}{(1-\lambda_k)^2}$.
This estimate is exactly \cite[]{Nesterov2004}, which implies that the radius of the quadratic convergence region is $\bar{\sigma} := (3-\sqrt{5})/2$.
\end{remark}

\paragraph{A modification of the proximal-Newton method:} 
In Algorithm \ref{alg:A2}, if we remove Step 4 and replace analytic step-size selection calculation in Step 3 with a backtracking line-search, then we reach the proximal Newton method of \citep{Lee2012}. Hence, this approach \emph{in practice} might lead to reduced overall computation since our step-size $\alpha_k$ is selected optimally with respect to the worst case problem structures as opposed to the particular instance of the problem. Since the backtracking approach always starts with the full-step, we also do not need to know whether we are within the quadratic convergence region. Moreover, the cost of evaluating the objective at the full-step in certain applications may not be significantly worse than the cost of calculating $\alpha_k$ or may be dominated by the cost of calculating the Newton direction.

In stark contrast to backtracking, our new theory behooves us to propose a new forward line-search procedure as illustrated by Figure \ref{fig:over-jump}. 
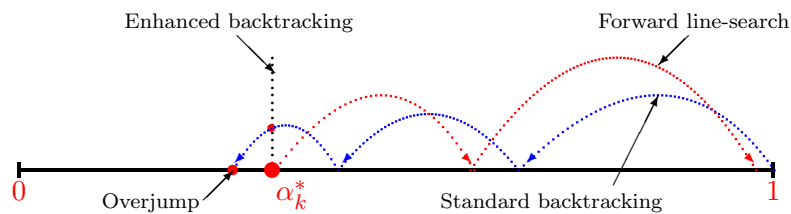
\begin{figure}[ht]
\vskip-1.5cm
\setlength{\unitlength}{1mm}
\begin{picture}(60, 40)
 {\linethickness{0.5mm}\put(20, 0){\line(1, 0){100}} 
  \put(20, -1){\line(0, 1){2}}
  \put(120.2, -1){\line(0, 1){2}}
  {\color{red}\put(19, -4){$0$}}
  {\color{red}\put(118, -4){$1$}}}
  {\color{black}\linethickness{0.4mm}\multiput(51.1,1)(0,1){15}{\circle*{0.5}}}
  {\color{red}\put(49.7,5.7){\circle*{1}}}
  {\color{red}\put(48.5,0){\circle*{2}}\put(49,-4){$\alpha_k^{*}$}}
  {\color{blue}
  {\linethickness{0.3mm}
  \qbezier[60](114, 0)(100, 20)(80, 0)}
  \put(84,4){\linethickness{4mm}\vector(-1,-1){3}}
  {\linethickness{0.3mm}
  \qbezier[50](80, 0)(68, 15)(56, 0)}
  \put(58.5,3){\linethickness{4mm}\vector(-1,-1){2}}
  {\linethickness{0.3mm}
  \qbezier[30](56, 0)(49, 12)(42, 0)}
  \put(44.2,3){\linethickness{4mm}\vector(-1,-1){2}}
  {\color{red}\put(42,0){\circle*{1.5}}}
  }
  {\color{red}
  {\linethickness{0.3mm}
  \qbezier[40](45, 0)(60, 20)(71.5, 0)}
  \put(69.2,3){\linethickness{4mm}\vector(1,-1){2}}
  {\linethickness{0.3mm}
  \qbezier[80](71, 0)(91, 30)(109, 0)}
  \put(106.8,3){\linethickness{4mm}\vector(1,-1){2}}
  }
 \put(25,19){\scriptsize Enhanced backtracking}
 \put(40,18){\linethickness{4mm}\vector(1,-1){5}}
\put(67,-5){\scriptsize Standard backtracking}
\put(100,18){\linethickness{4mm}\vector(-1,-1){4}}
 \put(88,19){\scriptsize Forward line-search}
\put(90,-2){\linethickness{4mm}\vector(1,2){6}}
\put(22,-5){\scriptsize Overjump}
\put(35.5,-4){\linethickness{4mm}\vector(1,1){4}}
\end{picture}
\vskip 0.3cm
\caption{Illustration of step-size selection procedures}\label{fig:over-jump}
\end{figure}  
The idea is quite simple: we start with the ``optimal'' step-size $\alpha_k$ and increase it towards full-step with a stopping condition based on the objective evaluations. Interestingly, when we analytically calculate the step, we also have access to the side information on whether or not we are within the quadratic convergence region, and hence, we can automatically switch to Step 4 in Algorithm \ref{alg:A2}. Alternatively, calculation of the analytic step-size can enhance backtracking since the knowledge of $\alpha_k$ reduces the backtracking range from $(0,1]$ to $(\alpha_k,1]$ with the side-information as to when to automatically take the full-step without function evaluation. 

\subsection{A proximal quasi-Newton scheme}
Even if the function $f$ is self-concordant, the numerical evaluation of $\nabla^2f(\xb)$ can be expensive in many applications (e.g., $f(\xb) := \sum_{j=1}^pf_j(\mathbf{A}_j\xb)$, with $p \gg n$).
Hence, it is interesting to study proximal quasi-Newton method for solving \eqref{eq:COP}. Our interest in the quasi-Newton methods in this paper is for completeness; we do not provide any algorithmic details or implementations on our quasi-Newton variant.  

To this end, we need a symmetric positive definite matrix $\Hb_k$ that approximates $\nabla^2f(\xb^k)$ at the iteration $k$. As a result, our main assumption here is that matrix $\Hb_{k+1}$ at the next iteration $k+1$ satisfies the \textit{secant equation}:
\begin{equation}\label{eq:scant_eq}
\Hb_{k+1}(\xb^{k+1} - \xb^k) = \nabla{f}(\xb^{k+1}) - \nabla{f}(\xb^k). 
\end{equation}
For instance, it is well-known that the sequence of matrices $\set{\Hb_{k}}_{k\geq 0}$ updated by the following BFGS formula satisfies the secant equation
\eqref{eq:scant_eq}  \citep{Nocedal2006}:
\begin{equation}\label{eq:bfgs_formula}
\Hb_{k+1} := \Hb_k + \frac{1}{(\yb^k)^T\zb^k}\yb^k(\yb^k)^T - \frac{1}{(\zb^k)^T\Hb_k\zb^k}\Hb_k\zb^k(\Hb_k\zb^k)^T,
\end{equation}
where $\zb^k := \xb^{k+1} - \xb^k$ and $\yb^k := \nabla{f}(\xb^{k+1}) - \nabla{f}(\xb^k)$. 
Other methods for updating matrix $\Hb_k$ can be found in \citep{Nocedal2006}, which are not listed here.

In this subsection, we only analyze the full-step proximal quasi-Newton scheme based on the BFGS updates. The global convergence characterization of the BFGS quasi-Newton method can be obtained using our analysis in the next subsection. To this end, we have the following update equation, where the subscript $q$ is used to distinguish the quasi-Newton method: 
\begin{equation}\label{eq:PFQNM}
\xb^{k+1} := \xb^k + \db_q^k.
\end{equation} 
Here we use $\db^k_q$ to stand for the proximal quasi-Newton search direction.

Under certain assumptions, one can prove that the sequence $\set{\xb^k}_{k\geq 0}$ generated by \eqref{eq:PFQNM} converges to $\xb^{*}$ the unique solution of \eqref{eq:COP}.
One of the common assumptions used in quasi-Newton methods is the Dennis-Mor\'{e} condition, see \citep{Dennis1974}. Adopting the Dennis-Mor\'{e} criterion, we impose the following condition in our context:
\begin{equation}\label{eq:Dennis_More_cond}
\lim_{k\to\infty}\frac{\norm{\left[\Hb_k - \nabla^2f(\xb^{*})\right](\xb^{k+1} - \xb^k)}^{*}_{\xb^{*}}}{\norm{\xb^{k+1} - \xb^k}_{\xb^{*}}} = 0.
\end{equation} 
 Now, we establish the superlinear convergence of the sequence $\set{\xb^k}_{k\geq 0}$ generated by \eqref{eq:PFQNM}  as follows: 

\begin{theorem}\label{th:quasi_newton}
Assume that  $\xb^{*}$ is the unique solution of \eqref{eq:COP} and is strongly regular.
Let matrix $\Hb_k$ maintains the secant equation \eqref{eq:scant_eq} and let $\set{\xb^k}_{k\geq 0}$ be a sequence generated by scheme \eqref{eq:PFQNM}.
Then the following statements hold:
\begin{itemize}
\item[(a)]
Suppose, in addition, that the sequence of matrices $\set{\Hb_k}_{k\geq 0}$ satisfies the Dennis-Mor\'{e} condition \eqref{eq:Dennis_More_cond} for sufficiently large $k$. Then the sequence $\set{\xb^k}_{k\geq 0}$ converges to the solution $\xb^{*}$ of \eqref{eq:COP} at a superlinear rate provided that $\norm{\xb^0 - \xb^{*}}_{\xb^{*}} < 1$.

\item[(b)]
Suppose that a matrix  $\Hb_0\succ 0$ is chosen. 
Then $(\yb^k)^T\zb^k > 0$ for all $k \geq 0$ and hence the sequence $\set{\Hb_k}_{k\geq 0}$ generated by \eqref{eq:bfgs_formula} is symmetric positive definite and satisfies the secant equation \eqref{eq:scant_eq}. 
Moreover, if  the sequence $\set{\xb^k}_{k\geq 0}$ generated by \eqref{eq:PFQNM} satisfies $\sum_{k=0}^{\infty}\norm{\xb^k - \xb^{*}}_{\xb^{*}} < +\infty$, then this sequence converges to  $\xb^{*}$ at a superlinear rate.
\end{itemize}
\end{theorem}
The proof of this theorem can be found in the appendix.
We note that if the sequence $\set{\xb^k}_{k\geq 0}$ locally converges to $\xb^{*}$ at a linear rate w.r.t. the local norm at $\xb^{*}$, i.e. $\norm{\xb^{k+1} - \xb^{*}}_{\xb^{*}} \leq \kappa\norm{\xb^{k} - \xb^{*}}_{\xb^{*}}$ for some $\kappa \in (0, 1)$ and $k\geq 0$, then the condition $\sum_{k=0}^{\infty}\norm{\xb^k - \xb^{*}}_{\xb^{*}} < +\infty$ automatically holds.
From \eqref{eq:Dennis_More_cond} we also observe that the matrix $\Hb_k$ is required to well approximate $\nabla^2f(\xb^{*})$ along the direction $\db^k_q$, which is not in the whole space.

\subsection{A proximal-gradient method}\label{subsec:prox_gradient_method}
If we choose matrix $\Hb_k := \mathbf{D}_k$, where $\mathbf{D}_k$ is a positive diagonal matrix, then the iterative scheme \eqref{eq:iterative_scheme} is called the \textit{proximal-gradient} scheme.
In this case, we can write \eqref{eq:iterative_scheme} as
\begin{equation}\label{eq:PGM}
\xb^{k+1} := \xb^k + \alpha_k\db^k_g = (1-\alpha_k)\xb^k + \alpha_k\mathbf{s}_g^k, 
\end{equation}
where  $\alpha_k\in (0,1]$ is an appropriate step size, $\db^k_g$ is the proximal-gradient search direction and $\sb^k_g \equiv \sb^k$ as in \eqref{eq:cvx_subprob}.

The following lemma shows how we can choose the step size $\alpha_k$ corresponding to $\mathbf{D}_k$ such that we obtain a descent direction in the proximal-gradient scheme \eqref{eq:PGM}. The proof of this lemma can be found in the appendix.

\begin{lemma}\label{le:step_size}
Let $\set{\xb^k}_{k\geq 0}$ be a sequence generated by \eqref{eq:PGM}. 
Suppose that the matrix $\mathbf{D}_k \succ 0$ is chosen such that the step size $\alpha_k $ satisfies $\alpha_k := \frac{\beta_k^2}{\lambda_k(\lambda_k + \beta_k^2)}\in (0, 1]$ $($see below$)$, where $\beta_k := \Vert\db^k_g\Vert_{\mathbf{D}_k}$ and $\lambda_k := \Vert\db^k_g\Vert_{\xb^k}$. Then $\set{\xb^k}_{k\geq 0}\subset \dom{F}$ and
\begin{equation}\label{eq:decrease_F}
F(\xb^{k+1}) \leq F(\xb^k) - \omega\left(\frac{\beta_k^2}{\lambda_k}\right).
\end{equation}
Moreover, the step-size $\alpha_k$ as defined above is optimal.
\end{lemma}

From Lemma \ref{le:step_size}, we observe that $\alpha_k \leq 1$ if $\frac{\lambda_k^2}{\beta_k^2} + \lambda_k \geq 1$. 
It is obvious that if $\lambda_k \geq 1$ then the last condition is automatically satisfied. 
We only consider the case $\lambda_k < 1$.
In fact,  since $\lambda_k \geq 0$, we relax actually the condition $\frac{\lambda_k^2}{\beta_k^2} + \lambda_k \geq 1$ to a simpler condition $\lambda_k \geq \beta_k$.  

We now study the case $\mathbf{D}_k := L_k\mathbb{I}$, where $L_k \geq \underline{L} > 0$ is a positive constant and $\mathbb{I}$ is the identity matrix with dimensions apparent from the context.
Hence, $\beta_k^2 = L_k\Vert\db^k_g\Vert_2^2$ and
\begin{align}
\frac{\lambda_k^2}{\beta_{k}^2} = \frac{(\db^k_g)^T\nabla^2f(\xb^k)\db^k_g}{L_k\Vert\db^k_g\Vert_2^2}. \nonumber
\end{align} 
However, since
\begin{align}
\sigma_{\min}(\nabla^2f(\xb^k)) \leq  \sigma^k := \frac{(\db^k_g)^T\nabla^2f(\xb^k)\db^k_g}{\Vert\db^k_g\Vert_2^2} \leq \sigma_{\max}(\nabla^2f(\xb^k)),
\end{align} 
the condition $\lambda_k \geq \beta_k$ is equivalent to
\begin{equation}\label{eq:L_condition}
L_k \leq \sigma_k,
\end{equation} 
where $\sigma_{\min}^k := \sigma_{\min}(\nabla^2f(\xb^k))$ and $\sigma_{\max}^k := \sigma_{\max}(\nabla^2f(\xb^k))$ are the smallest and largest eigenvalue of $\nabla^2f(\xb^k)$, respectively.
Under the assumption that $\dom{f}$ contains no straight-line, then we have the Hessian $\nabla^2f(\xb^k) \succ 0$  by \cite[Theorem 4.1.3]{Nesterov2004}, which implies that $\sigma^k_{\min} > 0$.
Therefore, in the worst-case, we can choose $L_k := \sigma_{\min}^k$. However, this lower bound may be too conservative. In practice, we can apply a \textit{bisection procedure} to meet the condition \eqref{eq:L_condition}. It is not difficult to prove via contradiction that the number of bisection steps is upper bounded by a {constant}.

We note that if $g$ is separable, i.e. $g(\xb) := \sum_{i=1}^ng_i(\xb_i)$ (e.g. $g(\xb) := \rho\norm{\xb}_1$), then we can compute $\mathbf{s}^k_{\mathbf{D}_k}$ in \eqref{eq:cvx_subprob} in a component-wise fashion as:
\begin{equation}\label{eq:cvx_subprob_grad_method1}
(\mathbf{s}^k_{L_k})_i :=  \mathcal{P}^{g_i}_{\tau^k_{i}}\left(\xb^k_i - \tau^k_i(\nabla{f}(\xb^k))_i\right), ~ i = 1, \dots, n,
\end{equation}
where $\tau^k_i := 1/(\mathbf{D}_k)_{ii}$ and $\mathcal{P}_{\tau_i}^{g_i}(\cdot) $ is the proximity operator of $g_i$ function, with parameter $\tau_i$. 
The computation of $\lambda_k$ only requires one matrix-vector multiplication and one vector inner-product; but it can be reduced by exploiting concrete structure of the smooth part $f$.

Based on Lemma \ref{le:step_size}, we describe the proximal-gradient scheme \eqref{eq:PGM} in Algorithm \ref{alg:A1}. The main computation cost of Algorithm \ref{alg:A1} is incurred at Step 2 and in calculating $\lambda_k$. If $g$ is separable, then the computation of Step 2 can be done in a \textit{closed form}.
\begin{algorithm}[ht]\caption{(\textit{Proximal-gradient method})}\label{alg:A1}
\begin{algorithmic}
 \STATE {\bfseries Inputs:} $\xb^0\in\dom{F}$, tolerance $\varepsilon > 0$.
 \STATE \hrulefill
\FOR{$k=0$ {\bfseries to} $k_{\max}$}
\STATE 1. Choose an appropriate $\mathbf{D}_k \succ 0$ based on \eqref{eq:L_condition}.
\STATE 2. Compute $\db^k_g := \mathcal{P}^g_{\mathbf{D}_k}\left(\mathbf{D}_k\xb^k - \nabla{f}(\xb^k)\right) - \xb^k$ due to \eqref{eq:cvx_subprob}.
\STATE 3. Compute $\beta_k := \Vert\db^k_g\Vert_{\mathbf{D}_k}$ and $\lambda_k := \Vert\db^k_g\Vert_{\xb^k}$.
\STATE 4. If $e_k := \Vert\db^k_g\Vert_2 \leq \varepsilon$ then terminate.
\STATE 5. Update $\xb^{k+1} := \xb^k + \alpha_k\db^k_g$, where $\alpha_k := \frac{\beta_k^2}{\lambda_k(\lambda_k + \beta_k^2)} \in (0, 1]$.
\ENDFOR
\end{algorithmic}
\end{algorithm}
One main step of Algorithm \ref{alg:A1} is Step 2, which depends on the cost of prox-operator $\mathcal{P}^g_{\mathbf{D}_k}$.  In practice, $\mathbf{D}_k$ is determined by a bisection procedure whenever $\lambda_k < 1$, which requires additional computational cost. 
If we choose $D_k := L_k\mathbb{I}$, then in order to fulfill \eqref{eq:L_condition}, we can perform a back-tracking linesearch procedure on $L_k$. This linesearch procedure does not require the evaluations of the objective function. We modify Steps 1-3 of Algorithm \ref{alg:A1} as
\begin{itemize}
\item[1.] Initialize $L_k := L_k^0 > 0$, e.g., by a Barzilai-Borwein step.
\item[2.] Compute $\db^k_g := \mathcal{P}^g_{{L_k\mathbf{I}}_k}\left(L_k\xb^k - \nabla{f}(\xb^k)\right) - \xb^k$ due to \eqref{eq:cvx_subprob}.
\item[3a.] Compute $\beta_k := \Vert\db^k_g\Vert_{L_k\mathbf{I}}$ and $\lambda_k := \Vert\db^k_g\Vert_{\xb^k}$.
\item[3b.] If $\lambda_k^2/\beta_k^2 + \lambda_k < 1$, then set $L_k := L_k/2$ and go back to Step 2. 
\end{itemize}
We note that computing $\lambda_k$ at Step 3 does not need to form the full Hessian $\nabla^2f(\xb^k)$, it only requires a directional derivative, which is relatively cheap in applications \cite[Chapter 7]{Nocedal2006}.

\paragraph{Global and local convergence. }
The global and local convergence of Algorithm \ref{alg:A1}  is stated in the following theorems, whose proof can be found  in the appendix.

\begin{theorem}\label{th:global_convergence_of_grad_method}
Assume that there exists $\underline{L} > 0$ such that $\mathbf{D}_k\succeq \underline{L}\mathbb{I}$ for $k\geq 0$, and the solution $\xb^{*}$ of \eqref{eq:COP} is unique.
Let the sublevel set
\begin{equation*}
\mathcal{L}_F(F(\xb^0)) := \set{\xb\in\dom{F} ~|~ F(\xb) \leq F(\xb^0)}
\end{equation*}
be bounded from below. Then, the sequence $\set{\xb^k}_{k\geq 0}$, generated by Algorithm \ref{alg:A1}, converges to the unique solution $\xb^{*}$ of \eqref{eq:COP}. 
\end{theorem}

\begin{theorem}\label{th:convergence_of_grad_method}
Assume that  $\xb^{*}$ is the unique solution of \eqref{eq:COP} and is strongly regular.
Let $\set{\xb^k}_{k\geq 0}$ be the sequence generated by Algorithm \ref{alg:A1}. Then, for $k$ sufficiently large, if 
\begin{equation}\label{eq:linear_convg_cond}
\frac{\norm{[\mathbf{D}_k - \nabla^2f(\xb^{*})]\db^k_g}_{\xb^{*}}^{*}}{\Vert\db^k_g\Vert_{\xb^{*}}} < \frac{1}{2},
\end{equation}
then $\set{\xb^k}_{k\geq 0}$ locally converges to $\xb^{*}$ at a linear rate. 
In particular, if $\mathbf{D}_k := L_k\mathbb{I}$ and $\gamma_{*} := \max\set{\abs{1 - \frac{L_k}{\sigma_{\min}^{*}}}, \abs{1 - \frac{L_k}{\sigma_{\max}^{*}}}} < \frac{1}{2}$ then the condition \eqref{eq:linear_convg_cond} holds.
\end{theorem}

We note that $\xb^{*}$ is {\it unknown}; thus, evaluating $\gamma_{*}$ a priori is infeasible in reality. In implementation, one can choose an appropriate value $L_k\geq \underline{L} > 0$ and then adaptively update  $L_k$ based on the knowledge of the eigenvalues of $\nabla^2f(\xb^k)$ near to the solution $\xb^{*}$. 
The condition \eqref{eq:linear_convg_cond} can be expressed as $(\db^k_g)^T[L_k^2\nabla^2f(\xb^{*})^{-1} + \nabla^2f(\xb^{*}) - 2L_k\mathbb{I}]\db^k_g \leq \frac{1}{4}\Vert\db_g^k\Vert_{\xb^{*}}^2$, which leads to 
\begin{align}\label{eq:linear_converg_cond}
\frac{3}{4}\Vert\db_g^k\Vert_{\xb^{*}}^2 + L^2[\Vert\db_g^k\Vert_{\xb^{*}}^{*}]^2 < 2L_k\Vert\db_g^k\Vert_2^2.
\end{align}
We note that to find $L_k$ such that  \eqref{eq:linear_converg_cond} holds, we require $\Vert\db_g^k\Vert_{\xb^{*}}^{*}\Vert\db_g^k\Vert_{\xb^{*}} < \sqrt{\frac{4}{3}}\Vert\db_g^k\Vert_2^2$.
If the last condition in Theorem \ref{th:convergence_of_grad_method} is satisfied then the condition \eqref{eq:linear_converg_cond} also holds.
While the last condition in  Theorem \ref{th:convergence_of_grad_method} seems too imposing, we claim that, for most $f$ and $g$, we only require \eqref{eq:linear_converg_cond} to be satisfied (see also the empirical evidence in Subsection  \ref{subsec:linear_convergence}). 
The condition \eqref{eq:linear_convg_cond} (or \eqref{eq:linear_converg_cond}) can be referred to as a \emph{\bf restricted} approximation gap between $\mathbf{D}_k$ and the true Hessian $\nabla^2f(\xb^{*})$ along the direction $\db^{k}_g$ for $k$ sufficiently large. 
For instance,  when $g$ is based on the $\ell_1$-norm/the nuclear norm, the search direction $\db^{k}_g$ have at most twice the sparsity/rank of $\xb^{*}$ near the convergence region. Given a subspace generated by all the directions $\db^k_g$, one can prove, via probabilistic assumptions on $f$ that the restricted condition \eqref{eq:linear_convg_cond} is satisfied with a high probability using statistic tools.

\begin{remark}\label{re:sparsity}
From the scheme \eqref{eq:PGM} we observe that the step size $\alpha_k < 1$ may not preserve some of the desiderata on $\xb^{k+1}$ due to the closed form solution of the prox-operator $\mathcal{P}^g_{\mathbf{D}_k}$. For instance, when $g$ is based on the $\ell_1$-norm, $\alpha_k < 1$, might increase the sparsity level of the solution as opposed to monotonically increasing it. However, in practice, the numerical values of $\alpha_k$ are often $1$ near the convergence, which maintain properties, such as sparsity, low-rankedness, etc.
\end{remark}

\paragraph{Global convergence rate:}
In proximal gradient methods, proving global convergence rate guarantees requires a global constant to be known \emph{a priori}---such as the Lipschitz constant. However such an assumption does not apply for the class of just self-concordant functions that we consider in this paper.
We only characterize the following property in an ergodic sense.
Let $\set{\xb^k}_{k\geq 0}$ be the sequence generated by \eqref{alg:A1}. We define
\begin{equation}\label{eq:ergodic_sequence}
\bar{\xb}^k := S_k^{-1}\sum_{j=0}^k\alpha_j\xb^j, ~~\mathrm{where}~~S_k := \sum_{j=0}^k\alpha_j > 0.
\end{equation}
Then we can show that $F(\bar{\xb}^k) - F^{*} \leq \frac{\bar{L}}{2S_k}\norm{\xb^0 - \xb^{*}}_2^2$, where $\bar{L} := \displaystyle\max_{j=0,k}L_k$.
If $\alpha_j \geq \underline{\alpha} > 0$ for $0 \leq j \leq k$, then $S_k \geq \underline{\alpha}(k+1)$, which leads to $F(\bar{\xb}^k) - F^{*} \leq \frac{\bar{L}}{2(k+1)\underline{\alpha}}\norm{\xb^0 - \xb^{*}}_2^2$.
The proof of this statement can be found in \citep{Tran-Dinh2014}, which we omit  here.

\paragraph{A modification of the proximal-gradient method:}
If the point $\sb^k_g$ generated by \eqref{eq:cvx_subprob} belongs to $\dom{F}$, then $F(\sb^k_g) < +\infty$.
Similarly to the definition of $\xb^{k+1}$ in \eqref{eq:PGM}, we can define a new trial point
\begin{equation}\label{eq:trial_point}
\hat{\xb}^k := (1-\alpha_k)\xb^k + \alpha_k\sb^k_g.
\end{equation}
If $F(\sb^k_g) \leq F(\xb^k)$, then, by the convexity of $F$, it is easy to show that
\begin{equation*}
F(\hat{\xb}^{k}) = F\left((1-\alpha_k)\xb^k + \alpha_k\sb^k_g\right) \leq (1-\alpha_k)F(\xb^k) + \alpha_kF(\sb^k_g) \overset{F(\sb^k_g) \leq F(\xb^k)}{\leq} F(\xb^k).
\end{equation*}
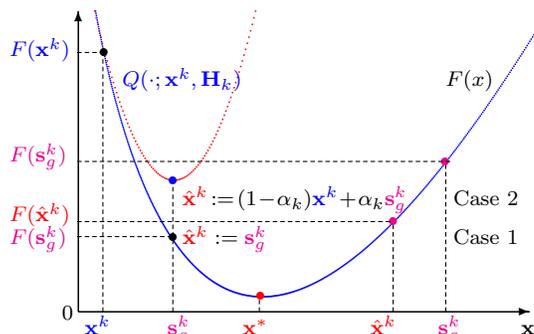
\begin{figure}[ht]
\begin{center}
\setlength{\unitlength}{1mm}
\begin{picture}(0,42)
 {\linethickness{0.2mm}
 \put(-30, 0){\vector(1, 0){60}}
 \put(-30, 0){\vector(0, 1){40}} 
}
{\linethickness{0.2mm}\color{blue} \qbezier(-29, 40)(-15, -35)(30, 38)}
{\linethickness{0.2mm}\color{red} \qbezier[100](-30.4, 40)(-20, -5)(-9, 40)}
\multiput(-29.3,0)(0,1){35}{\line(0,1){0.6}}
{\color{black}\put(-29.2,34.5){\circle*{1.2}}}
\multiput(-20,0)(0,1){18}{\line(0,1){0.6}}
{\color{blue}\put(-20,17.5){\circle*{1.2}}}
\multiput(-8.5,0)(0,1){3}{\line(0,1){0.6}}
{\color{black}\put(-20,10){\circle*{1.2}}}
{\color{red}\put(-9.65,2.2){\circle*{1.2}}}
\multiput(15,0)(0,1){20}{\line(0,1){0.6}}
{\color{magenta}\put(14.9,20){\circle*{1.2}}}
\multiput(8,0)(0,1){12}{\line(0,1){0.6}}
{\color{magenta}\put(8,12){\circle*{1.2}}}
\multiput(-34,10)(1,0){13}{\line(1,0){0.6}}
\multiput(-34,34.5)(1,0){4}{\line(1,0){0.6}}
\multiput(-34,20)(1,0){50}{\line(1,0){0.6}}
\multiput(-34,12)(1,0){42}{\line(1,0){0.6}}
\put(-36,-1){\scriptsize$0$}
\put(-12,-3){\scriptsize\color{red}$\xb^{*}$}
\put(-33,-3){\scriptsize\color{blue}$\xb^k$}
\put(-22,-3){\scriptsize\color{magenta}$\sb^k_g$}
\put(14,-3){\scriptsize\color{magenta}$\sb^k_g$}
\put(25,-3){\scriptsize\color{black}$\xb$}
\put(5,-3){\scriptsize\color{red}$\hat{\xb}^{k}$}
\put(-43,34){\scriptsize\color{blue}$F(\xb^k)$}
\put(-43,20){\scriptsize\color{magenta}$F(\sb^k_g)$}
\put(-43,9){\scriptsize\color{magenta}$F(\sb^k_g)$}
\put(-43,12){\scriptsize\color{red}$F(\hat{\xb}^{k})$}
\put(15,30){\scriptsize\color{black}$F(x)$}
\put(-28,30){\scriptsize\color{blue}$Q(\cdot;\xb^k,\Hb_k)$}
\put(-20,14.2){\scriptsize${\color{red}\hat{\xb}^k} \!:=\! (1 \!-\! \alpha_k){\color{blue}\xb^k} \!+\! \alpha_k{\color{magenta}\sb^k_g}$}
\put(-20,9){\scriptsize${\color{red}\hat{\xb}^k} := {\color{magenta}\sb^k_g}$}
\put(16,9){\scriptsize Case 1}
\put(16,14.2){\scriptsize Case 2}
\end{picture}
\caption{Illustration of the modified proximal-gradient method}\label{fig:enhanced_proxgrad}
\end{center}
\end{figure}  
In this case, based on the function values $F(\sb^k_g)$, $F(\hat{\xb}^k)$ and $F(\xb^{k})$ we can eventually choose the next iteration $\xb^{k+1}$ as follows:
\begin{equation}\label{eq:enhanced_prox_grad}
\xb^{k+1} := \left\{\begin{array}{lll}
\sb^k_g &\textrm{if} ~\sb^k\in\dom{F} ~\textrm{and}~ F(\sb^k_g) < F(\hat{\xb}^k) &\textrm{(Case 1)},\\
\hat{\xb}^k &\textrm{otherwise} &\textrm{(Case 2)}. 
\end{array}\right.
\end{equation}
The idea of this \emph{greedy} modification is illustrated in Figure \ref{fig:enhanced_proxgrad}. 
We note that here we need to check $\sb^{k}_g \in \dom{F}$ such that $F(\sb^k_g) < F(\xb^k)$ and additional function evaluations $F(\sb^k_g)$ and $F(\hat{\xb}^{k})$. However, careful implementations can recycle quantities that enable us to evaluate the objective at $\sb^k_g$ and at $\xb^{k+1}$ with very little overhead over the calculation of $\alpha_k$ (see Section \ref{sec:applications}).  By using \eqref{eq:enhanced_prox_grad}, we can specify a modified proximal gradient algorithm for solving \eqref{eq:COP}, whose details we omit here since it is quite similar to Algorithm \ref{alg:A1}.

\section{Concrete instances of our optimization framework}\label{sec:applications}
We illustrate three instances of our framework for some of the applications described in Section \ref{sec:intro}.
For concreteness, we describe only the first and second order methods. Quasi-Newton methods based on (L-)BFGS updates or other adaptive variable metrics can be similarly derived in a straightforward fashion. 

\subsection{Graphical model selection}\label{sec:app_graph_select}
We customize our optimization framework to solve the graph selection problem \eqref{eq:glearn_prob}. For notational convenience, we maintain a matrix variable $\mathbf{\Theta}$ instead of vectorizing it. We observe that $f(\mathbf{\mathbf{\Theta}}) := -\log(\det(\mathbf{\Theta})) + \trace({\hat{\mathbf{\Sigma}}\mathbf{\Theta}})$ is a standard self-concordant function, while $g(\mathbf{\Theta}) := \rho \norm{\vec({\mathbf{\Theta}})}_1$ is convex and nonsmooth. The gradient and the Hessian of $f$ can be computed explicitly as $\nabla{f}(\mathbf{\Theta}) := \hat{\mathbf{\Sigma}} - \mathbf{\Theta}^{-1}$ and $\nabla^2f(\mathbf{\Theta}) := \mathbf{\Theta}^{-1}\otimes \mathbf{\Theta}^{-1}$, respectively.
Next, we formulate our proposed framework to construct two algorithmic variants for \eqref{eq:glearn_prob}.

\subsubsection{Dual proximal-Newton algorithm}
We  consider a second order algorithm via a dual solution approach for \eqref{eq:cvx_subprob}. This approach is first introduced in our earlier work \citep{Tran-Dinh2013b}, which did not consider the new modifications we propose in Section \ref{subsec:prox_newton_method}. 

We begin by deriving the following dual formulation of the convex subproblem \eqref{eq:cvx_subprob}. 
Let $\pb_k := \nabla{f}(\xb^k)$, the
convex subproblem \eqref{eq:cvx_subprob} can then be written equivalently as
\begin{equation}\label{eq:cvx_subprob2}
\min_{\xb\in\mathbb{R}^n}\set{\frac{1}{2}\xb^T\Hb_k\xb + (\pb_k - \Hb_k\xb^k)^T\xb + g(\xb)}.
\end{equation}
By using the min-max principle, we can write \eqref{eq:cvx_subprob2} as
\begin{equation}\label{eq:min_max_subprob}
\max_{\ub\in\mathbb{R}^n}\min_{\xb\in\mathbb{R}^n}\set{\frac{1}{2}\xb^T\Hb_k\xb + (\pb_k - \Hb_k\xb^k)^T\xb + \ub^T\xb - g^{*}(\ub)}, 
\end{equation}
where $g^{*}$ is the Fenchel conjugate function of $g$, i.e. $g^{*}(\ub) := \displaystyle\sup_{\xb}\set{ \ub^T\xb - g(\xb)}$.
Solving the inner minimization in \eqref{eq:min_max_subprob} we obtain
\begin{equation}\label{eq:cvx_dual_subprob2}
\min_{\ub\in\mathbb{R}^n}\set{\frac{1}{2}\ub^T\Hb_k^{-1}\ub + \tilde{\pb}_k^T\ub + g^{*}(\ub)},
\end{equation}
where $\tilde{\pb}_k :=  \Hb_k^{-1}\pb_k - \xb^k$.
Note that the objective function $\varphi(\ub) := g^{*}(\ub) + \frac{1}{2}\ub^T\Hb_k^{-1}\ub + \tilde{\pb}_k^T\ub$ of \eqref{eq:cvx_dual_subprob2}
is strongly convex, one can apply the fast projected gradient methods with a linear convergence rate for solving this problem, see \citep{Nesterov2007,Beck2009}.

In order to recover the solution of the primal subproblem \eqref{eq:cvx_subprob}, we note that the solution of the parametric minimization problem in
\eqref{eq:min_max_subprob} is given by $\xb^{*}(\ub) := \xb^k - \Hb_k^{-1}(\pb_k + \ub)$.
Let $\ub^{*}_{\xb^k}$ be the optimal solution of \eqref{eq:cvx_dual_subprob2}. 
We can recover the primal proximal-Newton search direction $\db^k$ defined in \eqref{eq:search_dir_dk} as
\begin{equation}\label{eq:primal_sol_of_cvx_subprob2}
\db^k_n = - \nabla^{2}f(\xb^k)^{-1}\left(\nabla{f}(\xb^k) + \ub^{*}_{\xb^k}\right).
\end{equation}
To compute the quantity $\lambda_k$ defined by \eqref{eq:prox_Newton_decrement} in Algorithm \ref{alg:A2}, we use \eqref{eq:primal_sol_of_cvx_subprob2} such
that
\begin{equation}\label{eq:r_norm}
\lambda_k = \Vert\db^k_n\Vert_{\xb^k} = \norm{\nabla{f}(\xb^k) + \ub^{*}_{\xb^k}}^{*}_{\xb^k}. 
\end{equation}
Note that computing $\lambda_k$ by \eqref{eq:r_norm} requires the inverse of the Hessian matrix $\nabla^2{f}(\xb^k)$.

Surprisingly, this dual approach allows us to avoid matrix inversion as well as Cholesky decomposition in computing the gradient $\nabla{f}(\mathbf{\Theta}_i)$ and the Hessian $\nabla^2f(\mathbf{\Theta}_i)$ of $f$ in graph selection. An alternative is of course to solve \eqref{eq:cvx_subprob} in its primal form. Though, in such case, we need to compute $\mathbf{\Theta}_i^{-1}$ at each iteration $i$ (say, via Cholesky decompositions). 

The dual subproblem \eqref{eq:cvx_dual_subprob2} becomes as:
\begin{equation}\label{eq:cvx_dual_subprob2_app}
{\bf U}^{\ast} = \argmin_{\norm{\vec({\bf U})}_{\infty}\leq 1}\set{\frac{1}{2}\trace((\T_i {\bf U})^2) + \trace(\widetilde{\bf Q}{\bf U})},
\end{equation}
for the graph selection, where $\widetilde{\bf Q} := \rho^{-1}[\T_i\widehat{\cov}\T_i - 2\T_i]$. 
Given the dual solution ${\bf U}^{\ast}$ of \eqref{eq:cvx_dual_subprob2_app}, the primal proximal-Newton search direction (i.e. the solution of
\eqref{eq:cvx_subprob}) is computed as
\begin{equation}\label{eq:primal_sol_of_cvx_subprob2_app}
\boldsymbol{\Delta}_i := -\left((\T_i\widehat{\cov} - \mathbb{I})\T_i + \rho\T_i {\bf U}^{\ast}\T_i\right).  
\end{equation} 
The quantity $\lambda_i$ defined in \eqref{eq:r_norm} can be computed as follows, where ${\bf W}_i := \T_i(\widehat{\cov} + \rho {\bf U}^{\ast})$:
\begin{equation}\label{eq:r_norm_app}
\lambda_i  := \left(p - 2\cdot \trace\left({\bf W}_i\right) + \trace\left( {\bf W}_i^2\right)\right)^{1/2}.
\end{equation} 
Algorithm \ref{alg:A2_glearn} summarizes the description above.
\begin{algorithm}[!ht]\caption{(\textit{Dual PN for graph selection} (\texttt{DPNGS}))}\label{alg:A2_glearn}
\begin{algorithmic}   
   \STATE\textbf{Input: } Matrix $\widehat{\cov}\succ 0$ and a given tolerance $\varepsilon > 0$. Set $\sigma := 0.25(5-\sqrt{17})$.
   \STATE {\bfseries Initialization:} Find a starting point $\T_0\succ 0$.
   \FOR{$i=0$ {\bfseries to} $i_{\max}$}
      \STATE 1. Set $\widetilde{\bf Q} := \rho^{-1}\left(\T_i\widehat{\cov}\T_i - 2\T_i\right)$.
      \STATE 2. Compute ${\bf U}^{\ast}$ in \eqref{eq:cvx_dual_subprob2_app}.
      \STATE 3. Compute $\lambda_i$ by \eqref{eq:r_norm_app}, where $\textbf{W}_i \!:=\! \T_i(\widehat{\cov} \!+\! \rho {\bf U}^{\ast})$. 
      \STATE 4. If $\lambda_i \leq \varepsilon$ terminate.
      \STATE 5. Compute $\small{\boldsymbol{\Delta}_i := -\left((\T_i\widehat{\cov} - \mathbb{I})\T_i + \rho\T_i{\bf U}^{\ast}\T_i\right)}$.
      \STATE 6. If $\lambda_i > \sigma$, then set $\alpha_i := (1 + \lambda_i)^{-1}$. Otherwise, set $\alpha_i = 1$.
      \STATE 7. Update $\T_{i+1} := \T_i + \alpha_i\boldsymbol{\Delta}_i$.
   \ENDFOR
\end{algorithmic}
\end{algorithm}
Overall, this proximal-Newton (PN) algorithm {\it does not require any matrix inversions or Cholesky decompositions}. It only needs matrix-vector and matrix-matrix calculations, which might be attractive for different computational platforms (such as GPUs or simple parallel implementations). Note however that as we work through the dual problem, the primal solution can be dense even if majority of the entries are rather small (e.g.,
smaller than $10^{-6}$).\footnote{In our MATLAB implementation below, we have not exploited the fact that the primal solutions are sparse. 
The overall efficiency can be improved via thresholding tricks, both in terms of time-complexity (e.g., less number of iterations) and matrix estimation quality.}

We now explain the underlying costs of each step in Algorithm \ref{alg:A2_glearn}, which is useful when we consider different strategies for the selection of the step size $\alpha_k$. The computation of $\widetilde{\bf Q}$ and $\boldsymbol{\Delta}_i$ require basic matrix multiplications. 
For the computation of $\lambda_i$, we require two trace operations: $\tr({\bf W}_i)$ in $\mathcal{O}(p)$ time-complexity and $\tr({\bf W}_i^2)$ in
$\mathcal{O}(p^2)$ complexity. We note here that, while ${\bf W}_i$ is a {\it dense} matrix, the trace operation in the latter case requires only the computation of the
diagonal elements of ${\bf W}_i^2$. 
Given $\T_i$, $\alpha_i$ and $\boldsymbol{\Delta}_i$, the calculation of $\T_{i+1}$ has $\mathcal{O}(p^2)$ complexity. In contrast, evaluation of the objective can be achieved through Cholesky decompositions, which has $\mathcal{O}(p^3)$ time complexity. 

To compute \eqref{eq:cvx_dual_subprob2_app}, we can use the fast proximal-gradient method (FPGM) \citep{Nesterov2007,Beck2009} with step size $1/L$   where $L$
is the Lipschitz constant of the gradient of the objective function in \eqref{eq:cvx_dual_subprob2_app}. 
It is easy to observe that $L := \gamma_{\max}^2(\T_i)$ where $\gamma_{\max}(\T_i)$ is the largest eigenvalue of $\T_i$. For sparse $\T_i$, we can
approximately compute $\gamma_{\max}(\T_i)$ is $O(p^2)$ by using \textit{iterative power methods} (typically, 10 iterations suffice). 
The projection onto $\norm{\text{\textbf{vec}}({\bf U})}_\infty \leq 1$ clips the elements by unity in $O(p^2)$ time. 
Since FPGM requires a constant number of iterations $k_{\max}$ (independent of $p$) to achieve an $\varepsilon_{\mathrm{in}}$
solution accuracy, the time-complexity for the solution in  \eqref{eq:cvx_dual_subprob2_app} is $O(k_{\max}M)$,  where $M$ is the cost of matrix multiplication. We have also implemented block coordinate descent and active set methods which scale $O(p^2)$ in practice when the solution is quite sparse. 

Overall, the major operation with general proximal maps  in the
 algorithm is typically the matrix-matrix multiplications of the form $\T_i\bf{U}\T_i$, where $\T_i$ and $\bf{U}$ are symmetric positive definite.  This operation can naturally be computed (e.g., in a GPU) in a parallel or distributed manner. For more details of such computations we refer the reader to \citep{Bertsekas1989}. It is important to note that without Cholesky decompositions used in objective evaluations, the basic DPNGS approach theoretically scales with the cost of matrix-matrix multiplications.

\subsubsection{Proximal-gradient algorithm}
Since $g(\T) := \rho\norm{\vec(\T)}_1$ and $\nabla{f}(\T_i) = \vec(\widehat{\cov}-\T_i^{-1})$, the subproblem \eqref{eq:cvx_subprob} becomes
\begin{equation}\label{eq:cvx_subprob_glearn}
\DT_{i+1} := \mathcal{T}_{\tau_i\rho}\left(\T_i - \tau_i(\widehat{\cov} -\T^{-1}_i)\right) - \T_i,
\end{equation}
where $\mathcal{T}_{\tau} : \mathbb{R}^{p\times p}\to\mathbb{R}^{p\times p}$ is the component-wise matrix thresholding operator which is defined as $\mathcal{T}_{\tau}(\T) :=
\max\set{0, \abs{\T}-\tau}$. We also note that the computation of $\DT_{i+1}$ requires a matrix inversion $\T_i^{-1}$. Since $\T_i$ is positive
definite, one can apply Cholesky decompositions to compute $\T_i^{-1}$ in $O(p^3)$ operations. 
To compute the quantity $\lambda_i$, we have $\lambda_i := \norm{\DT_i}_{\T_i} = \norm{\T_i^{-1}\DT_i}_2$. We also choose $L_i := 0.5\norm{\nabla^2f(\T_i)}_2 = 0.5\norm{\T_i^{-1}}_2^2$.
The above are summarized in Algorithm \ref{alg:A1_glearn}.

\begin{algorithm}[!ht]\caption{(\textit{Proximal-gradient method for graph selection} $(\texttt{ProxGrad1})$)}\label{alg:A1_glearn}
\begin{algorithmic}
\STATE\textbf{Initialization:} Choose a starting point $\T_0\succ 0$ .
\FOR{$i=0$ {\bfseries to} $i_{\max}$}
\STATE 1. Compute $\T_i^{-1}$ via Cholesky decomposition. 
\STATE 2. Choose $L_i$ satisfying \eqref{eq:L_condition} and set $\tau_i := L_i^{-1}$. 
\STATE 3. Compute the search direction $\DT_i$ as \eqref{eq:cvx_subprob_glearn}. 
\STATE 4. Compute $\beta_i := L_i\norm{\vec(\DT_i)}_2$ and $\lambda_i := \norm{\T_i^{-1}\DT_i}_2$. 
\STATE 5. Determine the step size $\alpha_i := \frac{\beta_i}{\lambda_i(\lambda_i + \beta_i)}$. 
\STATE 6. Update $\T_{i+1} := \T_i + \alpha_i\DT_i$. 
\ENDFOR
\end{algorithmic}
\end{algorithm}

The per iteration complexity is dominated by matrix-matrix multiplications and Cholesky decompositions for matrix inversion calculations. In particular, Step 1 requires a Cholesky decomposition with $O(p^3)$ time-complexity. Step 2 requires to compute $\ell_2$-norm of a symmetric positive matrix, which can be done by a power-method in $O(p^2)$ time-complexity. The complexity of Steps 3, 4 and 6 requires $O(p^2)$ operations. 
Step 2 may require additional bisection steps as mentioned in Algorithm \ref{alg:A1} whenever $\lambda_k < 1$.

\subsection{Poisson intensity reconstruction}
We now describe a variant of Algorithm \ref{alg:A1}; a similar instance based on Algorithm \ref{alg:A2} can be easily devised and we omit the details here.
First, we can easily check that the function $\tilde{f}(\xb) := \sum_{i=1}^m\left(\mathbf{a}_i^T\xb -  y_i\log(\mathbf{a}_i^T\xb)\right)$ in \eqref{eq:Poisson_prob} is convex and self-concordant with parameter $M_{\tilde{f}} := 2\cdot\max\left\{\frac{1}{\sqrt{y_i}} ~|~  y_i > 0 , i=1,\dots, m\right\}$, see \cite[Theorem 4.1.1]{Nesterov2004}. 
We define the functions $f$ and $g$ as:
\begin{equation}\label{eq:f_and_g}
f(\xb) := \frac{M_{\tilde{f}}^2}{4}\tilde{f}(\xb), ~~ g(\xb) := \frac{M_{\tilde{f}}^2}{4}\left(\rho\phi(\xb) + \delta_{\set{\ub~|~\ub\geq 0}}(\xb)\right),
\end{equation} 
where $f$ and $g$ satisfy Assumption \ref{as:A1} and $\delta_{\mathcal{C}}$ is the indicator function of $\mathcal{C}$. Thus, the problem in \eqref{eq:Poisson_prob} can be equivalently transformed into \eqref{eq:COP}. Here, the gradient and the Hessian of $f$ satisfy:
\begin{equation}\label{eq:grad_hessian_of_f}
\nabla{f}(\xb) = \frac{M_{\tilde{f}}^2}{4}\sum_{i=1}^m\left(1 - \frac{y_i}{\mathbf{a}_i^T\xb}\right)\mathbf{a}_i ~~\text{and}~~\nabla^2f(\xb) = \frac{M_{\tilde{f}}^2}{4}\sum_{i=1}^m\frac{y_i}{(\mathbf{a}_i^T\xb)^2}\mathbf{a}_i\mathbf{a}_i^T,
\end{equation} respectively. For a given vector $\db\in\mathbb{R}^n$, the local norm $\norm{\db}_{\xb}$ can then be written as:
\begin{equation}\label{eq:d_norm}
\norm{\db}_{\xb} := \left(\db^T\nabla^2f(\xb)\db\right)^{1/2} = \frac{M_{\tilde{f}}}{2}\left(\sum_{i=1}^m\frac{y_i(\mathbf{a}_i^T\db)^2}{(\mathbf{a}_i^T\xb)^2}\right)^{1/2}.
\end{equation}
Computing this quantity requires one matrix-vector multiplication and $\mathcal{O}(m)$ operations.

For the Poisson model, the subproblem \eqref{eq:cvx_subprob} is expressed as follows:
\begin{equation}\label{eq:cvx_subprob_poisson}
\min_{\xb \geq 0} \set{ \frac{1}{2}\Vert\xb - \mathbf{w}^k\Vert_2^2 + \rho_k\phi(\xb) },
\end{equation}
where $\mathbf{w}^k :=  \xb^k - L_k^{-1}\nabla{f}(\xb^k)$ and $\rho_k := \frac{\rho M_{\tilde{f}}^2}{4L_k}$. 
As a penalty function $\phi$ in the Poisson intensity reconstruction, we use the Total Variation norm (TV-norm), defined as $\phi(\xb) := \norm{\mathbf{D}\xb}_1$ (isotropic) or $\phi(\xb) := \norm{\mathbf{D}\xb}_{1,2}$ (anti-isotropic), where $\mathbf{D}$ is a forward linear operator \citep{Chambolle2011,Beck2009a}. For both TV-norm regularizers, the method proposed in \citep{Beck2009a} can solve \eqref{eq:cvx_subprob_poisson} efficiently. 

The above discussion leads to Algorithm \ref{alg:A1_Poisson}. We note that the constant $L_k$ at Step 2 of this algorithm can be estimated based on different rules. In our implementation below, we initialize $L_k$ at a Barzilai-Borwein step size, i.e., $L_k := \frac{(\nabla{f}(\xb^k) - \nabla{f}(\xb^{k-1}))^T(\xb^k - \xb^{k-1})}{\Vert\xb^k - \xb^{k-1}\Vert^2_2}$ and may perform a few backtracking iterations on $L_k$ to ensure the condition \eqref{eq:L_condition} whenever $\lambda_k < 1$. 

\begin{algorithm}[!ht]\caption{(\texttt{ProxGrad} \textit{for Poisson intensity reconstruction} (\texttt{ProxGrad2}))}\label{alg:A1_Poisson}
\begin{algorithmic}
 \STATE {\bfseries Inputs:} $\xb^0 \geq 0$, $\varepsilon > 0$ and $\rho > 0$.
\STATE Compute $M_{\tilde{f}} := 2\max\left\{\frac{1}{\sqrt{\yb_i}} ~|~ \yb_i > 0 , i=1,\dots, m\right\}$.
\FOR{$k=0$ {\bfseries to} $k_{\max}$}
\STATE 1. Evaluate the gradient of $f$ as \eqref{eq:grad_hessian_of_f}. 
\STATE 2. Compute an appropriate value $L_k > 0$ that satisfies \eqref{eq:L_condition}. 
\STATE 3. Compute $\rho_k := 0.25\rho M_{\tilde{f}}^2L_k^{-1}$ and $\mathbf{w}^k :=  \xb^k - L_k^{-1}\nabla{f}(\xb^k)$.
\STATE 4. Compute $\mathbf{s}^k_g$ by solving  \eqref{eq:cvx_subprob_poisson} and then compute $\db^k_g := \sb^k_g - \xb^k$.
\STATE 5. Compute $\beta_k := L_k\Vert\db^k_g\Vert_2^2$ and $\lambda_k := \Vert\db^k_g\Vert_{\xb^k}$ as \eqref{eq:d_norm}.
\STATE 6. If $e_k := L_k^{-1}\sqrt{\beta_k} \leq \varepsilon$ then terminate.
\STATE 7. Determine the step size $\alpha_k := \frac{\beta_k}{\lambda_k(\lambda_k + \beta_k)}$.
\STATE 8. Update $\xb^{k+1} := \xb^k + \alpha_k\db^k_g$.
\ENDFOR
\end{algorithmic}
\end{algorithm}
Note that we can modify Step 8 in Algorithm \ref{alg:A1_Poisson} by using the update scheme \eqref{eq:enhanced_prox_grad} to obtain a new variant of this algorithm. We omit the details here.

\subsection{Heteroscedastic LASSO}\label{subsec:app_unknown_var_LASSO}
We focus on a convex formulation of the unconstrained LASSO problem with unknown variance studied in \citep{Stadler2012} as:
\begin{equation}\label{eq:unvarLASSO}
(\boldsymbol{\beta}^{*}, \sigma^{*}) := \mathrm{arg}\!\!\!\!\!\!\min_{\boldsymbol{\beta} \in \mathbb{R}^p, \sigma\in\mathbb{R}_{++}}\left\{ -\log(\sigma) + \frac{1}{2n}\norm{\mathbf{X}\boldsymbol{\beta} - \sigma \yb}_2^2 + \rho\norm{\boldsymbol{\beta}}_1\right\}.
\end{equation} 
However, our algorithm  can be applied to solve the multiple unknown variance case considered in \citep{Dalalyan2013}.

By letting $\xb := (\boldsymbol{\beta}^T, \sigma)^T\in\mathbb{R}^{p+1}$, $f(\xb) :=  -\log(\sigma) + \frac{1}{2n}\norm{\mathbf{X}\boldsymbol{\beta} - \sigma \yb}_2^2$. Then, it is easy to see that the function $f$ is standard self-concordant. Hence, we can apply Algorithm \ref{alg:A1} to solve this problem. To highlight the salient differences in the code, we note the following:
\begin{itemize}
\item Define $\mathbf{z} := \mathbf{X}\boldsymbol{\beta} - \sigma \yb$, then the gradient vector of function $f$ can be computed  as  
\begin{equation*}
\nabla{f}(\xb) := \left(n^{-1}\mathbf{z}^T\mathbf{X}, -\sigma^{-1} - n^{-1}\yb^T\mathbf{z} \right)^T.
\end{equation*}
This computation requires two matrix-vector multiplications and one inner product.
\item The quantity $\lambda_k$ can be explicitly computed as 
\begin{equation*}
\lambda_k := \left(\left(\sigma_k^{-2} + n^{-1}\yb^T\yb\right)(\db^{k}_{\sigma})^2  + n^{-1}\mathbf{z}_k^T\mathbf{z}_k - 2n^{-1}\db^k_{\sigma}\yb^T\mathbf{z}_k\right)^{1/2},
\end{equation*}
where $\mathbf{z}_k := \mathbf{X}\db^k_{\boldsymbol{\beta}}$ and $\db^k_g := ((\db^k_{\boldsymbol{\beta}})^T, \db_{\sigma}^k)^T$ is the search direction. 
This quantity requires one matrix-vector multiplication and two inner products. Moreover, this matrix-vector product can be reused to compute the gradient for the next iteration.
\end{itemize}
The final algorithm is very similar to Algorithm \ref{alg:A1_Poisson} and hence we omit the details.

\section{Numerical experiments}\label{sec:num_experiment}
In this section, we illustrate our optimization framework via numerical experiments on the variants discussed in Section \ref{sec:applications}.  We only focus on proximal gradient and Newton variants and encourage the interested reader to try out the quasi-Newton variants for their own applications. All the tests are performed in MATLAB 2011b running on a PC Intel Xeon X5690 at 3.47GHz per core with 94Gb RAM.\footnote{We also provide MATLAB implementations of the examples in this section as a software package (SCOPT) at \url{http://lions.epfl.ch/software}.}

\subsection{Proximal-Newton method in action}
By using the graph selection problem, we first show that the modifications on the proximal-Newton method provides advantages in practical convergence as compared to state-of-the-art strategies and provides a safeguard for line-search procedures in optimization routines. We then highlight the impact of different subsolvers for \eqref{eq:cvx_subprob2} in the practical convergence of the algorithms.

\subsubsection{Comparison of different step-size selection procedures}
We apply four different step-size selection procedures in our proximal-Newton framework to solve problem \eqref{eq:glearn_prob}. Specifically, we test the algorithm based on the following configuration:
\begin{itemize}
\item [$(i)$] We implement Algorithm \ref{alg:A2_glearn} in MATLAB using FISTA \citep{Beck2009} to solve the dual subproblem with the following stopping criterion: $\norm{\boldsymbol{\Theta}_{i+1} - \boldsymbol{\Theta}_{i}}_F \leq 10^{-8} \times \max\set{\norm{\boldsymbol{\Theta}_{i+1}}_F, 1}$.

\item[$(ii)$] We consider four different globalization procedures, whose details can be found in Section \ref{subsec:prox_newton_method}: $a)$ \texttt{NoLS} which uses the analytic step size $\alpha_k^{*} = (1+\lambda_k)^{-1}$, $b)$ \texttt{BtkLS} which is an instance of the proximal-Newton framework of \citep{Lee2012} and uses the standard backtracking line-search based on Amirjo's rule, $c)$ \texttt{E-BtkLS} which is based on the standard backtracking line-search enhanced by the lower bound $\alpha_k^{*}$ and, $d)$ \texttt{FwLS} as the  forward line-search by starting from $\alpha_k^{*}$ and increasing the step size until either $\alpha_k = 1$, infeasibility or  the objective value does not improve.

\item [$(iii)$] We test our implementation on four problem cases: The first problem is a synthetic examples of  size $p = 10$, where the data is generated as in \citep{Kyrillidis2013}. 
We run this test for $10$ times and report computational primitives in average.
Three remaining problems are based on real data from \url{http://ima.umn.edu/~maxxa007/send_SICS/}, where the regularization parameters are chosen as the standard values (cf., \cite{Tran-Dinh2013b,Lee2012,Hsieh2011}). We terminate the proximal-Newton scheme if $\lambda_k \leq 10^{-6}$.
\end{itemize} 

The numerical results are summarized in Table \ref{tb:BT_vs_FT_table}. 
Here, $\#\mathrm{iter}$ denotes the (average) number of iterations, $\#\mathrm{chol}$ represents the (average) number of Cholesky decompositions and $\#\mathrm{Mm}$ is the (average) number of matrix-matrix multiplications.
\begin{table*}[!h]
\begin{footnotesize}
\caption{\textsc{Metadata for the line search strategy comparison}} \label{tb:BT_vs_FT_table}
\vskip0.15cm
\newcommand{\cell}[1]{{\!\!}#1{\!}}
\newcommand{\cellbf}[1]{{\!\!}{\color{blue}#1}{\!}}
\newcommand{\cellbff}[1]{{\!\!}{\color{red}#1}{\!}}
\begin{center}
\scriptsize{\begin{tabular}{l|ccc|ccc|ccc|ccc} \toprule
& \multicolumn{3}{|c|}{\cell{Synthetic ($\rho=0.01$)}} & \multicolumn{3}{|c}{\cell{Arabidopsis ($\rho=0.5$)}}  & \multicolumn{3}{|c}{\cell{Leukemia ($\rho=0.1$)}} & \multicolumn{3}{|c}{\cell{Hereditary ($\rho=0.1$)}}\\ \cmidrule{2-13}
\multicolumn{1}{c|}{\textsc{LS Scheme}} & \multicolumn{1}{|c}{\cell{\#iter}} & \cell{\#chol} & \cell{\#Mm} &  \multicolumn{1}{|c}{\cell{\#iter}} & \cell{\#chol} & \cell{\#Mm} &  \multicolumn{1}{|c}{\cell{\#iter}} & \cell{\#chol} & \cell{\#Mm} &  \multicolumn{1}{|c}{\cell{\#iter}} & \cell{\#chol} & \cell{\#Mm} \\ \cmidrule{2-13}
\texttt{NoLS}       &  \cell{25.4} & \cell{-} & \cell{3400} & \cell{18} & \cell{-} & \cell{1810} &  \cell{44} & \cell{-} & \cell{9842} & \cell{72} & \cell{-} & \cell{20960} \\
\texttt{BtkLS}      & \cell{25.5}  & \cell{37.0} & \cell{2436} & \cell{11} & \cell{25} & \cell{718} & \cell{15}  & \cell{50} & \cell{1282}   & \cell{19} & \cell{63} & \cell{2006} \\
\texttt{E-BtkLS}  & \cell{25.5}  & \cell{36.2} & \cell{2436} & \cell{11} & \cell{24} & \cell{718} & \cell{15}  & \cell{49} & \cell{1282}   & \cell{15} & \cell{51} & \cell{1282} \\
\texttt{FwLS}      & \cell{18.1}  & \cellbf{26.2} & \cellbf{1632} & \cell{10} & \cellbf{17} & \cellbf{612}  & \cell{12}  & \cellbf{34} & \cellbf{844} & \cell{14} & \cellbf{44} & \cellbf{1126}  
\\ \bottomrule
\end{tabular}}
\end{center}
\end{footnotesize}
\vskip -0.15in
\end{table*}

We can see that our new step-size selection procedure \texttt{FwLS} shows superior empirical performance as compared to the rest: The standard approach \texttt{NoLS} usually starts with pessimistic step-sizes which are designed for worst-case problem structures. Therefore, we find it advantageous to continue with a forward line-search procedure. Whenever it reaches the quadratic convergence, no Cholesky decompositions are required. This makes a difference, compared to standard backtracking line-search \texttt{BtkLS} where we need to evaluate the objective value at every iteration. While there is no free lunch, the cost of computing  $\lambda_k$  is $\mathcal{O}(p^2)$ in \texttt{FwLS}, which turns out to be quite cheap in this application. The \texttt{E-BtkLS} combines both backtrack line-search and our analytic step-size $\alpha^{*}_k := (1 + \lambda_k)^{-1}$, which outperforms \texttt{BtkLS} as the regularization parameter becomes smaller. Finally, we note that the \texttt{NoLS} variant needs more iterations but it does not require any Cholesky decompositions, which might be advantageous in homogeneous computational platforms.

\subsubsection{Impact of different solvers for the subproblems}
As mentioned in the introduction, an important step in our second order algorithmic framework is the solution of the subproblem \eqref{eq:cvx_subprob}. If the variable matrix $\Hb_k$ is not diagonal, computing $\sb_{\Hb_k}^k$ corresponds to solving a convex subproblem. For a given regularization term $g$, we can exploit different existing approaches to tackle this problem. We illustrate that the overall framework to be quite robust against the solution accuracy of the individual subsolver. 

In this test, we consider the broad used $\ell_1$-norm function as the regularizer. Hence, \eqref{eq:cvx_subprob} collapses to an unconstrained LASSO problem; cf. \citep{Wright2009}. 
To this end, we implement the proximal-Newton algorithm to solve the graph learning problem \eqref{eq:glearn_prob} where $g(\xb) := \rho\norm{\xb}_1$. To show the impact of the subsolver in \eqref{eq:glearn_prob}, we implement the following methods, which are all available in our software package SCOPT:
\begin{itemize}
\item[$(i)$] \texttt{pFISTA} and \texttt{dFISTA}: in these cases, we use the FISTA algorithm \citep{Beck2009} for solving the primal \eqref{eq:cvx_subprob2} and the dual subproblem \eqref{eq:cvx_dual_subprob2}. Morever, to speedup the computations, we further run these methods on the GPU [NVIDIA Quadro 4000].
\item[$(ii)$] \texttt{FastAS}: this method corresponds to the exact implementation of the fast active-set method proposed in \citep{Kim2010a} for solving the primal-dual \eqref{eq:cvx_subprob2}.
\item[$(iii)$] \texttt{BCDC}: here, we consider the block-coordinate descent method implemented in  \citep{Hsieh2011} for solving the primal subproblem \eqref{eq:cvx_subprob2}.
\end{itemize}
We also compare the above variants of the {\it proximal-Newton approach} with $(i)$ the proximal-gradient method (Algorithm \ref{alg:A1_glearn}) denoted by \texttt{ProxGrad1} and $(ii)$ a precise MATLAB implementation of QUIC (\texttt{MatQUIC}), as described in \citep{Hsieh2011}. For the proximal-Newton and \texttt{MatQUIC} approaches, we terminate the execution if the maximum number of iterations exceeds $200$ or the total execution time exceeds the $5$ hours. The maximum number of iterations in \texttt{ProxGrad1} is set to $10^4$.

The results are reported in Table \ref{tb:diff_solver_table}. Overall, we observe that \texttt{dFISTA} shows superior performance across the board  in terms of computational time and the total number of Cholesky decompositions required. Here, $\#\mathrm{nnz}$ represents the number of nonzero entries in the final solution. 
The notation ``$-$'' indicates that the algorithms exceed either the maximum number of iterations or the time limit ($5$ hours).
\begin{table*}[!tp]
\begin{footnotesize}
\caption{\textsc{Metadata for the subsolver efficiency comparison}} \label{tb:diff_solver_table}
\vskip0.15cm
\newcommand{\cell}[1]{{\!\!}#1{\!}}
\newcommand{\cellbf}[1]{{\!\!}{\color{blue}#1}{\!}}
\newcommand{\cellbff}[1]{{\!\!}{\color{red}#1}{\!}}
\begin{center}
\scriptsize{\begin{tabular}{l|rrr|rrr|rrr|rrr} \toprule
& \multicolumn{3}{|c|}{\cell{Estrogen ($p=692$)}} & \multicolumn{3}{|c}{\cell{Arabidopsis ($p=834$)}} &  \multicolumn{3}{|c}{\cell{Leukemia ($p=1255$)}}  & \multicolumn{3}{|c}{\cell{Hereditary($p=1869$)}}\\ \cmidrule{2-13}
\multicolumn{1}{c|}{\textsc{Sub-solvers}} & \multicolumn{1}{|c}{\cell{\#iter}} & \cell{\#chol} & \cell{time[s]} & \multicolumn{1}{|c}{\cell{\#iter}} & \cell{\#chol} & \cell{time[s]} & \multicolumn{1}{|c}{\cell{\#iter}} & \cell{\#chol} & \cell{time[s]} & \multicolumn{1}{|c}{\cell{\#iter}} & \cell{\#chol} & \cell{time[s]} \\ \cmidrule{2-13}
&\multicolumn{12}{|c}{$\rho = 0.5$} \\ \cmidrule{1-13}
& \multicolumn{3}{|c|}{\cell{$\#\textrm{nnz} = 0.022p^2$}} & \multicolumn{3}{|c}{\cell{$\#\textrm{nnz} = 0.030p^2$}} &  \multicolumn{3}{|c}{\cell{$\#\textrm{nnz} = 0.022p^2$}}  & \multicolumn{3}{|c}{\cell{$\#\textrm{nnz} = 0.020p^2$}} \\ \cmidrule{2-13}
\texttt{pFISTA}             &  \cell{9}      & \cell{29}       & \cell{13.10}   & \cell{10}      & \cell{35}     & \cell{24.76}     & \cellbf{9}     & \cell{31}     & \cell{286.57}   & \cell{17}    & \cell{80}   &\cell{1608.66} \\
\texttt{pFISTA[gpu]}     &  \cell{9}      & \cell{29}       & \cell{10.70}   & \cell{10}      & \cell{35}     & \cell{16.81}     & \cellbf{9}     & \cell{31}     & \cell{231.97}   & \cell{17}    & \cell{80}   & \cell{1265.97} \\
\texttt{dFISTA}    	        & \cell{8}       & \cellbf{16}    & \cell{4.66}     & \cell{10}      & \cellbf{17}  & \cell{10.92}     & \cell{14}     & \cellbf{22}   & \cell{50.19}    & \cell{14}    & \cellbf{27}  & \cell{147.86}\\
\texttt{dFISTA[gpu]}     & \cell{8}       & \cellbf{16}    & \cellbf{4.16}  & \cell{10}      & \cellbf{17}  & \cellbf{7.89}    & \cell{14}     & \cellbf{22}   & \cellbf{43.53}  & \cell{14}    & \cellbf{27}  & \cellbf{120.16}\\
\texttt{FastAS}             & \cellbf{7}    & \cell{24}       & \cell{28.69}   & \cellbf{8}     & \cell{27}     & \cell{96.93}     & \cellbf{9}    & \cell{31}      & \cell{532.11}   & \cellbf{11} & \cell{40}     & \cell{1682.28}\\
\texttt{BCDC}              & \cell{8}        & \cell{25}       & \cell{90.35}   & \cell{9}       & \cell{28}     & \cell{227.27}   & \cellbf{9}     & \cell{31}     & \cell{549.80}   & \cell{12}    & \cell{47}     & \cell{3452.82}\\
\texttt{MatQUIC}         & \cell{11}      &\cell{29}        & \cell{21.61}   & \cell{10}      & \cell{35}     & \cell{50.67}    & \cell{10}      & \cell{35}      & \cell{119.06}   & \cell{14}    & \cell{44}     & \cell{891.29}\\
\texttt{ProxGrad1}         & \cell{175}   & \cell{175}     & \cell{8.82}     & \cell{226}    & \cell{226}   & \cell{17.78}     & \cell{230}   & \cell{230}    & \cell{44.06}     & \cell{660}  & \cell{660}   & \cell{350.52}
\\ \midrule
&\multicolumn{12}{|c}{$\rho = 0.1$} \\ \cmidrule{1-13}
& \multicolumn{3}{|c|}{\cell{$\#\textrm{nnz} = 0.072p^2~(\sim 6\%)$}} & \multicolumn{3}{|c}{\cell{$\#\textrm{nnz} = 0.074p^2$}} &  \multicolumn{3}{|c}{\cell{$\#\textrm{nnz} = 0.065p^2$}}  & \multicolumn{3}{|c}{\cell{$\#\textrm{nnz} = 0.063p^2$}} \\ \cmidrule{2-13}
\texttt{pFISTA}             &  \cell{34}     & \cell{101}     & \cell{357.25}     & \cell{57}     & \cell{148}   & \cell{1056.90}   & \cell{143}    & \cell{242}     & \cell{7490.27}      & \cell{-}        & \cell{-}    & \cell{-} \\
\texttt{pFISTA[gpu]}     &  \cell{34}     & \cell{101}     & \cell{300.90}     & \cell{57}     & \cell{148}   & \cell{730.07}   & \cell{143}    & \cell{242}     & \cell{6083.06}      & \cell{-}        & \cell{-}    & \cell{-} \\
\texttt{dFISTA}    	        & \cell{14}      & \cellbf{32}    & \cell{12.51}       & \cellbf{12}    & \cellbf{35} & \cell{15.53}        & \cellbf{12}   & \cellbf{34}    & \cell{38.73}          & \cellbf{14}     & \cellbf{44} & \cell{150.03}\\
\texttt{dFISTA[gpu]}     & \cell{14}      & \cellbf{32}    & \cellbf{11.18}    & \cellbf{12}    & \cellbf{35} & \cellbf{11.18}     & \cellbf{12}    & \cellbf{34}   & \cellbf{33.45}       & \cellbf{14}     & \cellbf{44} & \cell{121.37}\\
\texttt{FastAS}             & \cell{-}      & \cell{-}          & \cell{-}               & \cell{-}         & \cell{-}       & \cell{-}                & \cell{-}         & \cell{-}          & \cell{-}                 & \cell{-}        & \cell{-}    & \cell{-}\\
\texttt{BCDC}              & \cellbf{13}    & \cell{48}       & \cell{1839.17}   & \cell{15}      & \cell{50}     & \cell{4806.62}    & \cell{-}         & \cell{-}         & \cell{-}                  & \cell{-}       & \cell{-}    & \cell{-}\\
\texttt{MatQUIC}         & \cell{30}       & \cell{88}       & \cell{573.87}     & \cell{36}      & \cell{95}     & \cell{1255.13}    & \cell{36}      & \cell{95}      & \cell{4260.97}      & \cell{-}       & \cell{-}     & \cell{-}\\
\texttt{ProxGrad1}         & \cell{4345}   & \cell{4345}   & \cell{224.95}     & \cell{6640}  & \cell{6640} & \cell{532.77}      & \cell{9225}  & \cell{9225}  & \cell{1797.49}      & \cell{-}       & \cell{-}    & \cell{-}
\\ \bottomrule
\end{tabular}}
\end{center}
\end{footnotesize}
\vskip -0.15in
\end{table*}

If the parameter $\rho$ is relatively large (i.e., the solution is expected to be quite sparse), \texttt{FastAS}, \texttt{BCDC} and \texttt{MatQUIC} perform well and converge in a reasonable time. This is expected since all three approaches vastly rely on the sparsity of the solution: the sparser the solution is, the faster their computations are performed, as restricted on the active set of variables. However, when $\rho$ is small, the performance of these methods significantly degrade due to the increased number of active (non-zero) entries.

Aside from the above, \texttt{ProxGrad1} performs well in terms of computational time, as compared to the rest of the methods. Unfortunately, the number of Cholesky decompositions in this method can become as many as the number of iterations, which indicates a computational bottleneck in high-dimensional problem cases. Moreover, when $\rho$ is small, this method also slows down and requires more iterations to converge. 

On the other hand, we also note that \texttt{pFISTA} is rather sensitive to the accuracy of the subsolver within the quadratic convergence region. In fact, while \texttt{pFISTA} reaches medium scale accuracies in a manner similar to \texttt{dFISTA}, it spends most of its iterations trying to achieve the higher accuracy values. 

\subsection{Proximal-gradient algorithm in action}
In this subsection, we illustrate the performance of proximal gradient algorithm in practice on various problems with different regularizers. 

\subsubsection{Linear convergence}\label{subsec:linear_convergence}
To show the  linear convergence of \texttt{ProxGrad1} (Algorithm \ref{alg:A1}) in practice, we consider the following numerical test.
Our experiment is based on the \texttt{Lymph} and \texttt{Estrogen} problems downloaded from \url{http://ima.umn.edu/~maxxa007/send_SICS/}. For both problem cases, we use different values for $\rho $ as $\rho = [0.1: 0.05: 0.6] $ in MATLAB notation.
For each configuration, we measure the quantity
\begin{equation}\label{eq:restricted_cond}
c^k_{\mathrm{res}} := \frac{\norm{(\mathbf{D}_k - \nabla^2f(\xb^{*}))\db^k_g}^{*}_{\xb^{*}}}{\Vert\db_g^k\Vert_{\xb^{*}}},
\end{equation}
for few last iterations. This quantity can be referred to as the restricted approximation gap of $D_k$ to $\nabla^2f(\xb^{*})$ along the proximal-gradient direction $\db^k_g$. 
We first run the proximal-Newton method up to $10^{-16}$ accuracy to obtain the solution $\xb^{*}$ and then run the proximal-gradient algorithm up to $10^{-8}$ accuracy to compute $c^k_{\mathrm{res}}$ and the norm $\Vert\xb^k - \xb^{*}\Vert_{\xb^{*}}$.
From the proof of Theorem \ref{th:convergence_of_grad_method}, we can show that if $\mathrm{c}_{\mathrm{res}}^k < 0.5$ for sufficiently large $k$, then the sequence $\set{\xb^k}_{k\geq 0}$ locally converges to $\xb^{*}$ at a linear rate.
We note that this condition is much weaker than the last condition given in Theorem \ref{th:convergence_of_grad_method} but more difficult to interpret. Note that the requirement in Theorem \ref{th:convergence_of_grad_method}  leads to a restriction on the condition number of $\nabla^2f(\xb^{*})$ to be less than $3$.
We perform this test on two problem instances with $11$ different values of the regularization parameter and then compute the median of $c^k_{\mathrm{res}}$ for each problem.
Figure \ref{fig:cond_number} shows the median of the restricted approximation gap $c^k_{\mathrm{res}}$ and the real condition number of $\nabla^2f(\xb^{*})$, respectively.
\begin{figure}[ht]
\vskip-0.05in
\begin{center}
\centerline{\subfigure[\texttt{Lymph} dataset ($p = 578$)]{\includegraphics[width=0.24\linewidth]{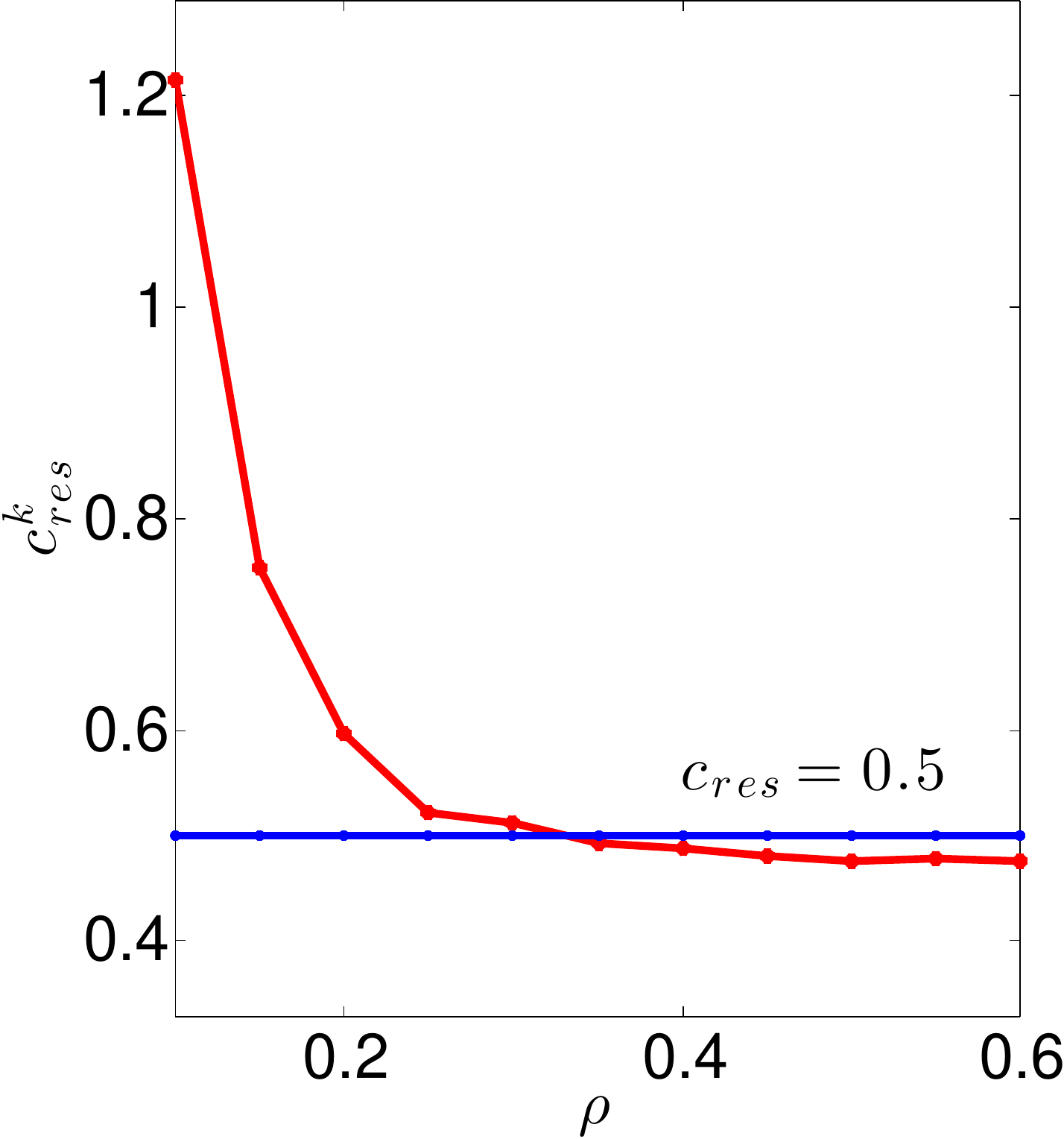} \includegraphics[width=0.245\linewidth]{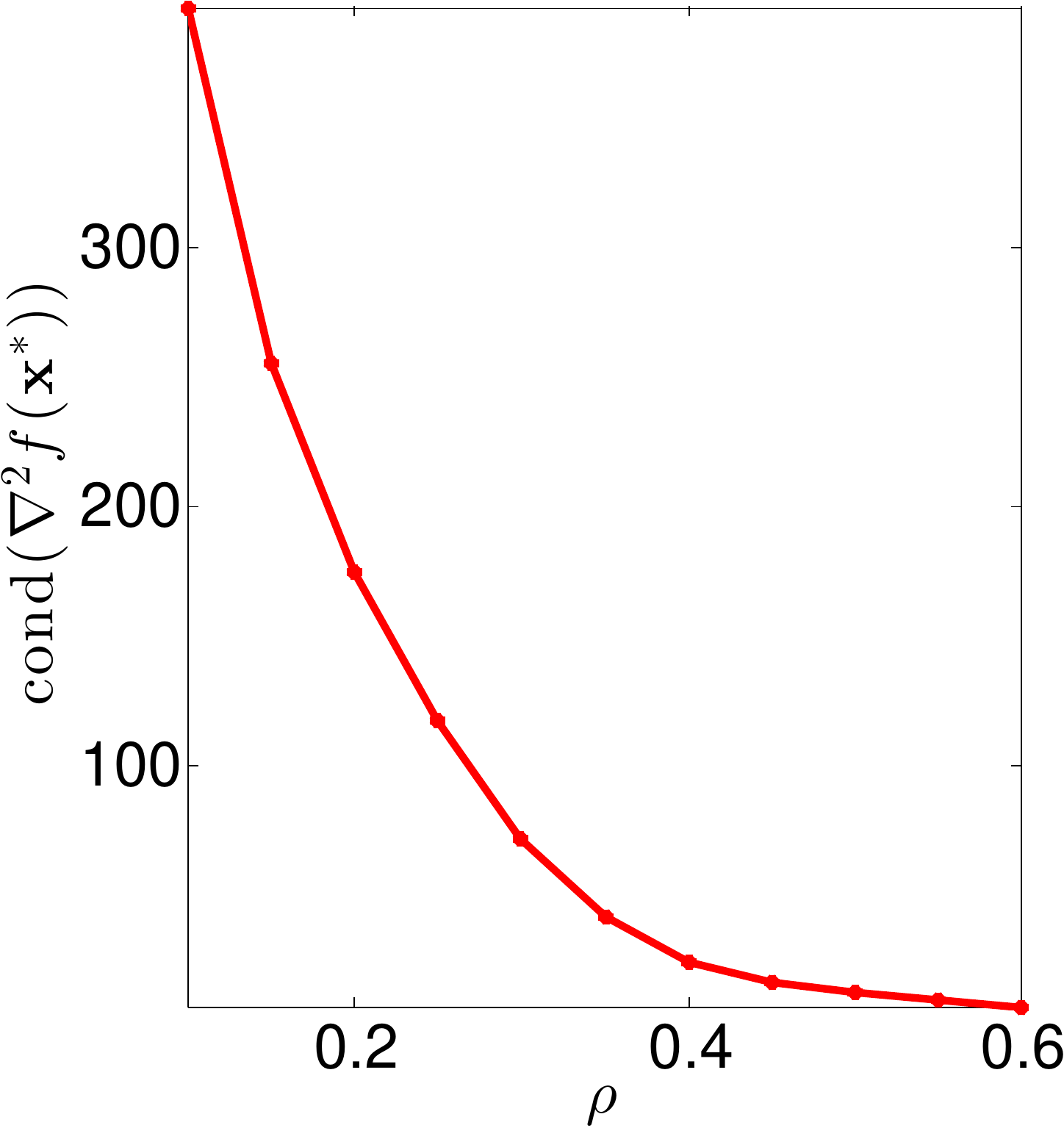}} \subfigure[\texttt{Estrogen} dataset ($p = 692$)]{\includegraphics[width=0.245\linewidth]{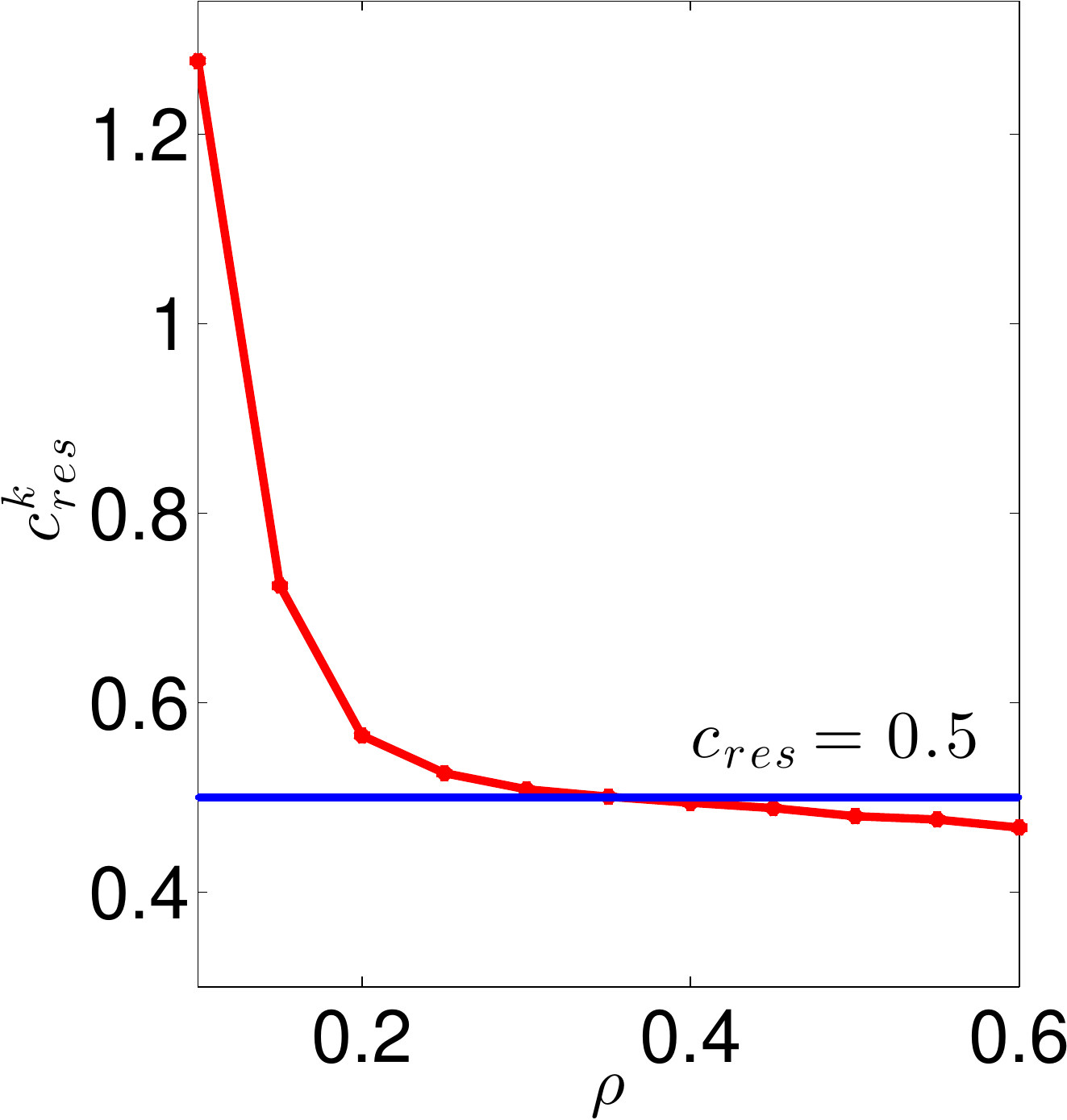}\includegraphics[width=0.25\linewidth]{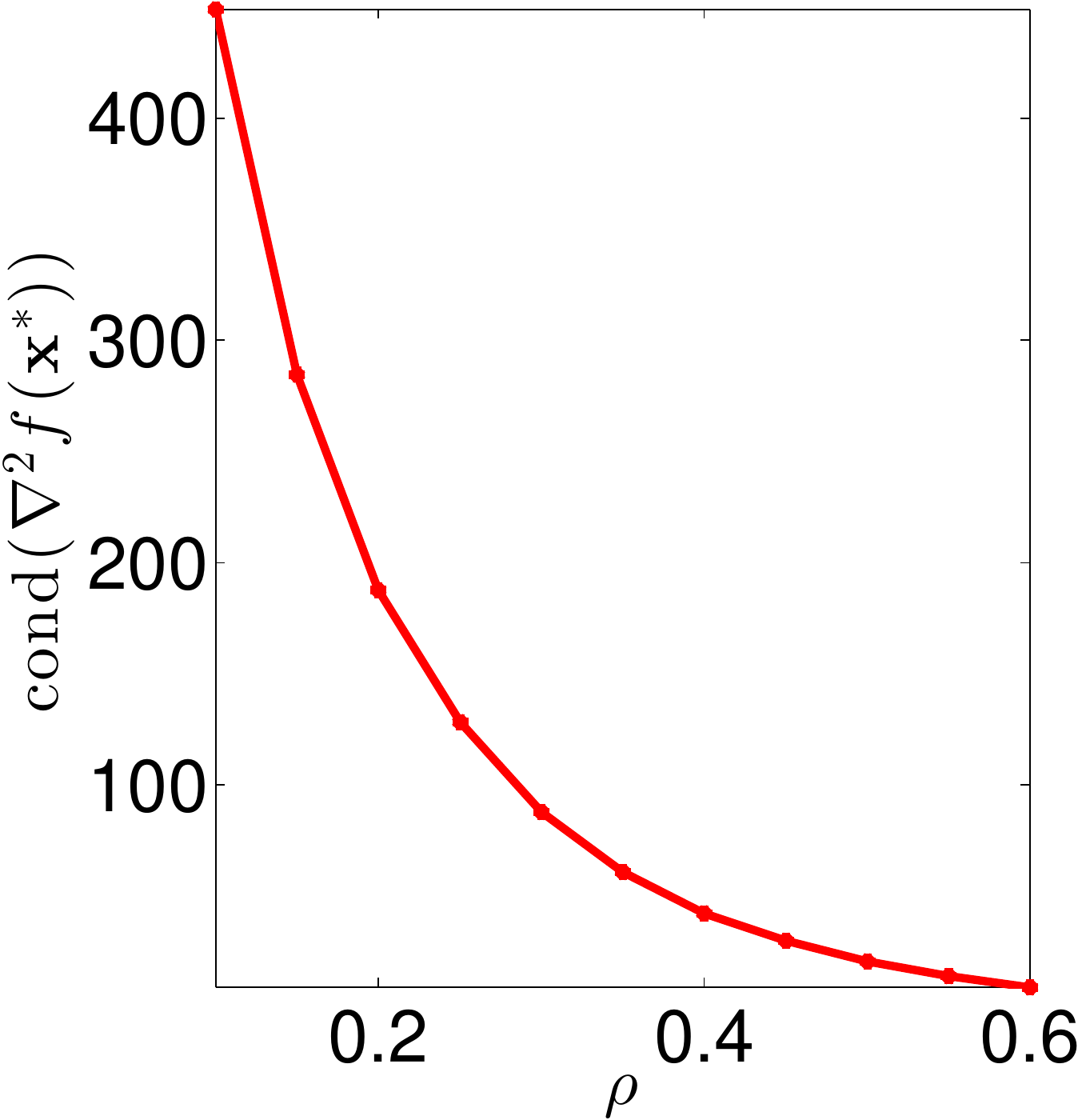}}}
\vskip -0.1in
\caption{For each test case: \textbf{(Left)} Restricted approximation gap $c^k_{\mathrm{res}}$ ~\textbf{(Right)} The actual condition number of $\nabla^2 f(\mathbf{x}^{\ast})$.} \label{fig:cond_number}
\end{center}
\vskip -0.3in
\end{figure} 

As expected, we observe that the real condition number of $\nabla^2f(\xb^{*})$ increases as the regularization parameter decreases. 
Moreover, the last condition given in Theorem \ref{th:convergence_of_grad_method} does not hold in this example. 
However, if we look at the restricted condition number computed by \eqref{eq:restricted_cond}, we can observe that for $\rho \gtrsim 0.3$, this value is strictly smaller than $0.5$.
In this case, the local linear convergence is actually observed in practice.

While $\mathrm{c}_{\mathrm{res}}^k < 0.5$ is only a sufficient condition and can  possibly be improved, we find it to be a good indicator of the convergence behavior.  
Figure \ref{fig:linear_conv} shows the last 100 iterations of our gradient method for the  \texttt{Lymph} problem with $\rho = 0.15$ and $\rho = 0.55$. 
The number of iterations needed to achieve the final solution in these cases is $1525$ and $140$, respectively.
In the former case, the calculated restricted condition number is above $0.5$ and the final convergence rate suffers. For instance, the contraction factor $\kappa$ in the estimate $\norm{\xb^{k+1} - \xb^{*}}_{\xb^{*}} \leq \kappa \norm{\xb^k - \xb^{*}}_{\xb^{*}}$ is close to $1$ when $\rho = 0.15$, while it is smaller when $\rho = 0.55$.
We can observe from Figure \ref{fig:linear_conv} (left) that the error $\Vert\xb^k - \xb^{*}\Vert_{\xb^{*}}$ drops rapidly at the last few iterations due to the affect of the bisection procedure, where we check the condition \eqref{eq:L_condition} for $\lambda_k < 1$.
\begin{figure}[ht]
\vskip-0.05in
\begin{center}
\centerline{\includegraphics[width=15.2cm, height=4.7cm]{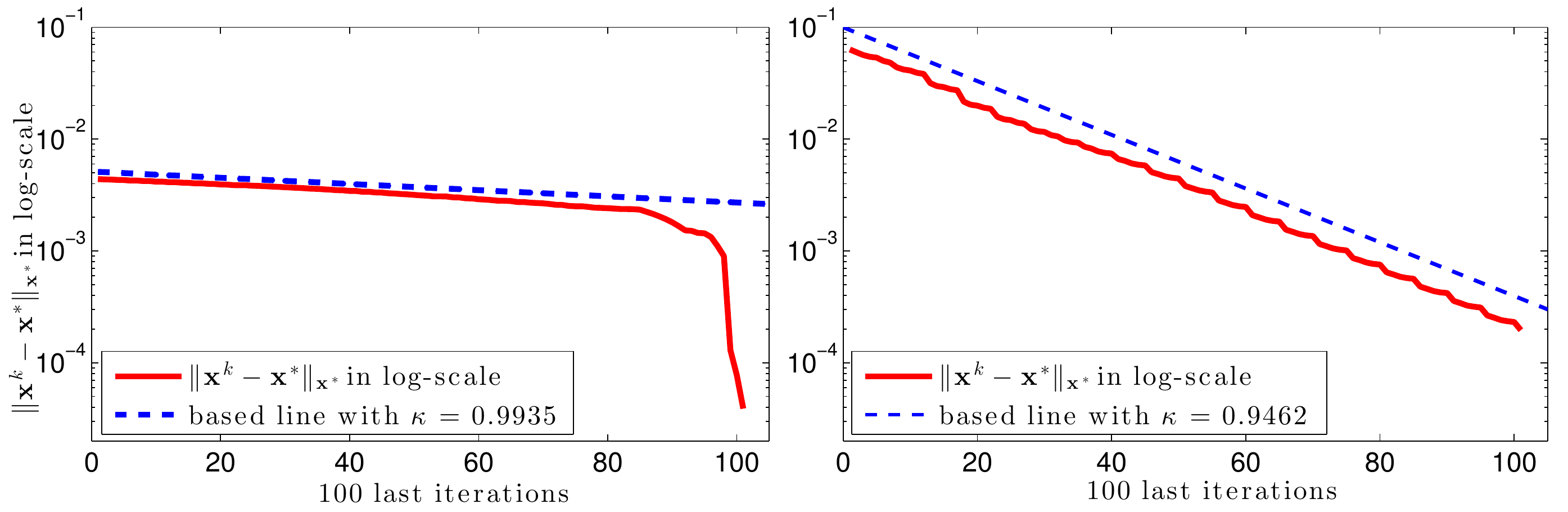}}
\vskip -0.1in
\caption{Linear convergence of \texttt{ProxGrad1} for \texttt{Lymph}: \textbf{Left}: $\rho = 0.15$ and \textbf{Right}: $\rho = 0.55$.}\label{fig:linear_conv}
\end{center}
\vskip -0.3in
\end{figure} 

\subsubsection{$\mathrm{TV}_{\ell_1}$-regularizer}
In this experiment, we consider the Poisson intensity reconstruction problem, where the regularizer $g$, the $\mathrm{TV}_{\ell_1}$-norm which is called the \textit{anisotropic}-TV; as an example, cf. \citep{Beck2009a}.
Hence, we implement Algorithm \ref{alg:A1_Poisson} (\texttt{ProxGrad2}) to solve \eqref{eq:Poisson_prob}, improve it using the greedy step-size modification as described in Section \ref{subsec:prox_gradient_method} (\texttt{ProxGrad2g}), and compare its performance with the state-of-the-art Sparse Poisson Intensity Reconstruction Algorithms (SPIRAL-TAP) toolbox \citep{Harmany2012}. 

As a termination criterion,  we have $\Vert\db^k_g\Vert_2 \leq 10^{-5}\max\set{1, \Vert\xb^k\Vert_2}$ or when the objective value does not significantly change after $5$ successive iterations, i.e., for each $k$, $\abs{f(\xb^{k+j}) - f(\xb^k)} \leq 10^{-8}\max\set{1, \abs{f(\xb^k)}}$ for $j=1,\dots, 5$.  

We first illustrate the convergence behavior of the three algorithms under comparison. 
We consider two image test cases: \texttt{house} and \texttt{cameraman}, and we set the regularization parameter of the $\mathrm{TV}_{\ell_1}$-norm to $\rho = 2.5\times 10^{-5}$.
Figure \ref{fig:convergence} illustrate the convergence of the algorithms both in iteration count and the timing.
\begin{figure}[ht]
\vskip-0.05in
\begin{center}
\centerline{\includegraphics[width=0.5\linewidth]{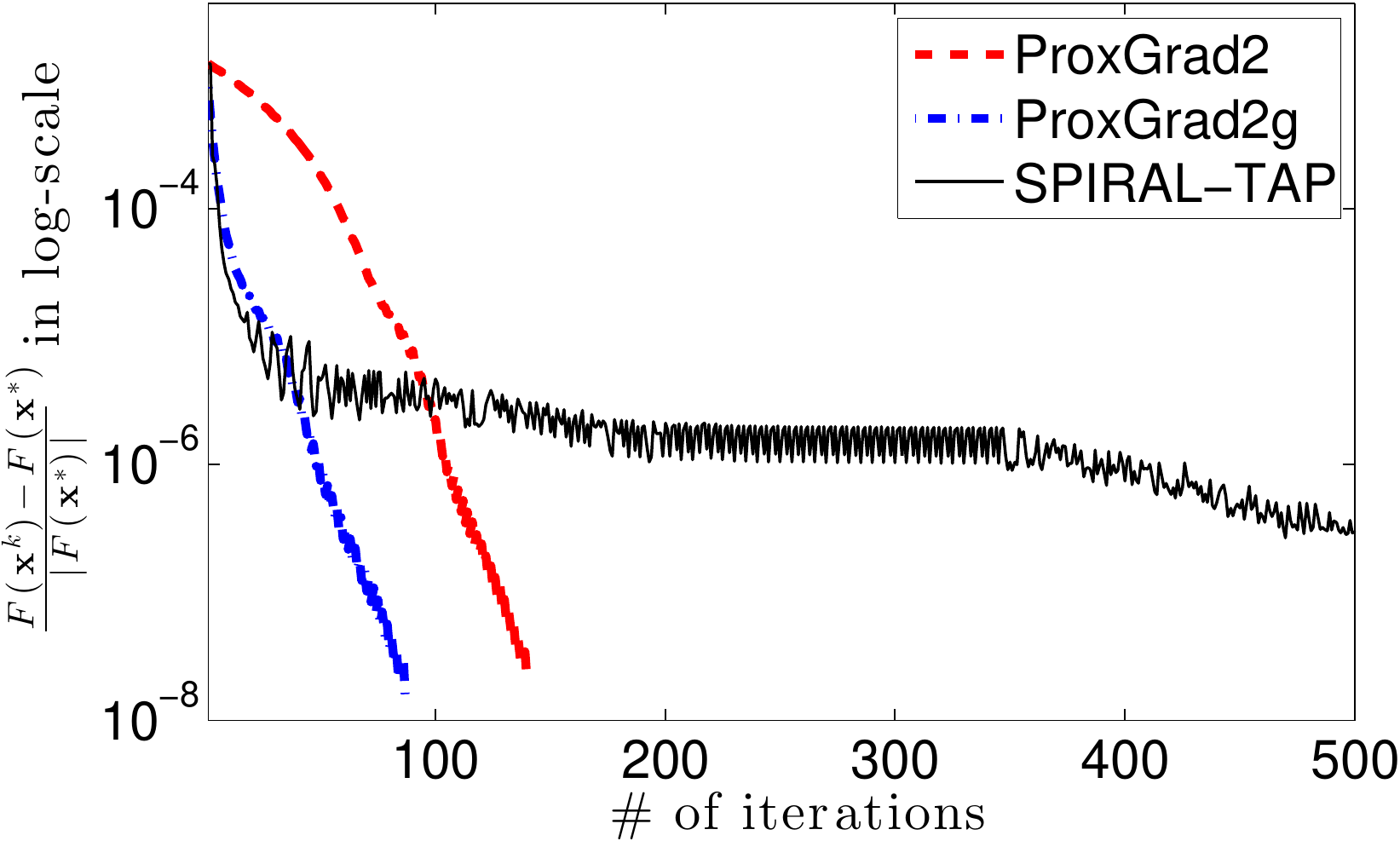} \includegraphics[width=0.49\linewidth]{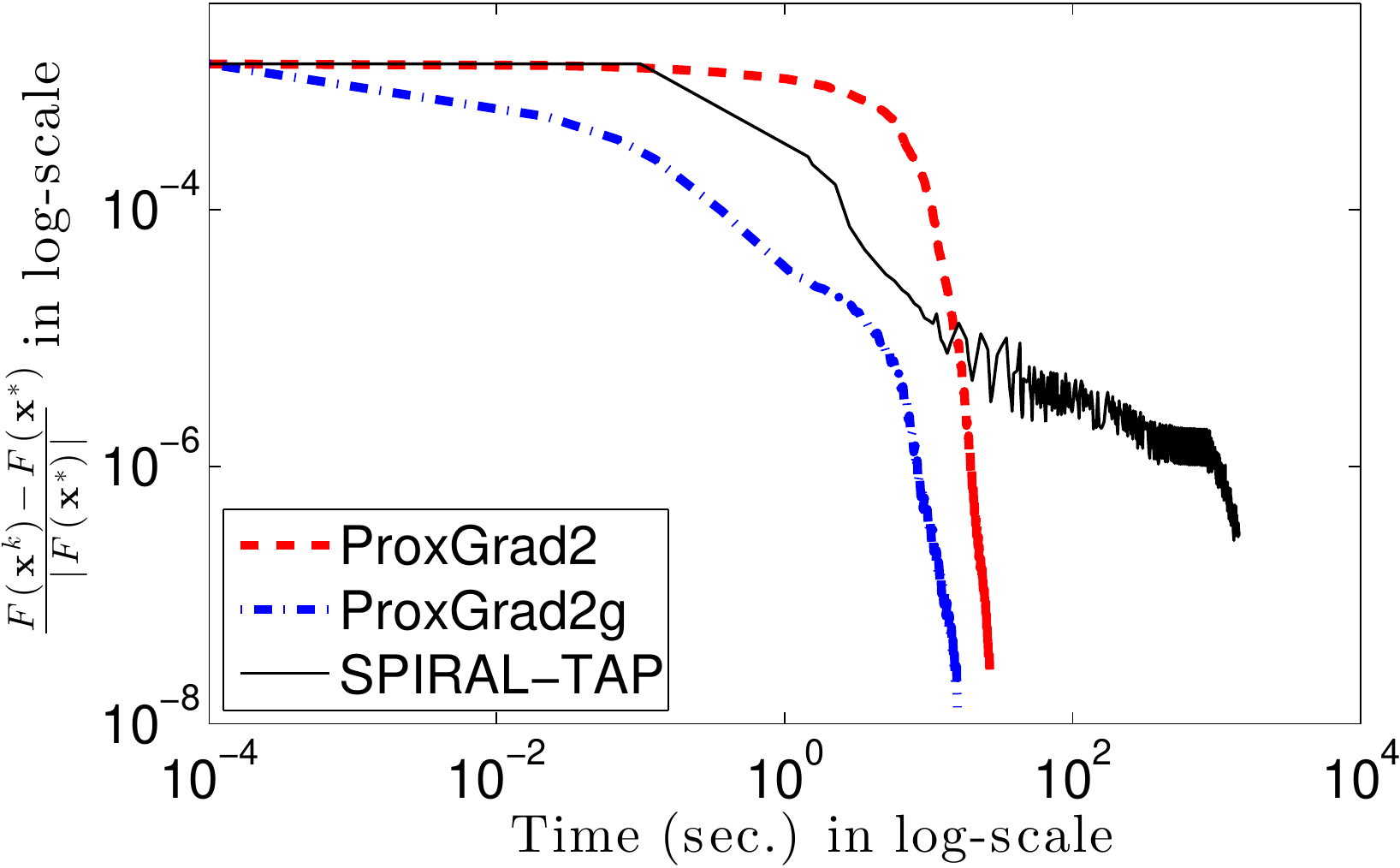}}
\vskip0.1in
\centerline{\includegraphics[width=0.5\linewidth]{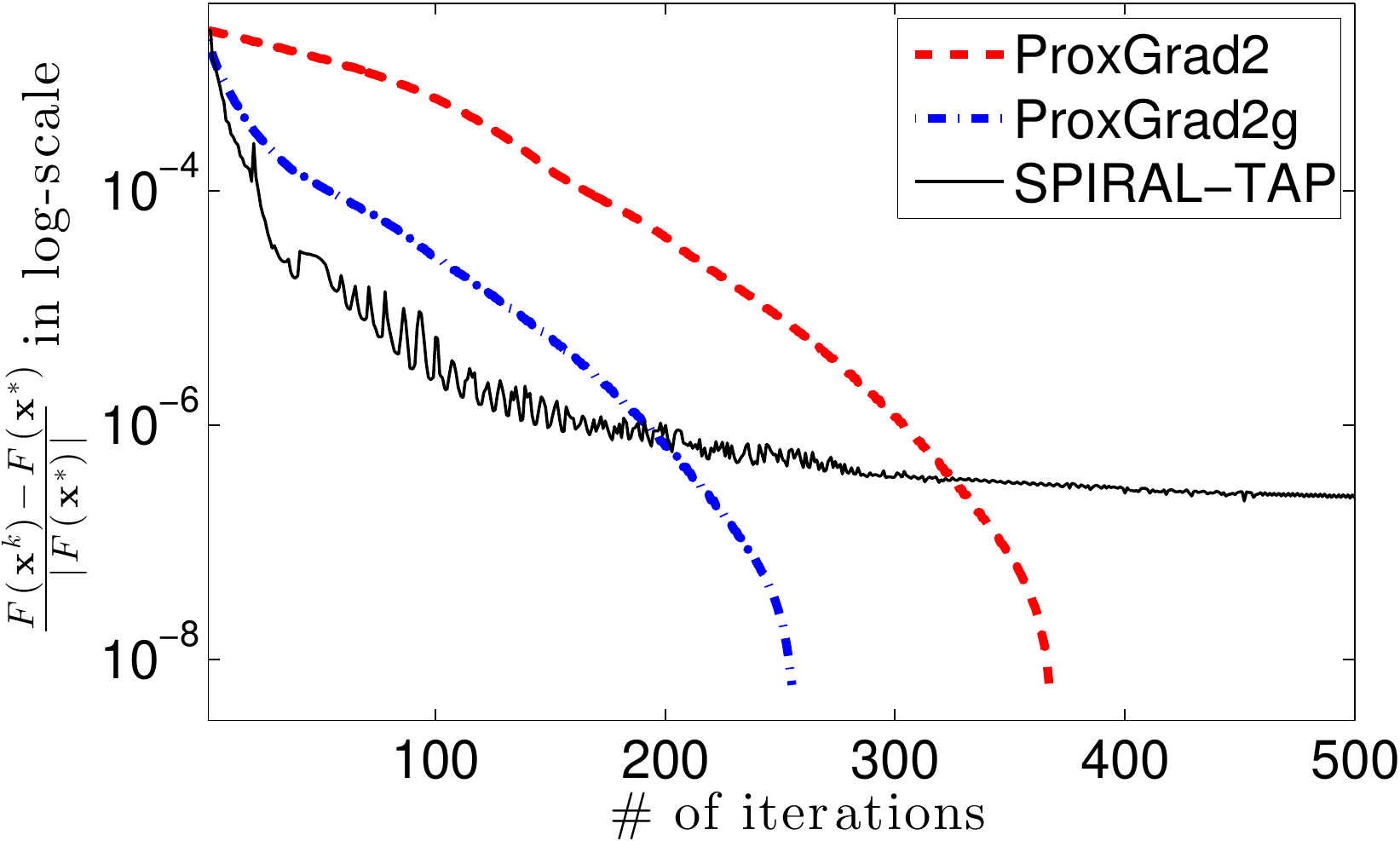}\includegraphics[width=0.49\linewidth]{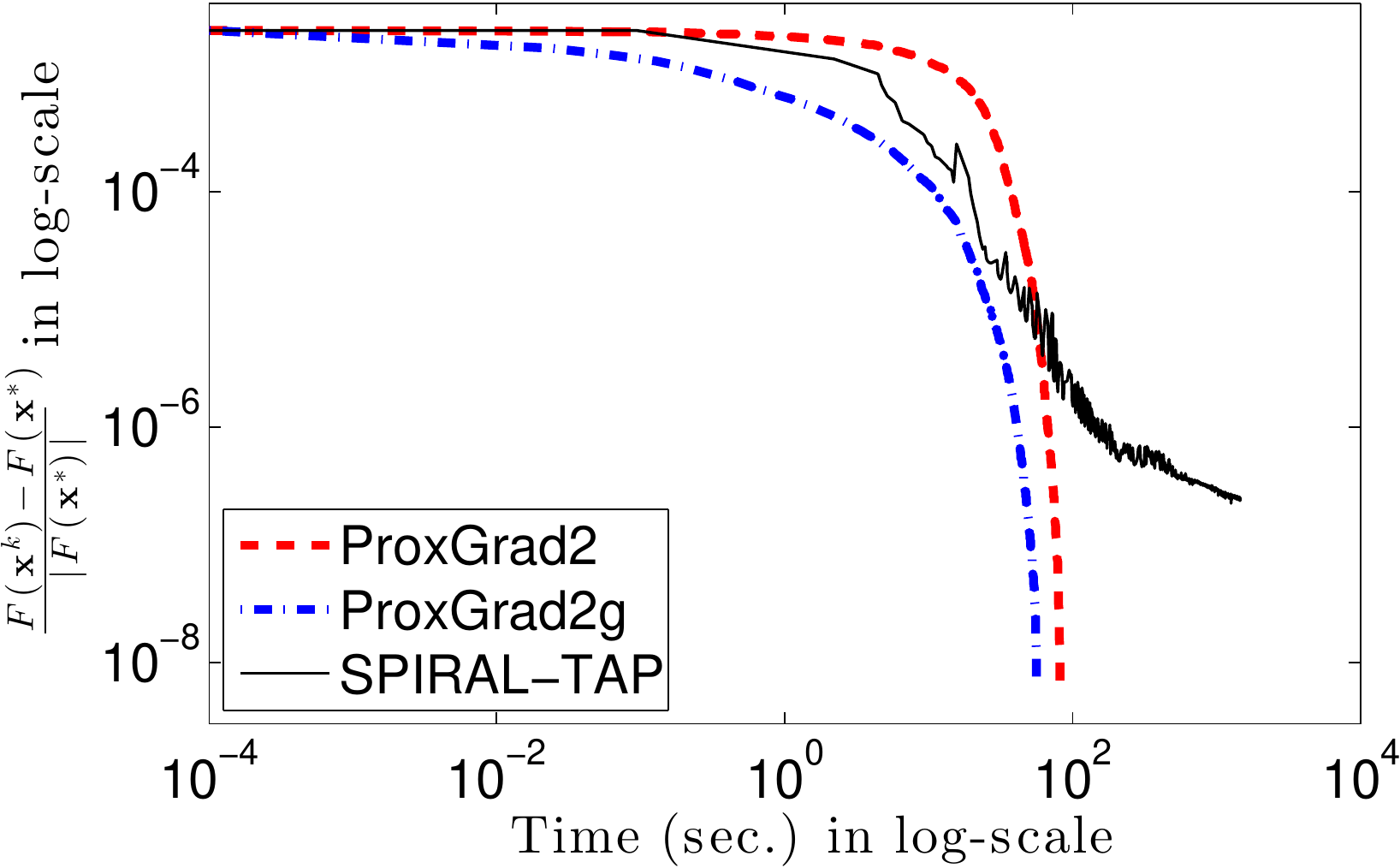}}
\vskip -0.1in
\caption{Convergence of three algorithms for \texttt{house} (top) and \texttt{cameraman} (bottom). \textbf{Left}: in iteration scale ~\textbf{Right}: in time log-scale.}\label{fig:convergence}
\end{center}
\vskip -0.3in
\end{figure} 

Overall, \texttt{ProxGrad2g} exhibits the best convergence behavior in terms of iterations and time. Due to the inaccurate solutions of the subproblem \eqref{eq:cvx_subprob_poisson},  the methods might exhibit oscillations.
Since SPIRAL-TAP employs a Barzilai-Borwein step-size and performs a line-search procedure up to very small step-size, the objective value is not sufficiently decreased; as a result of this, we observe more oscillations in the objective value. 

In stark contrast, \texttt{ProxGrad2} and \texttt{ProxGrad2g} use the Barzilai-Borwein step-size as an initial-guess for computing a search direction and then use the step-size correction procedure to ensure that the objective function decreases a certain amount at each iteration. This strategy turns out to be more effective since milder oscillations in the objective values are observed in practice (which are due to the inaccuracy of the TV-proximal operator).
\begin{figure}[!h]
\vskip-0.05in
\begin{center}
\centerline{\includegraphics[width=15.5cm, height=4.2cm]{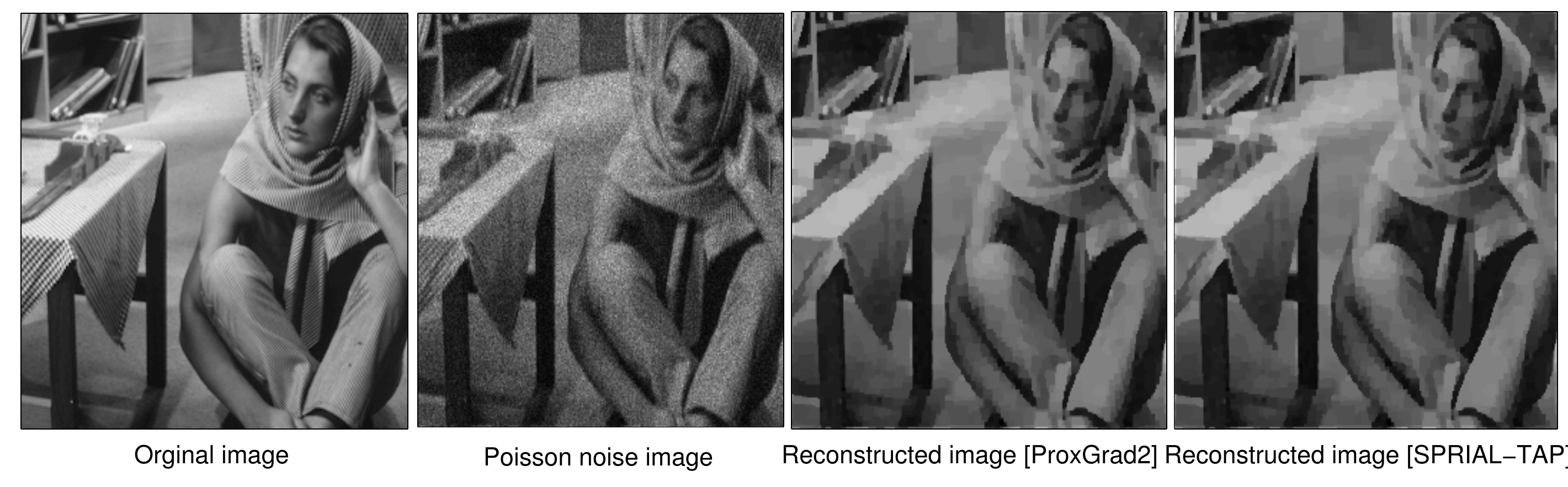}}
\vskip -0.1in
\caption{The reconstructed images for \texttt{barbara}  ($\rho = 2.5\times 10^{-5}$)}\label{fig:barbara_25}
\end{center}
\vskip -0.3in
\end{figure} 

Finally, we test the performance of \texttt{ProxGrad2}, \texttt{ProxGrad2g} and \texttt{SPIRAL-TAP} on $4$ different image cases: \texttt{barbara, cameraman, house} and \texttt{lena}. We set $\rho$ to two different values: $\rho \in \lbrace 10^{-5}, ~2.5 \cdot 10^{-5} \rbrace$. These values are chosen in order to obtain the best visual reconstructions (e.g., see Figure \ref{fig:barbara_25}) and are previously used in  \citep{Harmany2012}. The summary results reported in Table \ref{tb:real_data_frst_poisson}. 
Here, \textrm{AC} denotes the multiplicative factor in time acceleration of \texttt{ProxGrad2} as compared to \texttt{SPIRAL-TAP}, and $\Delta F$ is the difference between the corresponding obtained objective values between \texttt{ProxGrad2} and \texttt{SPIRAL-TAP} (a positive $\Delta F$ means that \texttt{SPIRAL-TAP} obtains a higher objective value at termination).
\begin{table*}[!h]
\begin{scriptsize}
\caption{\textsc{The results and performance of three algorithms}} \label{tb:real_data_frst_poisson}
\vskip0.15cm
\newcommand{\cell}[1]{{\!\!}#1{\!}}
\newcommand{\cellbf}[1]{{\!\!}{\color{blue}#1}{\!}}
\newcommand{\cellr}[1]{{\!\!}{\color{red}#1}{\!}}
\begin{center}
\scriptsize{\begin{tabular}{c|r|rrr|rrrrr|rrr} \toprule
& \multicolumn{11}{c}{\texttt{ProxGrad2g / ProxGrad2 / SPIRAL-TAP} } \\ \cmidrule{2-13}
\textsc{Image} &\multicolumn{1}{|c|}{$\rho\times10^{-5}$} & \multicolumn{3}{|c|}{\cell{\#iteration}} & \multicolumn{3}{|c}{\cell{CPU time [s]}} & \multicolumn{2}{c|}{\cell{AC}} 
& \multicolumn{1}{|c}{\cell{$F^k_{\min}$}} & \multicolumn{2}{c}{\cell{$\Delta F$}}\\ \cmidrule{1-13}
\texttt{house} 			& \cell{1.0} & \cellbf{116} & \cell{256} & \cell{500} &\cellbf{27.45} & \cell{56.95} & \cell{1658.00} & \cell{60} & \cell{29} & \cell{-10718352.93} & \cell{ 0.31} & \cell{ 0.70} \\ 
\texttt{$(256\times 256)$} 	& \cell{2.5} & \cellbf{92} & \cell{244} & \cell{500} &\cellbf{18.18} & \cell{50.26} & \cell{1431.94} & \cell{79} & \cell{28} & \cell{-10711758.80} & \cell{ 3.20} & \cell{ 3.32} \\  \midrule
\texttt{barbara}   		& \cell{1.0} & \cellbf{200} & \cell{324} & \cell{500} &\cellbf{46.92} & \cell{77.77} & \cell{1204.36} & \cell{26} & \cell{15} & \cell{-7388497.47} & \cell{ 0.05} & \cell{ 0.30} \\ 
\texttt{$(256\times 256)$} 	& \cell{2.5} & \cellbf{164} & \cell{268} & \cell{500} &\cellbf{36.45} & \cell{67.98} & \cell{1620.95} & \cell{44} & \cell{24} & \cell{-7377424.50} & \cell{ 1.90} & \cell{ 2.02} \\  \midrule
\texttt{cameraman}	 	& \cell{1.0} & \cellbf{396} & \cell{516} & \cell{500} &\cellbf{99.56} & \cell{117.75} & \cell{389.79} & \cell{ 4} & \cell{ 3} & \cell{-9186631.65} & \cell{0.19} & \cell{ 0.07} \\ 
\texttt{$(256\times 256)$} 	& \cell{2.5} & \cellbf{256} & \cell{368} & \cell{500} &\cellbf{59.75} & \cell{85.25} & \cell{1460.62} & \cell{24} & \cell{17} & \cell{-9175307.33} & \cell{ 2.29} & \cell{ 2.31} \\  \midrule
\texttt{lena} 			& \cell{1.0} & \cellbf{152} & \cell{220} & \cell{500} &\cellbf{27.43} & \cell{41.31} & \cell{1212.69} & \cell{44} & \cell{29} & \cell{-5797053.79} & \cell{ 0.10} & \cell{ 0.10} \\ 
\texttt{$(204\times 204)$} 	& \cell{2.5} & \cell{304} & \cellbf{184} & \cell{500} &\cell{59.20} & \cellbf{36.77} & \cell{1132.04} & \cell{19} & \cell{31} & \cell{-5789554.53} & \cell{ 1.52} & \cell{ 1.25} 
\\ \bottomrule
\end{tabular}}
\end{center}
\end{scriptsize}
\vskip -0.15in
\end{table*}

From Table \ref{tb:real_data_frst_poisson} we observe that both \texttt{ProxGrad2} and \texttt{ProxGrad2g} are  superior to \texttt{SPIRAL-TAP}, both in terms of CPU time and the final objective value in majority of problems. As the table shows, 
\texttt{ProxGrad2g} can be $4$ to $79$ times faster than \texttt{SPIRAL-TAP}. Moreover, it reports a better objective values in all cases.

\subsubsection{A comparison to standard gradient methods based on $\mathcal{F}_L$ assumption}
In this subsection, we use the LASSO problem \eqref{eq:unvarLASSO} with unknown variance as a simple test case to illustrate the improvements over the ``standard'' methods. 
Note that the standard Lipschitz gradient assumption  no longer holds in this example due to the log-term $\log(\sigma)$.
For this comparison, we dub our algorithm as \texttt{ProxGrad3(g)} and compare it against  a state-of-the-art TFOCS software package \citep{Becker2011}. The input data is synthetically generated based on the linear model $\yb = \mathbf{X}\boldsymbol{\beta}^{*} + \mathbf{s}$, where $\boldsymbol{\beta}$ is the true sparse parameter vector; 
$\mathbf{X}$ is a Gaussian  $n\times p$  matrix and $\mathbf{s} \sim \mathcal{N}(0,\sigma^2)$, where $\sigma = 0.01$. In TFOCS, we configure the Nesterov's accelerated algorithm with two proximal operations (TFOCS-N07) and adaptive restart as well as the standard gradient method (TFOCS-GRA). Both options use a backtracking step-size selection procedure due to the presence of the logarithmic term in the objective. 

As we can see in Figure \ref{fig:convergence} and Table \ref{tbl:scheds_5probs} that \texttt{ProxGrad3g} performs the best and manages to converge to a high accuracy solution at a linear rate in both examples. Interestingly, we find the per iteration complexity of \texttt{ProxGrad3g} is similar to \texttt{ProxGrad3} and TFOCS-GRA. In terms of per iteration cost, TFOCS-N07 is the most expensive one as it uses dual prox operations and adaptive restart, and requires more backtracking operations. Hence, while it takes less iterations as compared to the TFOCS-GRA, it performs worse in terms of timing. For illustration purposes, we ran the algorithms to high accuracy. However, if a typical stopping criteria such as $10^{-6}$ is used, our algorithm \texttt{ProxGrad3g} obtains $\times 3$ to $\times 8$ speed-ups over the standard gradient algorithm with backtracking enhancements. 
\begin{figure}[ht]
\vskip-0.05in
\begin{center}
\centerline{\includegraphics[width=0.33\linewidth]{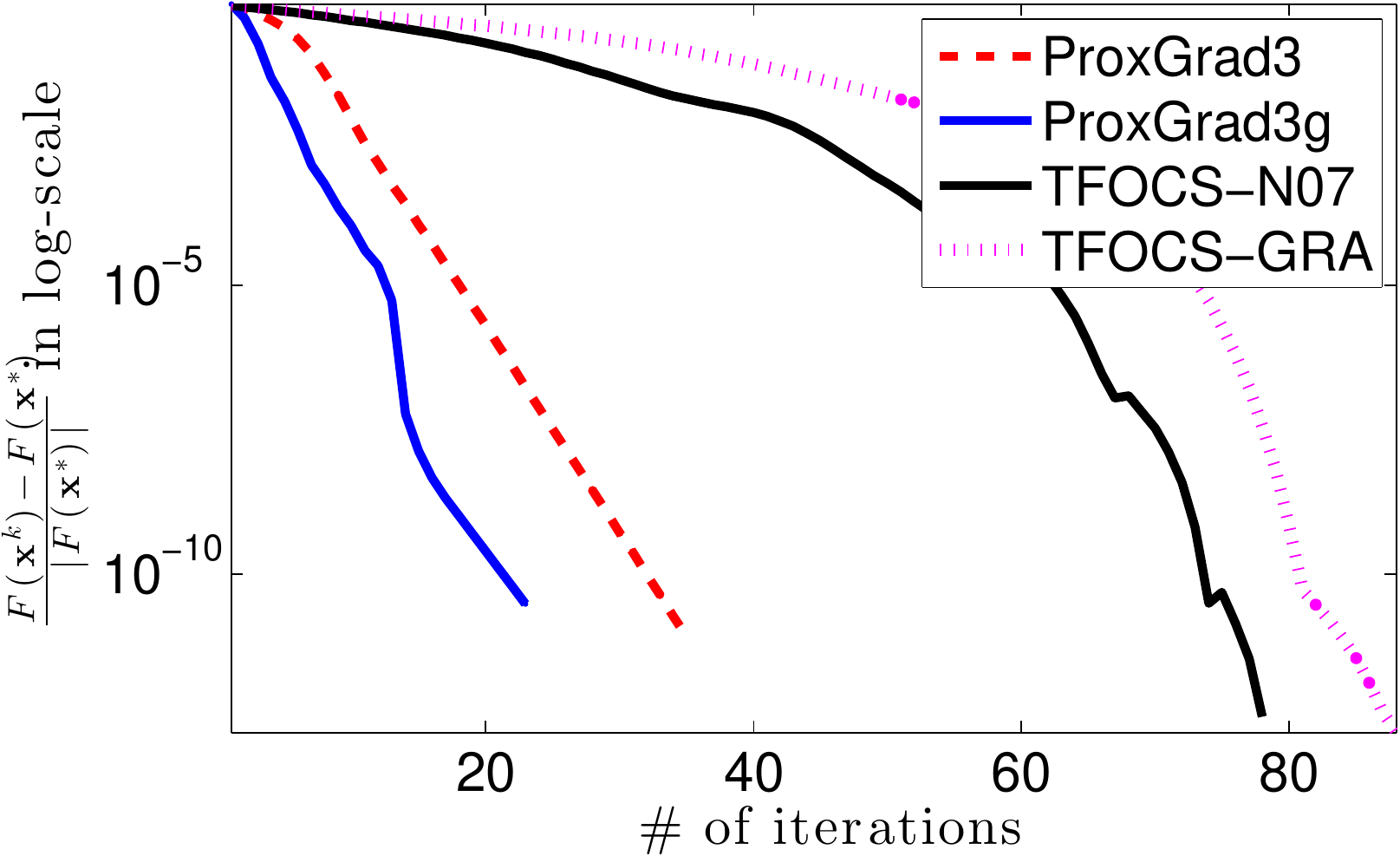} \includegraphics[width=0.33\linewidth]{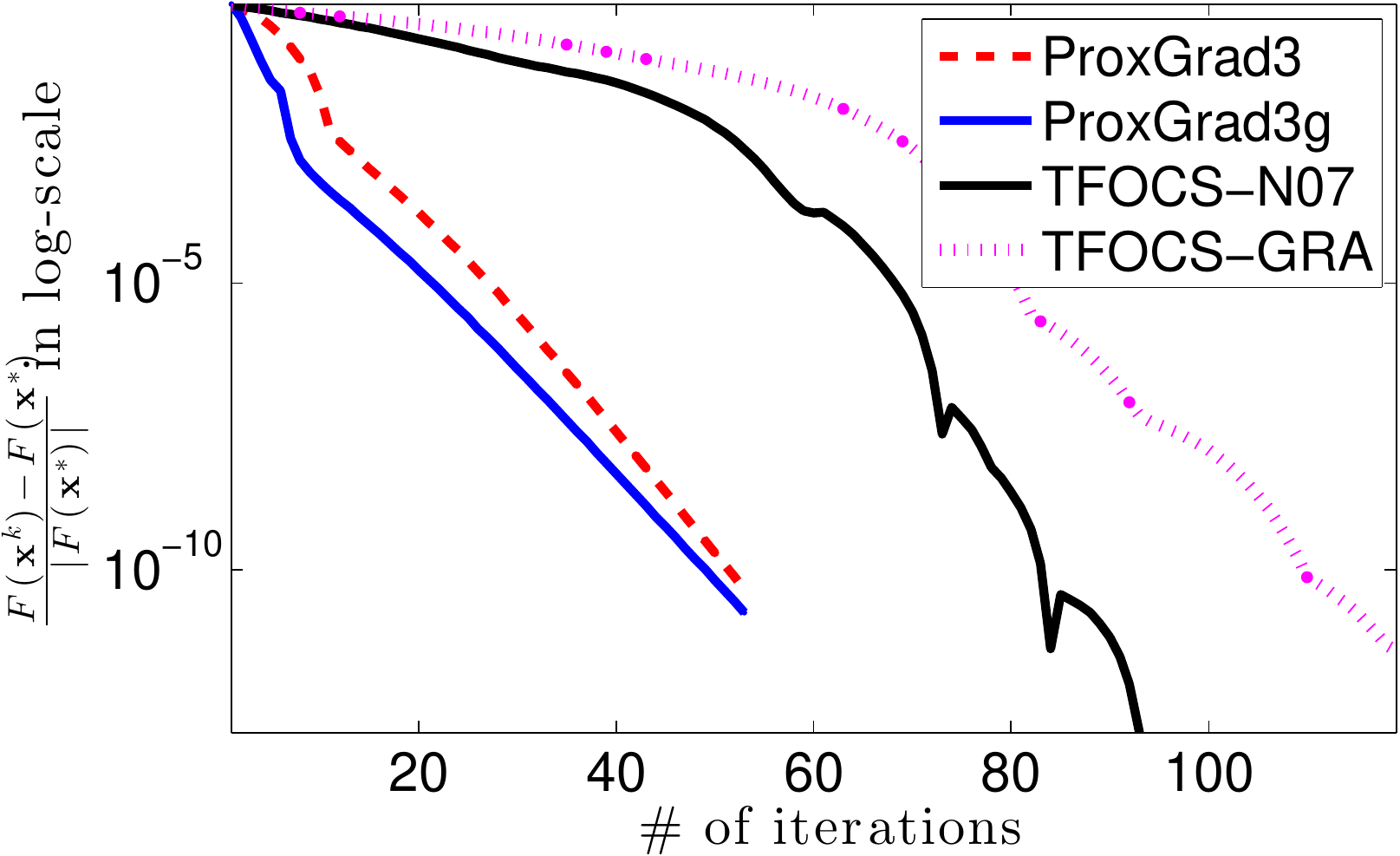} \includegraphics[width=0.33\linewidth]{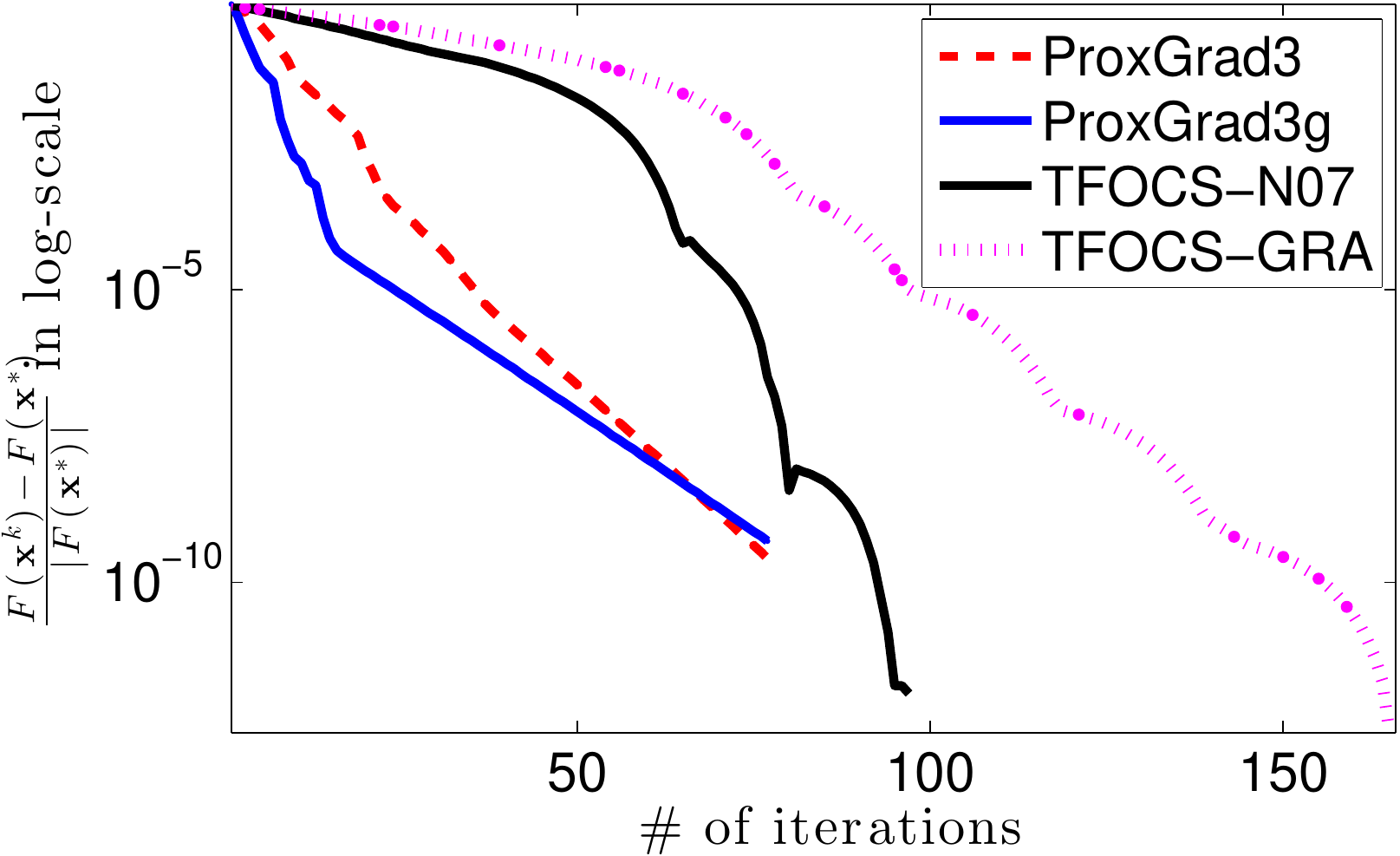}}
\vskip0.1in
\caption{Convergence plots of algorithms under comparison for $n = 3000$ and $p = 10000 $. From left to right, $\rho = 10^{-3}, \frac{2}{3}\cdot 10^{-4}, 5\cdot10^{-4}$.}\label{fig:convergence}
\end{center}
\vskip -0.3in
\end{figure} 

\begin{figure}[ht]
\vskip-0.05in
\begin{center}
\centerline{\includegraphics[width=0.33\linewidth]{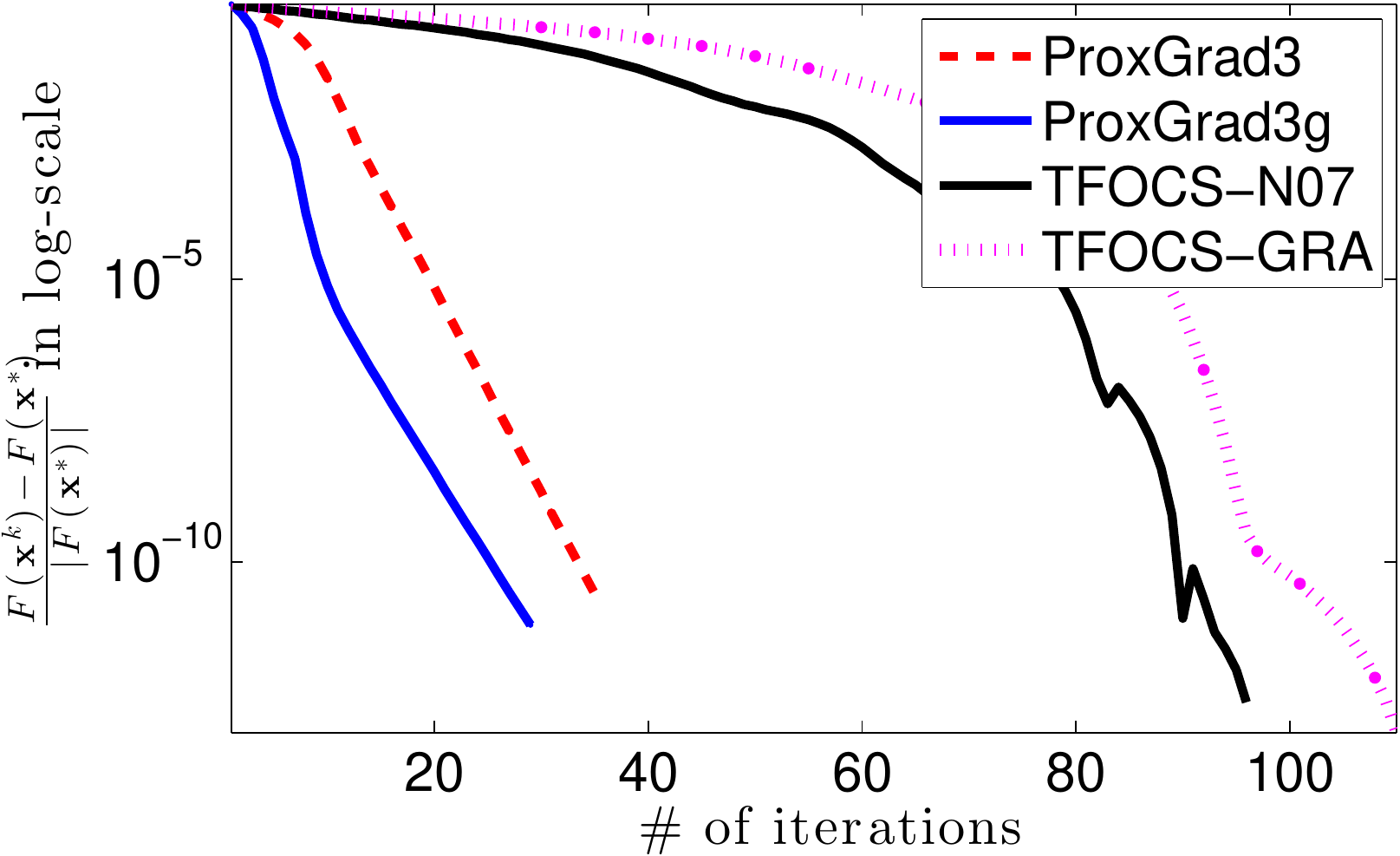} \includegraphics[width=0.33\linewidth]{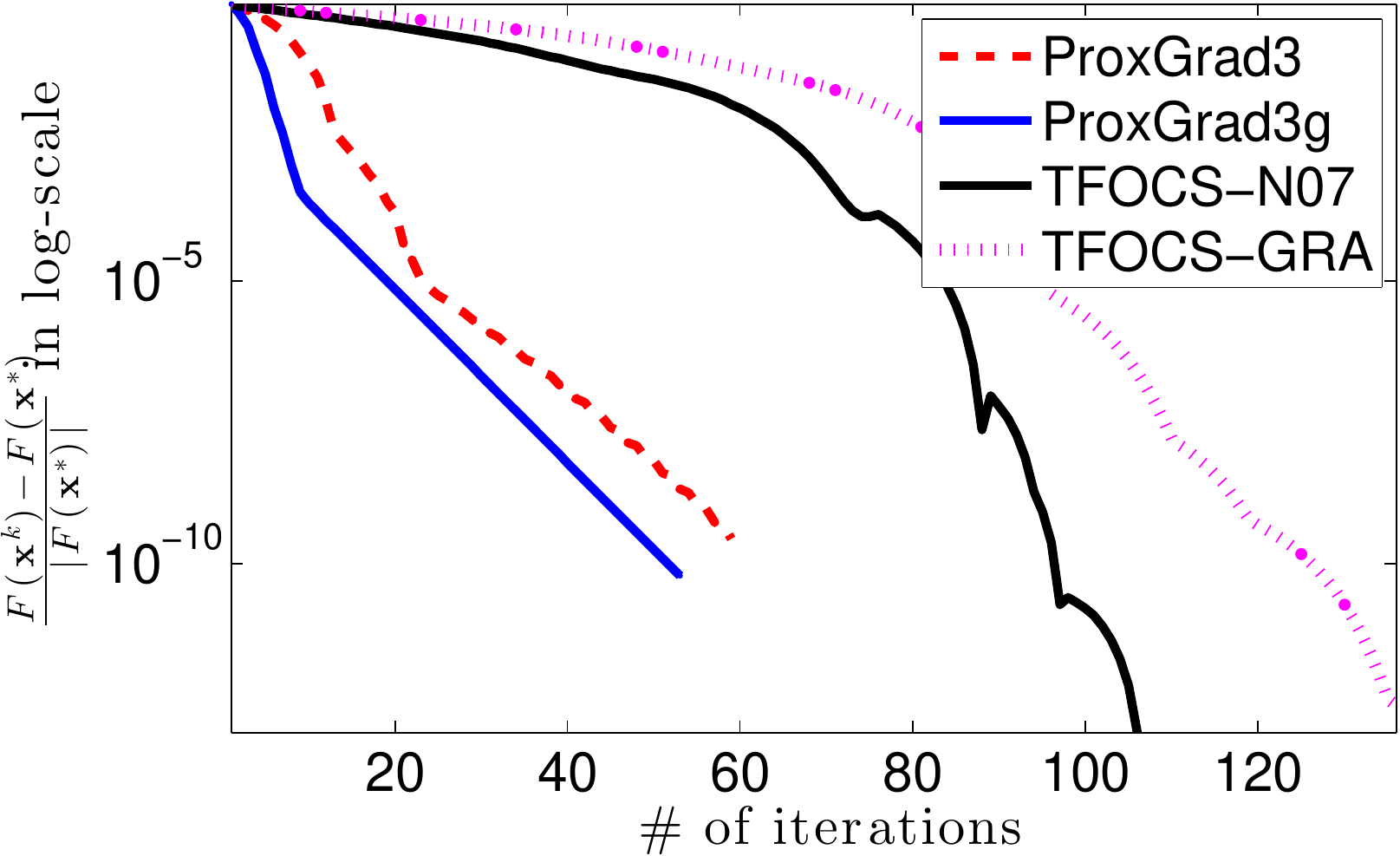} \includegraphics[width=0.33\linewidth]{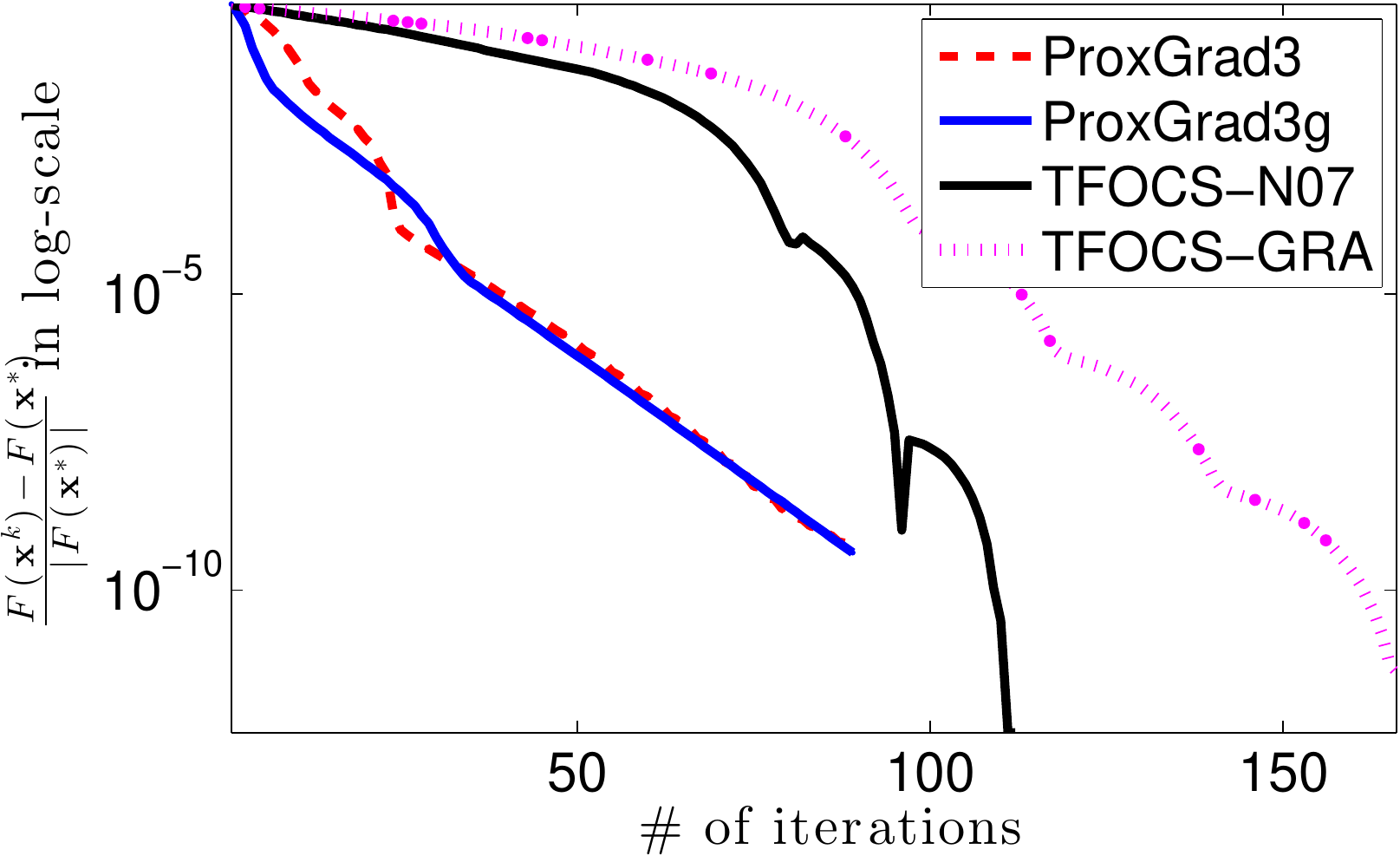}}
\vskip -0.1in
\caption{Convergence plots of algorithms under comparison for $n = 15000$ and $p = 50000$. From left to right, $\rho = 2\cdot 10^{-4}, \frac{4}{3}\cdot 10^{-4}, 10^{-4}$.}\label{fig:convergence}
\end{center}
\vskip -0.3in
\end{figure} 

\begin{table*}[!tp]
\begin{footnotesize}
\caption{\textsc{Metadata on the Lasso problem with unknown variance}} \label{tbl:scheds_5probs}
\vskip0.15cm
\newcommand{\cell}[1]{{\!\!}#1{\!}}
\newcommand{\cellbf}[1]{{\!\!}{\color{blue}#1}{\!}}
\newcommand{\cellbff}[1]{{\!\!}{\color{red}#1}{\!}}
\begin{center}
\scriptsize{\begin{tabular}{c|rrrr|rrrr|ccc	} \toprule
\textsc{Problem} & \multicolumn{11}{c}{ \texttt{ProxGrad3 / ProxGrad3g / TFOCS-N07 / TFOCS-GRA} } \\ \cmidrule{2-12}
($3000, 10000$) & \multicolumn{4}{|c|}{\cell{\#iteration}} & \multicolumn{4}{|c}{\cell{CPU time [s]}} & \multicolumn{1}{|c}{$\|\boldsymbol{\beta}\|_0$} & \multicolumn{1}{c}{$\|\widehat{\boldsymbol{\beta}}\|_0$} & \multicolumn{1}{c}{Overlap (\%)} \\ \cmidrule{1-12} 
$~\rho = 10^{-3}$ & \cell{36} & \cellbf{24} & \cell{79} & \cell{88} & \cell{1.0096} & \cellbf{0.7862} & \cell{3.2759} & \cell{1.7648} & \multirow{4}{*}{360} & \multicolumn{1}{|c}{166} & \cell{44.72} \\  
\cmidrule{1-9} \cmidrule{11-12} 
$~\rho = \frac{2}{3}\cdot 10^{-4}$ & \cellbf{54} & \cellbf{54} & \cell{94} & \cell{119} & \cell{1.2974} & \cellbf{1.2918} & \cell{3.6499} & \cell{2.4002} & & \multicolumn{1}{|c}{378} & \cell{92.22} \\ \cmidrule{1-9} \cmidrule{11-12} 
$~\rho =5\cdot 10^{-4}$ & \cellbf{78} & \cellbf{78} & \cell{97} & \cell{166} & \cellbf{1.7420} & \cell{1.7513} & \cell{3.7794} & \cell{3.3416} & & \multicolumn{1}{|c}{412} &  \cell{100} \\
\midrule \midrule
($15000, 50000$) & \multicolumn{4}{|c|}{\cell{\#iteration}} & \multicolumn{4}{|c}{\cell{CPU time [s]}} & \multicolumn{1}{|c}{$\|\boldsymbol{\beta}\|_0$} & \multicolumn{1}{c}{$\|\widehat{\boldsymbol{\beta}}\|_0$} & \multicolumn{1}{c}{Overlap (\%)}  \\ \cmidrule{1-12} 
$~\rho = 2\cdot 10^{-4}$ & \cell{36} & \cellbf{30} & \cell{99} & \cell{110} & \cell{21.7937} & \cellbf{19.3241} & \cell{82.3298} & \cell{46.0475} & \multirow{4}{*}{1800} & \multicolumn{1}{|c}{845} & \cell{44.98}  \\  
\cmidrule{1-9} \cmidrule{11-12} 
$~\rho = \frac{4}{3}\cdot 10^{-4}$ & \cell{60} & \cellbf{54} & \cell{108} & \cell{136} & \cell{31.7884} & \cellbf{29.1194} & \cell{89.4279} & \cell{57.9088}  & & \multicolumn{1}{|c}{1886} & \cell{87.91} \\ \cmidrule{1-9} \cmidrule{11-12} 
$~\rho = 10^{-4}$ & \cellbf{90} & \cellbf{90} & \cell{113} & \cell{166} & \cell{44.2692} & \cellbf{44.0611} & \cell{95.3060} & \cell{70.0946} & & \multicolumn{1}{|c}{2201} & \cell{100} \\
\bottomrule
\end{tabular}}
\end{center}
\end{footnotesize}
\vskip -0.15in
\end{table*}

\section{Conclusions}\label{sec:conclude}
We propose a variable metric method for minimizing convex functions that are compositions of proximity functions with self-concordant smooth functions. Our framework does not rely on the usual Lipschitz gradient assumption on the smooth part for its convergence theory. A highlight of this work is the new set of analytic step-size selection and correction procedures, which are best matched to the underlying problem structures. Our empirical results illustrate that the new theory leads to significant improvements in the practical performance of the algorithmic instances when tested on a variety of different applications. 

In this work, we present a convergence proof for composite minimization problems under the assumption of \emph{exact algorithmic calculations} at each step of the methods. As future research direction, an interesting problem to pursue is the extension of this analysis to include \emph{inexact calculations} and study how these errors propagate into the convergence and convergence rate guarantees \citep{Kyrillidis2014a}. We hope this paper triggers future efforts along this direction.

\section*{Acknowledgments}
This work is supported in part by the European Commission under Grant MIRG-268398, ERC Future Proof and SNF 200021-132548, SNF 200021-146750 and SNF CRSII2-147633.
The authors are also grateful to three anonymous reviewers as well as to the action editor for their thorough reviews of this work, comments and suggestions on improving the content and the presentation of this paper.

\appendix

\section{Technical proofs}\label{apdx:tech_proofs}
In this appendix, we provide the detailed proofs of the theoretical results in the main text. 
It consists of \textit{global convergence} and \textit{local convergence rate} of our algorithms and other technical proofs.

\subsection{Proof of Lemma \ref{le:unique_solution}}
Since $g$ is convex, we have $g(\yb)\geq g(\xb) + {\bf v}^T(\yb - \xb)$ for all $\vb \in\partial{g}(\xb)$.
By adding this inequality to  \eqref{eq:SC_bound1} and notting that  $F(\xb) := f(\xb) + g(\xb), ~\forall \xb$, we obtain
\begin{align}{\label{eq:00}}
F(\yb) &\geq F(\xb) + \left(\nabla{f}(\xb) + \vb\right)^T(\yb -\xb) + \omega(\norm{\yb - \xb}_{\xb}) \nonumber \\ &\geq F(\xb) - \lambda(\xb)\norm{\yb - \xb}_{\xb} + \omega(\norm{\yb - \xb}_{\xb}).
\end{align} 
Here, the last inequality is due to the generalized Cauchy-Schwartz inequality and $\lambda(\xb) := \norm{\nabla{f}(\xb) + \vb}_{\xb}^{*}$. 
Let $\mathcal{L}_F(F(\xb)) := \set{ \yb \in \dom{F} ~|~ F(\yb) \leq F(\xb)}$ be a sublevel set of $F$. Then, for any $\yb \in\mathcal{L}_F(F(\xb))$, we have $F(\yb) \leq F(\xb)$ which leads to
\begin{align}
\lambda(\xb)\norm{\yb -\xb}_{\xb} \geq \omega(\norm{\yb -\xb}_{\xb}), \nonumber
\end{align} 
due to \eqref{eq:00}.
Let $s(t) := \frac{\omega(t)}{t} = 1 - \frac{\ln(1 + t)}{t}$. The last inequality leads to $s(\norm{\yb -\xb}_{\xb}) \leq \lambda(\xb)$.
Since the equation $\ln(1 + t) = (1-\lambda(\xb))$ has unique solution $t^{*} > 0$ if $\lambda(\xb) < 1$. Moreover, the function $s$ is strictly increasing and $s(t) < 1$ for $t\geq 0$, which leads to $0 \leq t \leq t^{*}$. Since $s(\norm{\yb -\xb}_{\xb}) \leq \lambda(\xb)$, we have $\norm{\yb - \xb}_{\xb} \leq t^{\ast}$.
Thus, $\mathcal{L}_F(F(\xb))$ is bounded. Hence, $\xb^{*}$ exists due to the well-known Weierstrass theorem. 

The uniqueness of $\xb^{*}$ follows from the strict increase of $\omega(\cdot)$.
Indeed, for any $\xb\in\dom{F}$, by the convexity of $g$ we have $g(\xb) - g(\xb^{*}) \geq \vb_{*}^T(\xb - \xb^{*})$, where $\vb_{*}\in\partial{g}(\xb^{*})$.
By the self-concordant property of $f$, we also have $f(\xb) - f(\xb^{*}) \geq \nabla{f}(\xb^{*})^T(\xb - \xb^{*}) + \omega(\Vert\xb - \xb^{*}\Vert_{\xb^{*}})$. 
Adding these inequalities and using the optimality condition \eqref{eq:Fx_optimality}, i.e., $0 = \vb_{*} + \nabla{f}(\xb^{*})$, we have $F(\xb) - F(\xb^{*}) \geq \omega(\Vert\xb - \xb^{*}\Vert_{\xb^{*}})$. 
Now, let $\hat{\xb}^{*}\neq \xb^{*}$ is also an optimal solution of \eqref{eq:COP}. We have $0 = F(\hat{\xb}^{*}) - F(\xb^{*}) \geq  \omega(\Vert\xb - \xb^{*}\Vert_{\xb^{*}}) > 0$, which leads to  a contradiction. This implies that $\xb^{*}\equiv \hat{\xb}^{*}$.
\eproof

\subsection{Proofs of global convergence: Theorem \ref{th:choose_alpha} and Theorem \ref{th:global_convergence_of_grad_method}}
In this subsection, we provide the proofs of Theorem \ref{th:choose_alpha}, Lemma \ref{le:step_size} and Theorem \ref{th:global_convergence_of_grad_method} in a unified fashion. We first provide a key result quantifying the improvement of the objective as a function of the step-size $\alpha_k$.

\paragraph{Maximum decrease of the objective function:}
Let $\beta_k := \Vert \db^k \Vert_{\Hb^k}$, $\lambda_k := \Vert \db^k\Vert_{\xb^k}$ and 
\begin{equation*}
\xb^{k+1} := \xb^k + \alpha_k\db^k = (1-\alpha_k)\xb^k + \alpha_k\sb^k,
\end{equation*}
where $\alpha_k :=  \frac{\beta_k^2}{\lambda_k(\lambda_k + \beta_k^2)}\in (0, 1]$.
We will prove below that the following holds at each iteration of the algorithms 
\begin{equation}\label{eq:decrease_F_tmp}
F(\xb^{k+1}) \leq F(\xb^k) - \omega\left(\frac{\beta_k^2}{\lambda_k}\right).
\end{equation}
Moreover, the choice of $\alpha_k$ is \textit{optimal} (in the worse-case sense).

\begin{proof}
Indeed, since $g$ is convex and $\alpha_k \in (0, 1]$, we have $g(\xb^{k+1}) = g\left( (1-\alpha_k)\xb^k + \alpha_k\sb^k\right) \leq (1-\alpha_k)g(\xb^k) + \alpha_kg(\sb^k)$, which leads to\begin{align}{\label{eq:01}}
g(\xb^{k+1}) - g(\xb^k) \leq \alpha_k(g(\sb^k) - g(\xb^k)).
\end{align}
For $\xb^{k+1}\in\dom{F}$ so that $\Vert\xb^{k+1} - \xb^k\Vert_{\xb^k} < 1$, the bound \eqref{eq:SC_bound2} holds.
 Combining \eqref{eq:01} with the self-concordant property  \eqref{eq:SC_bound2} of $f$, we obtain
\begin{align}\label{eq:lm31_est4}
F(\xb^{k+1}) &\leq F(\xb^k) + \nabla{f}(\xb^k)^T(\xb^{k+1} - \xb^k) + \omega_{*}\left(\Vert\xb^{k+1} - \xb^k\Vert_{\xb^k}\right)  + \alpha_k\left(g(\mathbf{s}^k) - g(\xb^k)\right) \nonumber\\ 
&\stackrel{\eqref{eq:search_dir_dk}}{\leq} F(\xb^k) + \alpha_k\nabla{f}(\xb^k)^T\db^k + \omega_{*}\left(\alpha_k\Vert\db^k\Vert_{\xb^k}\right) + \alpha_k\left(g(\mathbf{s}^k) - g(\xb^k)\right).
\end{align}
Since $\mathbf{s}^k$ is the unique solution of \eqref{eq:cvx_subprob}, by using the optimality condition \eqref{eq:subprob_opt}, we get
\begin{align}{\label{eq:03}}
-\nabla{f}(\xb^k) - \Hb_k(\mathbf{s}^k - \xb^k) &\in \partial{g}(\mathbf{s}^k) \Rightarrow \nonumber \\ -\nabla{f}(\xb^k)^T(\mathbf{s}^k - \xb^k) - \Vert\mathbf{s}^k - \xb^k\Vert_{\Hb_k}^2 &\in (\mathbf{s}^k - \xb^k)^T\partial{g}(\mathbf{s}^k).
\end{align} 
Combining \eqref{eq:03} with $g(\xb^k) - g(\mathbf{s}^k) \geq \mathbf{v}^T(\xb^k - \mathbf{s}^k), ~\mathbf{v} \in \partial(\mathbf{s}^k)$, due to the convexity of $g(\cdot)$, we have
\begin{equation}\label{eq:lm31_est3}
g(\mathbf{s}^k) - g(\xb^k) \leq -\nabla{f}(\xb^k)^T(\mathbf{s}^k - \xb^k) - \Vert\mathbf{s}^k - \xb^k\Vert_{\Hb_k}^2.
\end{equation} 
Using \eqref{eq:lm31_est3} in \eqref{eq:lm31_est4} together with the definitions of $\beta_k$ and $\lambda_k$, we obtain
\begin{align}{\label{eq:04}}
F(\xb^{k+1}) &\overset{\eqref{eq:search_dir_dk}}{\leq} F(\xb^k) - \alpha_k\beta_k^2 + \omega_{*}\left(\alpha_k \lambda_k\right).
\end{align} 
Let us consider the function $\varphi(\alpha) := \alpha \beta_k^2 - \omega_{*}(\alpha\lambda_k)$. 
By the definition of $\omega_{\ast}(\cdot)$, we can easily show that $\varphi(\alpha)$ attains the maximum at
\begin{align}\label{eq:step_size_alpha}
\alpha_k := \frac{\beta_k^2}{\lambda_k(\lambda_k + \beta_k^2)},
\end{align} 
provided that $\alpha_k \in (0, 1]$.
We note that the choice of $\alpha_k$ as \eqref{eq:step_size_alpha} preserves the condition $\Vert\xb^{k+1} - \xb^k\Vert_{\xb^k} < 1$.
Moreover, $\varphi(\alpha_k) = \omega(\beta_k^2/\lambda_k)$, which  proves \eqref{eq:decrease_F_tmp}.
Since $\alpha_k$ maximizes $\varphi$ over $[0, 1]$, this value is optimal.
\end{proof}

\paragraph{Proof of Theorem \ref{th:choose_alpha}:}
Since $\Hb_k := \nabla^2 f(\xb^k)$, we observe $\beta_k := \Vert\db^k\Vert_{\Hb_k} \equiv \Vert\db^k\Vert_{\xb_k} =: \lambda_k$, where $\db^k \equiv \db^k_n$.
In this case, the step size $\alpha_k$ in \eqref{eq:step_size_alpha} becomes $\alpha_k = \frac{1}{1 + \lambda_k}$ which is in $(0, 1)$. 
Moreover, \eqref{eq:decrease_F_tmp} reduces to 
\begin{equation*}
F(\xb^{k+1}) \leq F(\xb^k) - \omega(\lambda_k),
\end{equation*}
which is indeed \eqref{eq:decrease_eq2}.

Finally, we assume that, for a given $\sigma \in (0, 1)$, we have $\lambda_k \geq \sigma$ for $0 \leq k \leq k_{\max} - 1$. Since $\omega$ strictly increases, it follows from \eqref{eq:decrease_eq2} by induction that
\begin{equation*}
F(\xb^{*}) \leq F(\xb^k) \leq F(\xb^0) - \sum_{j=0}^{k-1}\omega(\lambda_j) \leq F(\xb^0) - k\omega(\sigma). 
\end{equation*}
This estimate shows that the number of iterations to reach $\lambda_k < \sigma$ is at most $k_{\max} = \left\lfloor \frac{F(\xb^0) - F(\xb^{*})}{\omega(\sigma)}\right\rfloor + 1$.
\eproof

\paragraph{Proof of Lemma \ref{le:step_size}:}
Proof of Lemma \ref{le:step_size} immediately follows from \eqref{eq:decrease_F_tmp} by taking $\Hb_k \equiv \mathbf{D}_k$ and $\db^k \equiv \db^k_g$.
\eproof

\paragraph{Proof of Theorem \ref{th:global_convergence_of_grad_method}:}
We consider the sequence $\set{F(\xb^k)}_{k\geq 0}$. By Lemma \ref{le:step_size}, this sequence is nonincreasing. Moreover, $F(\xb^0) \geq F(\xb^k) \geq F(\xb^{*})$ for all $k \geq 0$. As a result, the sequence $\set{F(\xb^k)}_{k\geq 0}$ converges to a finite value $F^{*}$. By  Lemma \ref{le:step_size}, we can derive 
\begin{equation*}
\sum_{j=0}^{\infty}\omega\left( \frac{\Vert\db^j_g\Vert_{\mathbf{D}_j}^2}{\Vert\db^j_g\Vert_{\xb^j}}\right) \leq F(\xb^0) - F^{*} < +\infty.
\end{equation*}
Since the function $\omega(\tau) = \tau - \ln(1 + \tau) \geq \frac{\tau^2}{4}$ for $\tau \in (0, 1]$ is increasing, this implies that $\lim_{j\to\infty}\Vert\db^j_g\Vert_2^2/\Vert\db^j_g\Vert_{\xb^j} = 0$ due to the fact that $\mathbf{D}_k \succeq \underline{L}\mathbb{I} \succ 0$. Since $\mathcal{L}_F(F(\xb^0))$ is bounded, by applying Zangwill's convergence theorem in \citep{Zangwill1969}, we can show that  every limit point $\xb^{*}$ of the sequence $\set{\xb^k}_{k\geq 0}$ is the stationary point of \eqref{eq:Fx_optimality}. Since $\xb^{*}$ is unique, the whole sequence $\set{\xb^k}_{k\geq 0}$ converges to $\xb^{*}$.
\eproof

\subsection{Proofs of local convergence: Theorem \ref{th:quad_converg_DPNM}, Theorem \ref{th:quasi_newton} and Theorem \ref{th:convergence_of_grad_method}}
We first provide a fixed-point representation of the optimality conditions and prove some key estimates  used in the sequel.

\paragraph{Optimality conditions as fixed-point formulations:}
Let $f$ be a given standard self-concordant function, $g$ be a given proper, lower semicontinuous and convex function, and $\Hb_k$ be a given symmetric positive definite matrix. 
Besides the two key inequalities \eqref{eq:SC_bound1} and \eqref{eq:SC_bound2}, we also need the following inequality \citep[Theorem 4.1.6]{Nesterov1994,Nesterov2004} in the proofs below:
\begin{align}
&\left(1 - \norm{\yb - \xb}_{\xb}\right)^2\nabla^2{f}(\xb) \preceq \nabla^2f(\yb) \preceq \left(1 - \norm{\yb - \xb}_{\xb}\right)^{-2}\nabla^2{f}(\xb), \label{eq:self_cond3}
\end{align}
for any $\xb, \yb\in\dom{f}$ such that $\norm{\yb-\xb}_{\xb} < 1$.

Let $\xb^{*}$ be the unique solution of \eqref{eq:COP} and $\xb^{*}$ be \textit{strongly regular}, i.e., $\nabla^2f(\xb^{*}) \succ 0$. 
Then the Dikin ellipsoid $W(\xb^{*},1) := \set{ \xb \in \mathbb{R}^n ~|~ \Vert \xb - \xb^{*} \Vert_{\xb^{*}} < 1}$ also belongs to $\dom{f}$. 
Moreover, $\nabla^2f(\xb)\succ 0$ for all $\xb\in W(\xb^{*},1)$ due to \cite[Theorem 4.1.5]{Nesterov2004}. Hence, the \textit{strong regularity} assumption is sufficient to ensure that $\nabla^2f$ is positive definite in the neighborhood $W(\xb^{*},1)$.

For a fixed $\bar{\xb}\in\dom{F}$, where $F := f + g$, we redefined the following operators, based on the fixed-point characterization and \eqref{eq:cvx_subprob}:
\begin{equation}\label{eq:proof_opers}
P_{\bar{\xb}}^g(\mathbf{z}) := P_{\nabla^2 f(\bar{\xb})}^g(\mathbf{z}), 
~~ S_{\bar{\xb}}(\mathbf{z}) := \nabla^2f(\bar{\xb})\mathbf{z} - \nabla{f}(\mathbf{z}),
\end{equation}
and
\begin{equation}\label{eq:proof_opers2}
\eb_{\bar{\xb}}(\Hb_k, \mathbf{z}) := \left(\nabla^2f(\bar{\xb}) - \Hb_k\right)(\mathbf{z}  - \xb^k).
\end{equation}
Here, $P_{\bar{\xb}}^g$ and $S_{\bar{\xb}}$ can be considered as a generalized proximal operator of $g$ and the gradient step of $f$, respectively. While $\eb_{\bar{\xb}}(\Hb_k,\cdot)$ measures the error between $\nabla^2f(\bar{\xb})$ and $\Hb_k$ along the direction $z-\xb^k$.

Next, given $\sb^k$ is the unique solution of \eqref{eq:cvx_subprob}, we characterize the optimality condition of the original problem \eqref{eq:COP} and the subproblem \eqref{eq:cvx_subprob} based on the $P_{\bar{\xb}}^g$, $S_{\bar{\xb}}$ and $\eb_{\bar{\xb}}(\Hb_k,\cdot)$  operators.
From \eqref{eq:subprob_opt}, we have
\begin{equation*}
S_{\bar{\xb}}(\xb^k) + \eb_{\bar{\xb}}(\Hb_k, \sb^k) \in \nabla^2f(\bar{\xb})\sb^k + \partial{g}(\sb^k).
\end{equation*}
By the definition of $P_{\bar{\xb}}^g$ in \eqref{eq:proof_opers}, the above expression leads to
\begin{equation}\label{eq:proof_Pg}
\sb^k = P_{\bar{\xb}}^g\left(S_{\bar{\xb}}(\xb^k) + \eb_{\bar{\xb}}(\Hb_k, \sb^k)\right).
\end{equation}
By replacing $\bar{\xb}$ with $\xb^{*}$, i.e., the unique solution of \eqref{eq:COP}, into \eqref{eq:proof_Pg} we obtain
\begin{equation}\label{eq:proof_sk}
\sb^k = P_{\xb^{*}}^g\left(S_{\xb^{*}}(\xb^k) + \eb_{\xb^{*}}(\Hb_k, \sb^k)\right).
\end{equation}
Moreover, if we replace $\Hb_k$ by $\nabla^2f(\xb^{*})$  (which is assumed to be positive definite) in the fixed-point expression \eqref{eq:fixed_point}, we finally have
\begin{equation}\label{eq:proof_xstar}
\xb^{*} = P_{\xb^{*}}^g\left(S_{\xb^{*}}(\xb^{*})\right).
\end{equation}
Formulas \eqref{eq:proof_Pg} to \eqref{eq:proof_xstar} represent the fixed-point formulation of the optimality conditions.
 
\paragraph{Key estimates:}  
Let $\mathbf{r}_k := \Vert \xb^k - \xb^{*}\Vert_{\xb^{*}}$ and $\lambda_k$ be defined by \eqref{eq:prox_Newton_decrement}. \
For any $\alpha_k \in (0, 1]$, we prove the following estimates:
\begin{align}
&\Vert \sb^{k+1}_n - \sb^k_n\Vert_{\xb^k} \leq \frac{\alpha_k^2\lambda_k^2}{1 - \alpha_k\lambda_k} + \frac{2\alpha_k\lambda_k - \alpha_k^2\lambda_k^2}{(1 - \alpha_k\lambda_k)^2}\Vert\db^{k+1}\Vert_{\xb^k}, \label{eq:key_ineq1} \\
&\Vert \sb^k - \xb^{*} \Vert_{\xb^{*}} \leq \frac{\mathbf{r}_k^2}{1 - \mathbf{r}_k} + \Vert (\Hb_k - \nabla^2f(\xb^{*}))\db^{k}\Vert_{\xb^{*}}^{*},  \label{eq:key_ineq2}
\end{align}
provided that $\alpha_k\lambda_k < 1$, $\mathbf{r}_k < 1$ and the first estimate \eqref{eq:key_ineq1} requires $\Hb_k = \nabla^2f(\xb^k)$.

\begin{proof}
First, by using the nonexpansiveness of $P_{\xb^k}^g$ in Lemma  \ref{le:Pg_properties}, it follows from \eqref{eq:proof_Pg} that
\begin{align}\label{eq:th2_est1}  
\Vert\sb^{k+1} - \sb^k\Vert_{\xb^k} &= \Big\Vert P^g_{\xb^k}(S_{\xb^k}(\xb^{k+1}) + \eb_{\xb^k}(\Hb_{k+1}, \sb^{k+1})) - P^g_{\xb^k}(S_{\xb^k}(\xb^k) + \eb_{\xb^k}(\Hb_k, \sb^k))\Big\Vert_{\xb^k} \nonumber\\
&\overset{\tiny\eqref{eq:P_g_property2}}{\leq} \norm{S_{\xb^k}(\xb^{k+1}) + \eb_{\xb^k}(\Hb_k, \sb^{k}) - S_{\xb^{*}}(\xb^{*})}_{\xb^{k}}^{*}\nonumber\\
&\stackrel{(i)}{\leq} \norm{\nabla{f}(\xb^{k+1}) - \nabla{f}(\xb^k) - \nabla^2{f}(\xb^k)(\xb^{k+1} - \xb^k)}_{\xb^k}^{*}\nonumber\\
& + \norm{\eb_{\xb^k}(\Hb_{k+1},\sb^{k+1}) - \eb_{\xb^k}(\Hb_k,\sb^k)}_{\xb^k}^{*}\nonumber\\
&\stackrel{(ii)}{=}\Big\Vert\int_0^1\left(\nabla^2{f}(\xb^k + \tau(\xb^{k+1} - \xb^k)) - \nabla^2f(\xb^k)\right)(\xb^{k+1}-\xb^k)d\tau\Big\Vert_{\xb^k}^{*} \nonumber\\
& + \Big\Vert\eb_{\xb^k}(\Hb_{k+1},\sb^{k+1}) - \eb_{\xb^k}(\Hb_k, \sb^k)\Big\Vert_{\xb^k}^{*},
\end{align} 
where $(i)$ and $(ii)$ are due to the triangle inequality and the mean-value theorem, respectively.

Second, we estimate the first term in \eqref{eq:th2_est1}. For this purpose, we  define
\begin{align}{\label{eq:Sigma}}
&\mathbf{\Sigma}_{k} := \int_0^1\left( \nabla^2f(\xb^{k} + \tau(\xb^{k+1} - \xb^{k})) - \nabla^2f(\xb^{k})\right)d\tau, \nonumber\\
&\mathbf{M}_{k} := \nabla^2f(\xb^{k})^{-1/2}\mathbf{\Sigma}_{k} \nabla^2f(\xb^{k})^{-1/2}. 
\end{align}
Based on the proof of \citep[Theorem 4.1.14]{Nesterov2004}, we can show that
\begin{equation*}
\norm{\mathbf{M}_{k}}_{2} \leq \frac{\Vert\xb^{k+1} - \xb^{k}\Vert_{\xb^{k}}}{1 - \norm{\xb^{k+1} - \xb^{k}}_{\xb^{k}}}.
\end{equation*} 
Using this estimate, the definition \eqref{eq:Sigma} and noting that $\xb^{k+1} = \xb^k + \alpha_k\db^k$, we obtain
\begin{align}\label{eq:proof_term1b}
\Vert\mathbf{\Sigma}_{k}(\xb^{k+1} - \xb^{k})\Vert_{\xb^{k}}^{*} &=\big[(\xb^{k+1} - \xb^{k})^T\Sigma_k\nabla^2{f}(\xb^k)^{-1}\Sigma_k(\xb^{k+1} - \xb^{k})\big]^{1/2}\nonumber\\
&=\big[ (\xb^{k+1} - \xb^{k})^T\nabla^2{f}(\xb^k)^{1/2}\mathbf{M}_k^T\mathbf{M}_k\nabla^2{f}(\xb^k)^{1/2}(\xb^{k+1} - \xb^{k}) \big]^{1/2} \nonumber\\
&=\Vert\mathbf{M}_k\nabla^2{f}(\xb^k)^{1/2}(\xb^{k+1} - \xb^{k}) \Vert_2 \nonumber\\
&\stackrel{(i)}{\leq} \Vert\mathbf{M}_{k}\Vert_{2} \big[(\xb^{k+1} - \xb^{k})^T\nabla^2f(\xb^k)(\xb^{k+1} - \xb^{k})\big]^{1/2}  \nonumber \\ 
&=\Vert\mathbf{M}_{k}\Vert_{2}\Vert\xb^{k+1}-\xb^k\Vert_{\xb^k}\nonumber\\
&\leq  \frac{\Vert\xb^{k+1} - \xb^{k}\Vert_{\xb^{k}}^2}{1 - \Vert\xb^{k+1} - \xb^{k}\Vert_{\xb^{k}}} \nonumber \\ 
&= \frac{\alpha_k^2\Vert\db^k\Vert_{\xb^k}^2}{1 - \alpha_k\Vert\db^k\Vert_{\xb^k}},
\end{align} 
where $(i)$ is due to the Cauchy-Schwartz inequality.

Third, we consider the second term in \eqref{eq:th2_est1} for $\Hb_k \equiv \nabla^2f(\xb^k)$. 
By the definition of $\eb_{\bar{\xb}}$, it is obvious that $\eb_{\xb^k}(\nabla^2f(\xb^k), \sb^k) = 0$. Hence, we have
\begin{align}\label{eq:proof_term1c}
\mathcal{T}_2 &:= \big\Vert\eb_{\xb^k}(\nabla^2f(\xb^{k+1}), \sb^{k+1}) - \eb_{\xb^k}(\nabla^2f(\xb^k), \sb^k)\big\Vert_{\xb^k}^{*}\nonumber\\
&= \big\Vert\eb_{\xb^k}(\nabla^2f(\xb^{k+1}), \sb^{k+1})\big\Vert_{\xb^k}^{*} \\
&= \Big\Vert\big(\nabla^2f(\xb^{k+1}) - \nabla^2f(\xb^k) \big)\db^{k+1}\Big\Vert_{\xb^k}^{*}. \nonumber 
\end{align}
We now define the following quantity, whose spectral norm we bound below
\begin{equation}{\label{eq:N}}
\mathbf{N}_k := \nabla^2f(\xb^k)^{-1/2}\left(\nabla^2f(\xb^{k+1}) - \nabla^2f(\xb^k)\right)\nabla^2f(\xb^k)^{-1/2}.
\end{equation}
By applying \eqref{eq:self_cond3} with $\xb = \xb^k$ and $\yb = \xb^{k+1}$, we can bound the spectral norm of $\mathbf{N}_k$ as follows
\begin{align}\label{eq:norm_Q}
\norm{\mathbf{N}_k}_2 &\leq \max\set{1 - \left(1 - \Vert\xb^{k+1} - \xb^k\Vert_{\xb^k}\right)^2, \left(1 - \Vert\xb^{k+1} - \xb^k\Vert_{\xb^k}\right)^{-2} - 1} \nonumber\\
& = \frac{2\Vert\xb^{k+1} - \xb^k\Vert_{\xb^k} - \Vert\xb^{k+1} - \xb^k\Vert_{\xb^k}^2}{(1 - \Vert\xb^{k+1} - \xb^k\Vert_{\xb^k})^2}.
\end{align}
Therefore, from \eqref{eq:proof_term1c} we can obtain the following estimate
\begin{align}\label{eq:th2_est3}
\left(\mathcal{T}_2\right)^2 &= \eb_{\xb^{k}}(\nabla^2f(\xb^{k+1}), \sb^{k+1})^T\nabla^2f(\xb^k)^{-1}\eb_{\xb^{k}}(\nabla^2f(\xb^{k+1}), \sb^{k+1}) \nonumber\\
& = (\db^{k+1})^T~\nabla^2f(\xb^k)^{1/2}~\mathbf{N}_k^2~\nabla^2f(\xb^k)^{1/2}~\db^{k+1}\nonumber\\
&\leq \norm{\mathbf{N}_k}^2_2 \Vert\db^{k+1}\Vert_{\xb^k}^2.
\end{align}
By substituting \eqref{eq:norm_Q} into \eqref{eq:th2_est3} and noting that $\alpha_k\db^k = \xb^{k+1} - \xb^k$, we obtain
\begin{equation}\label{eq:th2_est4}
\mathcal{T}_2 \leq \frac{2\alpha_k\norm{\db^k}_{\xb^k} - \alpha_k^2\norm{\db^k}_{\xb^k}^2}{(1 - \alpha_k\norm{\db^k}_{\xb^k})^2}\Vert\db^{k+1}\Vert_{\xb^k}. 
\end{equation}
Now, by substituting \eqref{eq:proof_term1b} and \eqref{eq:th2_est4} into \eqref{eq:th2_est1} and noting that $\Hb_k \equiv \nabla^2f(\xb^k)$, $\sb^k \equiv \sb^k_n$, $\db^k \equiv \db^k_n$ and $\lambda_k \equiv \norm{\db^k_n}_{\xb^k}$, we obtain
\begin{align*}
\big\Vert\sb^{k+1}_n - \sb^k_n\big\Vert_{\xb^k} \leq \frac{\alpha_k^2\norm{\db^k_n}_{\xb^k}^2}{1 - \alpha_k\norm{\db^k_n}_{\xb^k}} + \frac{2\alpha_k\norm{\db^k_n}_{\xb^k} -
\alpha_k^2\norm{\db^k_n}_{\xb^k}^2}{(1-\alpha_k\norm{\db^k_n}_{\xb^k})^2}\Vert\db^{k+1}_n\Vert_{\xb^k}.
\end{align*}
which is indeed \eqref{eq:key_ineq1}.

Similarly to the proof of \eqref{eq:th2_est1} and \eqref{eq:proof_term1b}, we have
\begin{align}\label{eq:th2_est1a}
\Vert\sb^{k} - \xb^{*}\Vert_{\xb^{*}} &\overset{\eqref{eq:proof_xstar}}{=} \Big\Vert P^g_{\xb^{*}}(S_{\xb^{*}}(\xb^k) + \eb_{\xb^{*}}(\Hb_k, \sb^{k})) - P^g_{\xb^{*}}(S_{\xb^{*}}(\xb^{*}))\Big\Vert_{\xb^{*}} \nonumber\\
&\overset{\tiny\eqref{eq:P_g_property2}}{\leq} \norm{S_{\xb^{*}}(\xb^{k}) + \eb_{\xb^{*}}(\Hb_k, \sb^{k}) - S_{\xb^{*}}(\xb^{*})}_{\xb^{*}}^{*}\nonumber\\
&\leq \norm{\nabla{f}(\xb^{k}) - \nabla{f}(\xb^{*}) - \nabla^2{f}(\xb^{*})(\xb^{k} - \xb^{*})}_{\xb^{*}}^{*} + \norm{\eb_{\xb^{*}}(\Hb_k,\sb^k)}_{\xb^{*}}^{*}\nonumber\\
&= \norm{\int_0^1\left(\nabla^2{f}(\xb^{*} + \tau(\xb^{k} - \xb^{*})) - \nabla^2f(\xb^{*})\right)(\xb^k - \xb^{*})d\tau}_{\xb^{*}}^{*} + \norm{\eb_{\xb^{*}}(\Hb_k,\sb^k)}_{\xb^{*}}^{*}\nonumber\\
&\overset{\eqref{eq:proof_term1b}}{\leq} \frac{\Vert\xb^k - \xb^{*}\Vert_{\xb^{*}}^2}{1 - \norm{\xb^k-\xb^{*}}_{\xb^{*}}} + \big\Vert\left(\Hb_k - \nabla^2f(\xb^{*})\right)\db^k\big\Vert_{\xb^{*}}^{*},
\end{align}
which is indeed \eqref{eq:key_ineq2} since $\mathbf{r}_k = \norm{\xb^k - \xb^{*}}_{\xb^{*}}$.
\end{proof}

\paragraph{Proof of Theorem \ref{th:quad_converg_DPNM}:}
Since $\xb^k = \sb^k_n - \db^k_n$ due to \eqref{eq:DPNM}, we have $\xb^{k+1}  = \xb^k + \alpha_k\db^k_n = \sb^k_n - (1-\alpha_k)\db^k_n$, 
which leads to 
\begin{align*}
\db^{k+1}_n &= \sb^{k+1}_n - \xb^{k+1} = \sb^{k+1}_n - \sb^k_n + (1-\alpha_k)\db^k_n.
\end{align*} 
By applying the triangle inequality to the above expression, we have
\begin{align}\label{eq:dk1_est}
\Vert\db^{k+1}_n\Vert_{\xb^k} = \Vert\sb^{k+1}_n - \sb^k_n + (1-\alpha_k)\db^k_n\Vert_{\xb^k} \leq \Vert\sb^{k+1}_n - \sb^k_n\Vert_{\xb^k} + (1-\alpha_k)\Vert\db^k_n\Vert_{\xb^k}.
\end{align}
Substituting \eqref{eq:key_ineq1} into \eqref{eq:dk1_est} we obtain
\begin{align*}
\Vert \db^{k+1}_n\Vert_{\xb^k} \leq \frac{\alpha_k^2\lambda_k^2}{1 - \alpha_k\lambda_k} + \frac{2\alpha_k\lambda_k - \alpha_k^2\lambda_k^2}{(1 - \alpha_k\lambda_k)^2}\Vert\db^{k+1}\Vert_{\xb^k} + (1-\alpha_k)\lambda_k.
\end{align*}
Rearranging this inequality we get
\begin{align}\label{eq:th2_est5}
\Vert\db^{k+1}_n\Vert_{\xb^k} \leq \left(\frac{\left(1 - \alpha_k\lambda_k\right)\left(1-\alpha_k + (2\alpha_k^2-\alpha_k)\lambda_k\right)}{1 - 4\alpha_k\lambda_k +
2\alpha_k^2\lambda_k^2}\right)\lambda_k,
\end{align}
provided that $1 - 4\alpha_k\lambda_k + 2\alpha_k^2\lambda_k^2 > 0$.
Now, by applying \eqref{eq:self_cond3} with $\xb = \xb^k$ and $\yb = \xb^{k+1}$, one can show that
\begin{align}\label{eq:th2_est6}
\Vert\db^{k+1}_n\Vert_{\xb^{k+1}} \leq \frac{\Vert\db^{k+1}_n\Vert_{\xb^k}}{1 - \alpha_k\Vert\db^k_n\Vert_{\xb^k}}.
\end{align}
We note that $1 - 4\alpha_k\lambda_k + 2\alpha_k^2\lambda_k^2 > 0$ if $\alpha_k\lambda_k < 1 - 1/\sqrt{2}$.
By combining \eqref{eq:th2_est5} and \eqref{eq:th2_est6} we obtain
\begin{align*} 
\Vert\db^{k+1}_n\Vert_{\xb^{k+1}} \leq \left(\frac{1-\alpha_k + (2\alpha_k^2 - \alpha_k)\lambda_k}{1-4\alpha_k\lambda_k + 2\alpha_k^2\lambda_k^2}\right)\lambda_k,
\end{align*}
which is \eqref{eq:DPNM_estimate}.

Next, we consider the sequence $\set{\xb^k}_{k\geq 0}$ generated by damped step proximal Newton method \eqref{eq:DPNM} with the step size $\alpha_k = (1 + \lambda_k)^{-1}$. 
It is clear that \eqref{eq:DPNM_estimate} is transformed into
\begin{align}
\lambda_{k+1} \leq \frac{2\lambda_k}{1 - 2\lambda_k - \lambda_k^2}\lambda_k.
\end{align} Assuming $\lambda_k \leq \bar{\sigma} := \sqrt{5}-2$, we can easily deduce that $\frac{2\lambda_k}{1-2\lambda_k - \lambda_k^2} \leq 1$ and thus, $\lambda_{k+1} \leq \lambda_k$. 
By induction, if $\lambda_0 \leq \bar{\sigma}$ then, $\lambda_{k+1} \leq \lambda_k$ for all $k \geq 0$. Moreover, we have $\lambda_{k+1} \leq \frac{2}{1 - 2\bar{\sigma} - \bar{\sigma}^2}\lambda_k^2$, which shows that the sequence $\set{\lambda_k}_{k\geq 0}$ converges to zero at a quadratic rate, which completes the proof of part b).

Now, since $\alpha_k = 1$, the estimate \eqref{eq:DPNM_estimate} reduces to $\lambda_{k+1} \leq \frac{\lambda_k^2}{1 - 4\lambda_k + 2\lambda_k^2}$. By the same argument as in the proof of part b), we can show that the sequence $\set{\lambda_k}_{k\geq 0}$ converges to zero at a quadratic rate.

Finally, we prove the last statement in Theorem \ref{th:quad_converg_DPNM}.
By substituting $\Hb_k := \nabla^2{f}(\xb^k)$ into \eqref{eq:key_ineq1}, we obtain
\begin{equation}\label{eq:th5a_est1}
\Vert \sb^k - \xb^{*} \Vert_{\xb^{*}} \leq \frac{\mathbf{r}_k^2}{1 - \mathbf{r}_k} + \Vert (\nabla^2{f}(\xb^k) - \nabla^2f(\xb^{*}))\db^{k}\Vert_{\xb^{*}}^{*}.
\end{equation}
Let $\mathcal{T}_3 := \Vert (\nabla^2{f}(\xb^k) - \nabla^2f(\xb^{*}))\db^{k}\Vert_{\xb^{*}}^{*}$. 
Similarly to the proof of \eqref{eq:th2_est4}, we can show that
\begin{equation}\label{eq:th5a_est2}
\mathcal{T}_3 \leq \left[\frac{2\Vert\xb^{k} - \xb^{*}\Vert_{\xb^{*}} - \Vert\xb^{k} - \xb^{*}\Vert_{\xb^{*}}^2}{(1 - \Vert\xb^{k} - \xb^{*}\Vert_{\xb^{*}})^2}\right]\Vert\db^k\Vert_{\xb^{*}} \leq \alpha_k\frac{(2-\rb_k)\rb_k}{(1-\rb_k)^2}(\rb_{k+1} + \rb_k).
\end{equation}
Here the second inequality follows from the fact that $\Vert\db^k\Vert_{\xb^{*}} = \alpha_k\Vert\xb^{k+1} - \xb^k\Vert_{\xb^{*}} \leq \alpha_k[\Vert\xb^{k+1}-\xb^{*}\Vert_{\xb^{*}} + \Vert\xb^k - \xb^{*}\Vert_{\xb^{*}}] = \alpha_k(\rb_{k+1} + \rb_k)$. We also have $\rb_{k+1} = \Vert\xb^{k+1} - \xb^{*}\Vert_{\xb^{*}} = \Vert(1-\alpha_k)\xb^k + \alpha_k\sb^k - \xb^{*}\Vert_{\xb^{*}} \leq (1-\alpha_k)\rb_k + \alpha_k\Vert\sb_k - \xb^{*}\Vert_{\xb^{*}}$. Using these inequalities, \eqref{eq:th5a_est2} and \eqref{eq:th5a_est1} we get
\begin{align}\label{eq:thm5_est3}
\rb_{k+1} \leq (1-\alpha_k)\rb_k + \alpha_k\frac{\rb_k^2}{1-\rb_k} + \alpha_k^2\frac{(2-\rb_k)\rb_k}{(1-\rb_k)^2}(\rb_{k+1} + \rb_k).
\end{align}
Rearranging this inequality to obtain
\begin{align}\label{eq:thm5_est4}
\rb_{k+1} \leq \left(\frac{1 - \alpha_k + (2\alpha_k^2 + 3\alpha_k - 2)\rb_k + (1-\alpha_k - \alpha_k^2)\rb_k^2}{1 - 2(1 + \alpha_k^2)\rb_k + (1+\alpha_k^2)\rb_k^2}\right)\rb_k.
\end{align}
We consider two cases:

\noindent\textbf{Case 1}: $\alpha_k = 1$: We have $\rb_{k+1} \leq \frac{3 - \rb_k}{1 - 4\rb_k + 2\rb_k^2}\rb_k^2$.  Hence, if $\rb_k < 1-1/\sqrt{2}$ then $1 - 4\rb_k + 2\rb_k^2 > 0$. Moreover, $\rb_{k+1} \leq \rb_k$ if $3\rb_k - \rb_k^2 < 1 - 4\rb_k + 2\rb_k^2$, which is satisfied if $\rb_k < (7 - \sqrt{37})/6\approx 0.152873$.
Now, if we assume that $\rb_0 \leq \sigma \in (0, (7 - \sqrt{37})/6)$, then, by induction, we have $\rb_{k+1} \leq \frac{3-\sigma}{1 - 4\sigma + 2\sigma^2}\rb_k^2$. This shows that $\set{\rb_k}_{k\geq 0}$ locally converges to $0^{+}$ at a quadratic rate. Since $\rb_k := \norm{\xb^k - \xb^{*}}_{\xb^{*}}$, we can conclude that $\xb^k \to\xb^{*}$ at a quadratic rate as $k\to\infty$.

\noindent\textbf{Case 2}: $\alpha_k = (1 + \lambda_k)^{-1}$: Since $\lambda_k = \norm{\xb^{k+1} - \xb^k}_{\xb^k} \leq \frac{\norm{\xb^{k+1} - \xb^{*}}_{\xb^{*}} + \norm{\xb^{k} - \xb^{*}}_{\xb^{*}}}{1 - \norm{\xb^k - \xb^{*}}_{\xb^{*}}} = \frac{\rb_{k+1} + \rb_k}{1 - \rb_k}$. We have $1-\alpha_k \leq \frac{\rb_{k+1} + \rb_k}{(1+\lambda_k)(1 - \rb_k)} \leq \frac{\rb_{k+1} + \rb_k}{1 - \rb_k}$. Substituting this into \eqref{eq:thm5_est3} and using the fact that $\alpha_k \leq 1$, we have 
\begin{equation*}
\rb_{k+1} \leq \frac{(\rb_{k+1} + \rb_k)\rb_k}{1 - \rb_k} + \frac{\rb_k^2}{1-\rb_k} + \frac{(2-\rb_k)\rb_k}{(1-\rb_k)^2}(\rb_{k+1} + \rb_k).
\end{equation*}
Rearranging this inequality, we finally get
\begin{equation}\label{eq:thm5_est6}
\rb_{k+1} \leq \frac{4-3\rb_k}{1 - 5\rb_k + 3\rb_k^2}\rb_k^2.
\end{equation}
Since $1 - 5\rb_k + 3\rb_k^2 > 0$ if $\rb_k < (5-\sqrt{13})/6$, we can see from \eqref{eq:thm5_est6} that $\rb_k < (9-\sqrt{57})/12 \approx 0.120847$ then $\rb_{k+1} \leq \rb_k$. By induction, if we choose $\rb_0 \leq \bar{\sigma} \in (0, (9-\sqrt{57})/12)$ then $\rb_{k+1} \leq \frac{4-3\bar{\sigma}}{1 - 5\bar{\sigma} + 3\bar{\sigma}^2}\rb_k^2$, which shows that $\set{\rb_k}_{k\geq 0}$ converges to $0^{+}$ at a quadratic rate. Consequently, the sequence $\set{\xb^k}_{k\geq 0}$ locally converges to $\xb^{*}$ at a quadratic rate.
\eproof

\paragraph{Proof of Theorem \ref{th:quasi_newton}:}
We first prove the statement (a). Since $\xb^{k+1} \equiv \sb^k_q$ due to \eqref{eq:PFQNM}, from \eqref{eq:key_ineq2} we have
\begin{equation}\label{eq:thm10_main_est}
\mathbf{r}_{k+1} \leq \frac{\mathbf{r}_k^2}{1 - \mathbf{r}_k} + \Big\Vert\left(\Hb_k - \nabla^2f(\xb^{*})\right)(\xb^{k+1} - \xb^{k})\Big\Vert_{\xb^{*}}^{*}.
\end{equation}
Now, by using the condition \eqref{eq:Dennis_More_cond}, we can easily show that the sequence $\set{\xb^k}_{k\geq 0}$ converges super-linearly to $\xb^{*}$ provided that $\norm{\xb^0 - \xb^{*}}_{\xb^{*}} \leq \rho_0 < 1$.

Next, we prove the statement (b).
It is well-known (see, e.g., \cite{Nocedal2006}) that if matrix $\Hb_k$ is positive definite and $(\yb^k)^T(\zb^k) > 0$ then the matrix $\Hb_{k+1}$ updated by \eqref{eq:bfgs_formula} is also positive definite. Indeed, we have $(\yb^k)^T(\zb^k) = \int_0^1(\zb^k)^T\nabla^2f(\xb^k + t\zb^k)\zb^k dt$. Therefore, under the condition $\norm{\zb^k}_{\xb^k} < 1$, we can show that $(\yb^k)^T(\zb^k) \geq (\zb^k)^T\nabla^2f(\xb^k)\zb^k = \norm{\zb^k}_{\xb^k}^2 > 0$. By multiplying \eqref{eq:bfgs_formula} by $\zb^k$ we can easily see that $\Hb_{k+1}$ satisfies the secant equation \eqref{eq:scant_eq}. 

Finally, we estimate $\norm{\yb^k - \nabla^2f(\xb^{*})\zb^k}_{\xb^{*}}^{*}$ as follows
\begin{equation}\label{eq:thm10_main_est2}
\Vert\yb^k - \nabla^2f(\xb^{*})\zb^k\Vert_{\xb^{*}}^{*} \leq \frac{\mathbf{r}_k + \mathbf{r}_{k+1}}{(1-\mathbf{r}_k)(1-\mathbf{r}_{k+1})} \Vert\zb^k\Vert_{\xb^{*}}.
\end{equation}
Now, by assumption that $\sum_{k=0}^{\infty} \mathbf{r}_k < +\infty$, we obtain from \eqref{eq:thm10_main_est2} that $\sum_{k=0}^{\infty}\varepsilon_k <+\infty$, where $\varepsilon_k := \frac{\mathbf{r}_k + \mathbf{r}_{k+1}}{(1-\mathbf{r}_k)(1-\mathbf{r}_{k+1})}$. 
By applying \cite[Theorem 3.2.]{Byrd1989}, we can show that the Dennis-Mor\'{e} condition \eqref{eq:Dennis_More_cond} is satisfied. This implies that the sequence $\set{\xb^k}_{k\geq 0}$ generated by scheme \eqref{eq:PFQNM} converges super-linearly to $\xb^{*}$.
\eproof

\paragraph{Proof of Theorem \ref{th:convergence_of_grad_method}:}
For $\norm{\xb^k - \xb^{*}}_{\xb^{*}} < 1$, from \eqref{eq:key_ineq2}, we have
\begin{equation}\label{eq:th9_est1}
\Vert \sb^k_g - \xb^{*}\Vert_{\xb^{*}} \leq \frac{\Vert \xb^k - \xb^{*}\Vert_{\xb^{*}}^2}{1 - \Vert \xb^k - \xb^{*}\Vert_{\xb^{*}}} + \norm{\left(\mathbf{D}_k - \nabla^2f(\xb^{*})\right)\db^k}_{\xb^{*}}^{*}.
\end{equation}
Now, using the condition $\norm{\left(\mathbf{D}_k - \nabla^2f(\xb^{*})\right)\db^k}_{\xb^{*}}^{*} \leq \frac{1}{2}\Vert\db^k_g\Vert_{\xb^{*}}$, \eqref{eq:th9_est1} implies
\begin{align*}
\Vert \sb^k_g - \xb^{*}\Vert_{\xb^{*}} &\leq \frac{\Vert \xb^k - \xb^{*}\Vert_{\xb^{*}}^2}{1 - \Vert \xb^k - \xb^{*}\Vert_{\xb^{*}}} + \gamma\Vert\db^k_g\Vert_{\xb^{*}} \\
&\leq \frac{\Vert \xb^k - \xb^{*}\Vert_{\xb^{*}}^2}{1 - \Vert \xb^k - \xb^{*}\Vert_{\xb^{*}}} + \gamma\Vert\sb^k_g - \xb^{*}\Vert_{\xb^{*}} + \gamma\Vert\xb^k - \xb^{*}\Vert_{\xb^{*}},
\end{align*}
where $\gamma \in (0, 1/2)$.
Rearranging this inequality, we obtain
\begin{align}\label{eq:th9_est2}
\Vert \sb^k_g - \xb^{*}\Vert_{\xb^{*}} &\leq \frac{1}{1-\gamma}\left(\gamma + \frac{\Vert \xb^k - \xb^{*}\Vert_{\xb^{*}}}{1 - \Vert \xb^k - \xb^{*}\Vert_{\xb^{*}}}\right)\Vert\xb^k - \xb^{*}\Vert_{\xb^{*}}.
\end{align}
Now, since $\xb^{k+1} = \xb^k + \alpha_k\db^k_g = (1-\alpha_k)\xb^k + \alpha_k\sb^k_g$, we can further estimate from \eqref{eq:th9_est2} as
\begin{align}\label{eq:th9_est3}
\Vert \xb^{k+1} - \xb^{*}\Vert_{\xb^{*}} &\leq (1-\alpha_k)\Vert \xb^k - \xb^{*}\Vert_{\xb^{*}} + \alpha_k\Vert \sb^k_g - \xb^{*}\Vert_{\xb^{*}} \nonumber\\
&\leq \left[ 1-\alpha_k + \frac{\alpha_k}{1-\gamma}\left(\gamma + \frac{\Vert \xb^k - \xb^{*}\Vert_{\xb^{*}}}{1 - \Vert \xb^k - \xb^{*}\Vert_{\xb^{*}}}\right)\right]\Vert \xb^k - \xb^{*}\Vert_{\xb^{*}}.
\end{align}
Let us define $\tilde{\psi}_k := (1-\alpha_k) + \frac{\alpha_k}{1-\gamma}\left(\gamma + \frac{\norm{\xb^{k} - \xb^{\ast}}_{\xb^{\ast}}}{1 - \norm{\xb^{k} - \xb^{\ast}}_{\xb^{\ast}}}\right)$. Then, for $\gamma < \frac{1}{2}$, $\tilde{\psi}_k < 1$ if  $\norm{\xb^{k} - \xb^{\ast}}_{\xb^{\ast}} < \frac{1 - 2\gamma}{2(1-\gamma)}$. Therefore, by induction, if we choose $\norm{\xb^0 - \xb^{*}}_{\xb^{*}} < \frac{1 - 2\gamma}{2(1-\gamma)}$, then $\norm{\xb^{k} - \xb^{\ast}}_{\xb^{\ast}} < \frac{1 - 2\gamma}{2(1-\gamma)}$ for all $k\geq 0$. Moreover, $\norm{\xb^{k+1} - \xb^{\ast}}_{\xb^{\ast}} \leq \tilde{\psi}_k\norm{\xb^{k} - \xb^{\ast}}_{\xb^{\ast}}$ for $k\geq 0$ and $\tilde{\psi}_k \in [0, 1)$. This implies that $\set{\norm{\xb^{k} - \xb^{\ast}}_{\xb^{\ast}}}_{k\geq 0}$ linearly converges to zero with the factor $\tilde{\psi}_k$.

Finally, we assume that $\mathbf{D}_k := L_k\mathbb{I}$, the quantity in \eqref{eq:N} satisfies
\begin{align}\nonumber
\mathbf{N}_{*} := \nabla^2{f}(\xb^{*})^{-1/2}\left(\nabla^2{f}(\xb^{*}) - \Hb_k\right) \nabla^2{f}(\xb^{*})^{-1/2} = \mathbb{I} -  L_k\nabla^2{f}(\xb^{*})^{-1}. 
\end{align} 
Then, we can easily observe that
\begin{equation}\label{eq:proof_Nest}
\norm{\mathbf{N}_{*}}_{2} = \norm{\mathbb{I} - L_k\nabla^2{f}(\xb^{*})^{-1}}_2 \leq \max\set{\Big\vert 1 - \frac{L_k}{\sigma_{\min}^{*}}\Big\vert, \Big\vert 1 - \frac{L_k}{\sigma_{\max}^{*}}\Big\vert} := \gamma_{*},
\end{equation} where $\sigma_{\min}^{*}$ (respectively, $\sigma_{\max}^{*}$) is the smallest (respectively, largest) eigenvalue of $\nabla^2f(\xb^{*})$. 
Using the estimate \eqref{eq:proof_Nest}, we can derive 
\begin{align*}
\big\Vert\left(\mathbf{D}_k - \nabla^2f(\xb^{*})\right)\db^k_g\big\Vert_{\xb^{*}}^{*} &\stackrel{\eqref{eq:proof_Nest}}{\leq} \Vert\mathbf{N}_{*}\Vert_{2} \Vert\sb^k - \xb^k\Vert_{\xb^{*}} \nonumber\\
&\leq \gamma_{*}\Vert\db^k_g\Vert_{\xb^{*}},
\end{align*}
which proves the last conclusion of Theorem \ref{th:convergence_of_grad_method}.
\eproof

\bibliographystyle{acm}
\bibliography{tran_bibtex_new}
\end{document}